\documentclass{article}

\usepackage[preprint]{neurips_2026}

\usepackage[utf8]{inputenc} 
\usepackage[T1]{fontenc}    
\usepackage{hyperref}       
\usepackage{url}            
\usepackage{booktabs}       
\usepackage{amsfonts}       
\usepackage{nicefrac}       
\usepackage{microtype}      
\usepackage{xcolor}         
\usepackage{graphicx}  
\usepackage{float}
\usepackage{placeins}
\usepackage{booktabs}
\usepackage{array}
\usepackage{tabularx}
\usepackage{colortbl}
\usepackage{xcolor}
\usepackage{tcolorbox}   
\usepackage{xcolor}    
\usepackage{listings}   
\usepackage{enumitem}
\usepackage{tcolorbox}
\usepackage{pifont}  
\tcbuselibrary{breakable}
\usepackage{listings}
\usepackage{subcaption}
\usepackage{makecell}
\usepackage{multirow}
\usepackage[normalem]{ulem}
 \usepackage{rotating}
\usepackage{bm}

\usepackage{xcolor}
\definecolor{niceRed}{RGB}{190,38,38}
\definecolor{blueGrotto}{HTML}{059DC0}
\definecolor{royalBlue}{HTML}{057DCD}
\definecolor{navyBlueP}{HTML}{0B579C}
\definecolor{limeGreen}{HTML}{81B622}
\definecolor{oliveGreen}{HTML}{026e00}

\definecolor{niceYellow}{HTML}{f5b400}
\definecolor{nicePurple}{HTML}{9c27b0}
\definecolor{lightRoyalBlue}{HTML}{def2ff}  
\definecolor{pastelGreen}{HTML}{4CB944}
\definecolor{gold}{HTML}{ffa300}
\definecolor{lightRoyalBlue}{HTML}{def2ff} 
\usepackage{color-edits} 
\addauthor[Panos]{pp}{gold}
\addauthor[Katerina]{km}{teal}
\addauthor[Kostas]{kt}{royalBlue}
\addauthor[Nikos]{nn}{niceRed}
\addauthor[Efty]{em}{nicePurple}

\usepackage{amsmath,amssymb,amsthm,mathtools}

\newcommand{\R}{\mathbb{R}}
\newcommand{\tr}{\operatorname{tr}}

\newtheorem{theorem}{Theorem}

\newtheorem{lemma}[theorem]{Lemma}

\theoremstyle{definition}
\newtheorem{definition}[theorem]{Definition}

\theoremstyle{remark}

\theoremstyle{assumption}
\newtheorem{assumption}[theorem]{Assumption}

\title{Rank Is Not Capacity: Spectral Occupancy for Latent Graph Models}

\author{%
  \textbf{Nikolaos Nakis}$^{1,}\thanks{Shared first authorship}$ \quad
  \textbf{Panagiotis Promponas}$^{2,*}$ \quad
  \textbf{Konstantinos Tsirkas}$^{3}$ \quad
  \textbf{Katerina Mamali}$^{4}$ \\
  \textbf{Eftychia Makri}$^{2}$ \quad
  \textbf{Leandros Tassiulas}$^{2}$ \quad
  \textbf{Nicholas A. Christakis}$^{1}$ \\[0.5em]
  $^{1}$Human Nature Lab, Yale University \\
  $^{2}$Department of Electrical and Computer Engineering, Yale University \\
  $^{3}$Department of Statistics and Data Science, Yale University \\
  $^{4}$Department of Computer Science, Yale University \\
  New Haven, CT 06511 \\
  \texttt{nicolaos.nakis@gmail.com, panagiotis.promponas@yale.edu} \\
}

\begin{document}

\maketitle

\begin{abstract}

Graph representation learning has become a standard approach for analyzing
networked data, with latent embeddings widely used for link prediction,
community detection, and related tasks. Yet a basic design choice, the latent
dimension, is still treated as a brittle hyperparameter, fixed before training
and tuned by held-out performance. Learned factors are also identifiable only up
to rotation and rescaling, so the nominal rank rarely coincides with the
quantity that governs model behavior. We propose Spectral Prefix Extraction and
Capacity-Targeted Representation Analysis (\textsc{Spectra}), which replaces
rank as the unit of analysis with the spectrum of a learned positive
semidefinite kernel, trace-normalized so that spectra are comparable across
fits. The normalized eigenvalues form a distribution on the simplex, and their
Shannon effective rank acts both as a summary of learned capacity and as a
controllable training-time coordinate: a single scalar shapes this realized
dimension during training, and bisection targets any desired value within the
rank cap. 
To theoretically support that, we show local regularity and monotonicity of the realized-dimension profile. 
Across collaboration,
social, biological, and infrastructure networks, \textsc{Spectra} traces
performance--capacity frontiers that make the trade-off between predictive
accuracy and realized dimension visible. It performs competitively with strong
link-prediction baselines, yields aligned lower-capacity views of the same
fitted model through spectral prefixes, and provides a principled handle on
capacity in the overparameterized regime. Capacity thus becomes a property of
the fitted model rather than a hyperparameter of the training.
\end{abstract}
\section{Introduction}
A central design choice in graph representation learning is the
dimension of the latent space. Despite its influence on model behavior,
this quantity is usually treated as a fixed hyperparameter, selected
before training and tuned by held-out performance. Modern graph
representation methods, including random-walk embeddings
\citep{grover2016node2vec,role2vec}, matrix-factorization approaches
\citep{netmf,grarep}, mixed-membership models
\citep{airoldi2007mixedmembershipstochasticblockmodels,MNMF}, latent
space and inner-product models
\citep{hoff2002latent,hoff2007modeling,athreya2018statistical,
rubin2022statistical,sussman2013consistent}, and simplex-volume
extensions such as \textsc{HM-LDM}
\citep{nakis2022hmldmhybridmembershiplatentdistance,pmlr-v206-nakis23a},
all rely, explicitly or implicitly, on such a choice. Across these
families, a basic question persists: how many latent dimensions do
the data actually require?

Existing selection procedures, including profile-likelihood elbows
\citep{zhu2006automatic}, Bayesian latent-dimension selection
\citep{passino2020bayesian}, network cross-validation
\citep{chen2018network,li2020network}, structural criteria
\citep{gu2021principled}, input-graph entropy methods
\citep{luo2021graph}, and spectral parallel analysis such as
\textsc{NetFlipPA} \citep{hong2025network}, largely select a single
dimension before or after fitting. They do not expose
\emph{representational capacity}, i.e., how many latent modes the fitted
model effectively uses, as a quantity that can be measured,
controlled, or audited in the fitted model itself. Also, learned
latent factors are identifiable only up to rotation and global
rescaling
\citep{hoff2002latent,airoldi2007mixedmembershipstochasticblockmodels,
athreya2018statistical}, and dimension misspecification has measurable
statistical costs \citep{taing2026effect}. The operative quantity is
therefore not only the declared rank, but how representational mass is
distributed across latent modes.

We propose the Spectral Prefix Extraction and Capacity-Targeted Representation Analysis (\textsc{Spectra}), to replace rank as the unit of analysis with the spectrum of a learned positive semidefinite (PSD) kernel. The kernel is
invariant under the gauge symmetries of the factorization, trace
normalization fixes total spectral mass, and the normalized
eigenvalues form a probability distribution on latent modes, which we
call the \emph{spectral occupancy distribution}. We summarize this
distribution by the \emph{effective spectral dimension}
$d_{\mathrm{spec}}$, the exponential of its Shannon entropy
\citep{roy2007effective,friedman2022vendi}. Spectral entropy has been
used as a diagnostic in language models
\citep{wei2024diff,jha2025spectral}, training dynamics
\citep{yang2024effective}, and adaptive-rank compression
\citep{cherukuri2025low}. Here, it becomes a controllable,
end-to-end training-time capacity coordinate for latent graph models.

This positioning connects to two views of overparameterized models:
effective complexity can govern generalization even when parameter
counts are large
\citep{belkin2019reconciling,bartlett2020benign}, and gradient
descent on factorized matrix models is implicitly biased toward
low-rank solutions
\citep{gunasekar2017implicit,arora2019implicit}. Our approach,
makes this bias explicit and parameterizable: an entropy weight
$\eta$ shapes the learned spectrum, bisection targets a desired
$d_{\mathrm{spec}}$, and the rank cap becomes a parameterization
ceiling. Trace normalization separates scale from rank, avoiding the
scale coupling of nuclear-norm penalties \citep{srebro2004maximum},
while spectral prefixes provide  optimal low-rank PSD summaries of the fitted kernel \citep{Eckart1936TheAO}.

Our contributions are: 
(i) \textbf{Spectral occupancy as realized capacity.}
We replace nominal rank with the trace-normalized spectrum of the learned PSD
kernel, and use its Shannon effective rank, $d_{\mathrm{spec}}$, as a smooth
measure of realized latent dimension. 
(ii) \textbf{Training-time capacity control.}
We introduce an entropy-regularized objective in which a scalar $\eta$ shapes
the realized spectral dimension, and use bisection to target a prescribed
$d_{\mathrm{spec}}$ within tolerance. 
(iii) \textbf{Spectral identifiability and summaries.}
We verify the identifiability of active spectral modes under a simple-spectrum
condition in our framework, we establish local $\mathcal{C}^1$ behavior of the realized-dimension profile under regularity conditions, and we use the optimality of spectral prefixes as
low-rank PSD summaries. 
(iv) \textbf{Empirical capacity frontiers.}
Across eight benchmark networks and three rank caps, AUC aligns by achieved
$d_{\mathrm{spec}}$ where rank caps overlap. The resulting frontiers distinguish
saturated datasets, whose best operating points lie below the cap, from
rank-cap-binding datasets, whose best operating points increase with $r$. 
(v) \textbf{Single-fit prefix families.}
A fitted kernel yields a nested family of aligned spectral prefixes, providing
lower-capacity views of the same representation without retraining. 
(vi) \textbf{Spectrally controlled overparameterization.}
At matched effective capacity, $\eta$-targeted overparameterized fits improve
test log-likelihood and train--test gap over rank-cap-only baselines in paired
experiments.

\section{Related Work}

Latent-distance and inner-product models~\citep{hoff2002latent,hoff2007modeling}
alongside random dot product graph (RDPG) frameworks~\citep{athreya2018statistical,sussman2013consistent}
treat spectra of latent positions or adjacency matrices as canonical inferential
objects, often with consistency guarantees, but do not provide explicit
training-time capacity control. Gradient-based RDPG inference
\citep{fiori2023gradient} similarly infers realized rank from the optimizer
rather than controlling it, while simplex-constrained and latent-distance
extensions~\citep{nakis2022hmldmhybridmembershiplatentdistance,
nakis2023hierarchicalblockdistancemodel} add geometric interpretability but
retain a fixed nominal dimension. \textsc{Spectra} works in this lineage, but
uses an explicit kernel parameterization to expose capacity as a continuous
training-time coordinate. Profile-likelihood elbow rules~\citep{zhu2006automatic}, Bayesian
latent-dimension selection~\citep{passino2020bayesian}, network
cross-validation~\citep{chen2018network,li2020network}, structural information
criteria~\citep{gu2021principled,luo2021graph}, spectral parallel analysis
\citep{hong2025network}, and logarithmic search under metric latent-distance
models~\citep{nakis2025how} select a dimension from spectral, posterior, or
held-out diagnostics in a one-shot procedure. The operative coordinate is fixed
before or after training. \textsc{Spectra} instead shapes realized capacity
during training and measures it on the learned kernel spectrum. Factorized PSD models exhibit gradient-descent bias toward low rank
\citep{gunasekar2017implicit,arora2019implicit}, within a broader literature in
which effective complexity governs overparameterized generalization
\citep{belkin2019reconciling,bartlett2020benign}. Nuclear-norm
relaxations~\citep{candes2012exact,recht2010guaranteed,srebro2004maximum}
provide explicit but scale-coupled regularization, and stable-rank
normalization~\citep{sanyal2020stable} controls a Frobenius-to-spectral ratio.
By trace-normalizing the learned PSD kernel, \textsc{Spectra} separates scale
from rank, exposes capacity on the simplex of normalized eigenvalues, and yields
Eckart--Young--Mirsky-optimal low-rank PSD summaries
\citep{Eckart1936TheAO}. Shannon effective rank has appeared as a signal-processing spectral-entropy
measure~\citep{roy2007effective} and as the Vendi score for diversity
\citep{friedman2022vendi}; spectral entropy also diagnoses capacity in language
models~\citep{wei2024diff,jha2025spectral}, characterizes neural-network
training dynamics~\citep{yang2024effective}, and guides adaptive-rank
compression~\citep{cherukuri2025low}. Related alternatives include the
participation ratio, intrinsic dimension~\citep{li2018measuring}, stable
rank~\citep{ipsen2025stable}, and entropy functionals in determinantal point
processes~\citep{kulesza2012determinantal}. In these settings the spectrum is
typically read post hoc; \textsc{Spectra} uses Shannon effective rank as a
controllable training-time coordinate.
\section{Proposed Method}
\label{sec:method}

\begin{figure}
    \centering
    \includegraphics[scale=0.2]{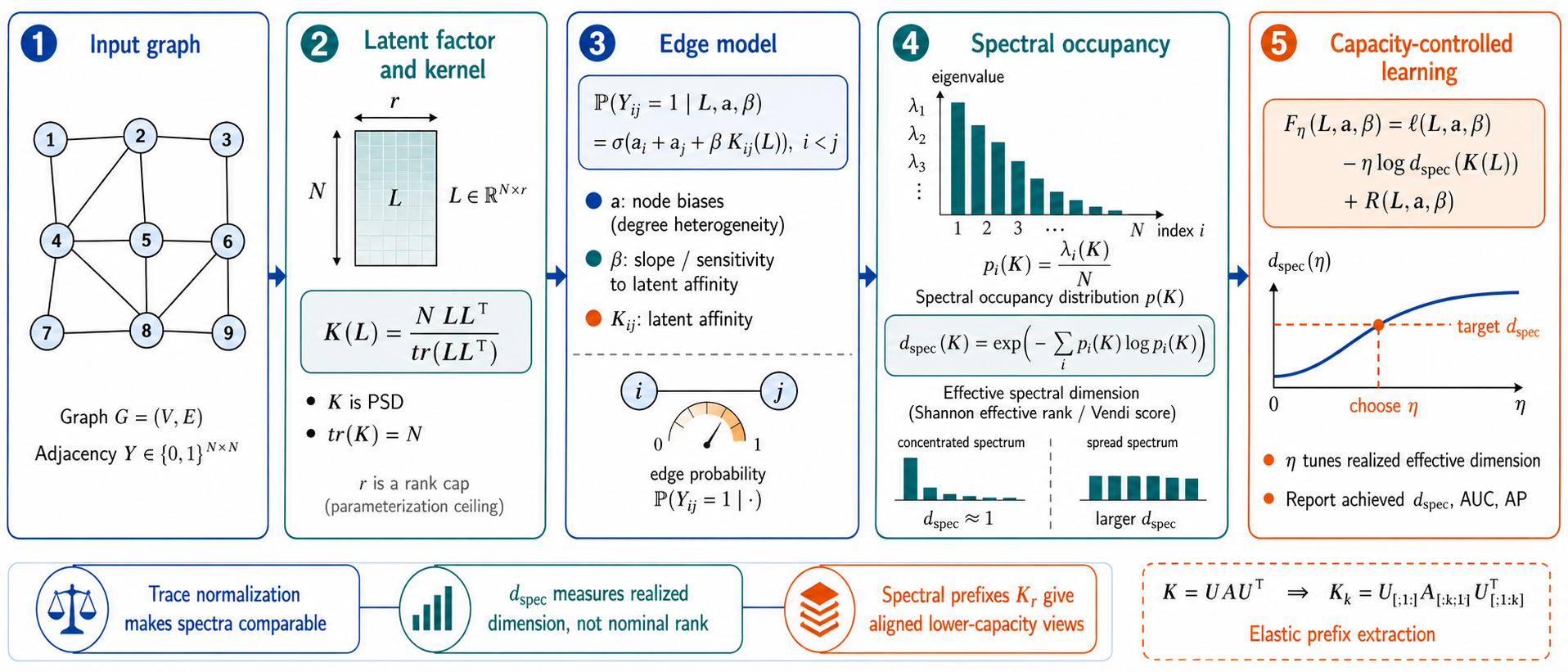}
\caption{Overview of \textsc{Spectra}. A rank-capped factor $L$ induces a
trace-normalized PSD kernel $K(L)$ used in the edge model. The normalized
spectrum of $K(L)$ defines spectral occupancy.}
    \label{fig:overview}
\end{figure}

\textbf{Preliminaries.}
Let $\mathcal{G}=(V,E)$ be a simple undirected graph with $N=|V|$ nodes and
adjacency matrix $\bm{Y}\in\{0,1\}^{N\times N}$, where
$\bm{Y}_{ij}=\bm{Y}_{ji}$ and $\bm{Y}_{ii}=0$. Our goal is to learn a latent
graph representation whose realized dimension can be measured from the fitted
model, rather than fixed in advance by a chosen factorization rank. We represent latent affinities between nodes by a positive semidefinite (PSD) kernel
$K\in\mathbb{R}^{N\times N}$. To compare the spectra of learned kernels across
runs, rank caps, and regularization strengths, we fix the total spectral mass:
\begin{equation}
    K \succeq 0,
    \qquad
    \operatorname{tr}(K)=N .
\end{equation}
This trace normalization removes arbitrary global rescaling of the kernel and
places all learned spectra on the same scale. As a result, differences in the
eigenvalue distribution of $K$ reflect how the representation allocates its
fixed spectral budget across latent modes, rather than differences in overall
magnitude. We parameterize the kernel by a rank-capped factor
$L\in\mathbb{R}^{N\times r}$:
\begin{equation}
\label{eq:trace-normalized-kernel}
    K(L)
    =
    N
    \frac{LL^\top}{\operatorname{tr}(LL^\top)},
    \qquad L\neq 0 .
\end{equation}
The integer $r$ is therefore a parameterization ceiling: it limits the maximum
rank available to the optimizer, but it does not define the dimension actually
used by the learned representation. The realized dimension will be measured from
the normalized spectrum of $K(L)$.

\textbf{Spectral Prefix Extraction and Capacity-Targeted Representation Analysis (\textsc{Spectra}).}
Given the trace-normalized kernel $K(L)$, node-specific offsets
$\bm{a}\in\mathbb{R}^N$, and a slope parameter $\beta>0$, we model each
undirected edge as
\begin{equation}
    \mathbb{P}(\bm{Y}_{ij}=1\mid L,\bm{a},\beta)
    =
    \phi\!\left(a_i+a_j+\beta K_{ij}(L)\right),
    \qquad i<j,
\end{equation}
where $\phi(t)=(1+\exp(-t))^{-1}$. The offsets $\bm{a}$ capture degree
heterogeneity: nodes with larger offsets have higher baseline propensity to form
edges. The kernel entry $K_{ij}(L)$ captures latent affinity between nodes $i$
and $j$. The slope $\beta$ controls the strength of this latent-affinity signal
on the log-odds scale: when $\beta$ is small, edge probabilities are driven
mostly by the node offsets, whereas larger $\beta$ makes differences in
$K_{ij}(L)$ more influential. Because the kernel trace is fixed, $\beta$ cannot
be absorbed into an arbitrary rescaling of $K$ and it has a well-defined role as the
global sensitivity to latent affinity.

\textbf{Spectral occupancy and effective spectral dimension.}
Let $\lambda_1(K)\geq \lambda_2(K)\geq\cdots\geq \lambda_N(K)\geq 0 $
be the eigenvalues of the trace-normalized kernel $K$. Since
$\operatorname{tr}(K)=N$, the normalized spectrum
\begin{equation}
\label{eq:spectral-occupancy}
    p_i(K)=\frac{\lambda_i(K)}{N}
\end{equation}
defines a probability distribution over spectral modes. We call
$p(K)=(p_1(K),\ldots,p_N(K))$ the \emph{spectral occupancy distribution}: it
records how the fixed spectral budget of $K$ is allocated across latent
directions.

\begin{definition}[Effective spectral dimension]
\label{def:effective-spectral-dimension}
The \emph{effective spectral dimension} of $K$ is
\begin{equation}
\label{eq:effective-spectral-dimension}
    d_{\mathrm{spec}}(K)
    :=
    \exp\!\left(
        -\sum_{i:p_i(K)>0}p_i(K)\log p_i(K)
    \right).
\end{equation}
Equivalently, $d_{\mathrm{spec}}(K)=\exp(H(p(K)))$, where $H$ is the Shannon
entropy. This quantity is the Shannon effective rank of the trace-normalized
kernel.
\end{definition}

The value $d_{\mathrm{spec}}(K)$ can be interpreted as the number of
equal-mass spectral modes that would have the same Shannon entropy as the
observed occupancy distribution $p(K)$. It equals one when the entire trace
budget is concentrated in a single mode. It equals $s$ when the trace budget is
spread uniformly over $s$ nonzero modes, in which case
$p_1(K)=\cdots=p_s(K)=1/s$ and $p_i(K)=0$ for $i>s$. Between these extremes,
$d_{\mathrm{spec}}(K)$ varies continuously with the spectrum and provides a
smooth effective-dimension coordinate for the learned representation. In
contrast, the hard rank
$\operatorname{rank}(K)=|\{i:p_i(K)>0\}|$ only counts the support size of the
spectrum and is sensitive to modes with arbitrarily small spectral mass.

\textbf{Capacity-control objective.}
The effective spectral dimension $d_{\mathrm{spec}}(K)$ is bounded above by the
hard rank of $K$, with equality attained when the spectral mass is
uniformly distributed. This maximum-entropy reference case motivates controlling
realized dimension through the entropy of the spectral occupancy distribution.
Higher entropy corresponds to a more diffuse allocation of the trace budget
across latent modes, while lower entropy corresponds to a more concentrated
allocation.

Let $\ell(L,\bm{a},\beta)$ denote the sampled link-prediction loss induced by
the latent-kernel edge model, and let $R(L,\bm{a},\beta)$ denote quadratic
regularization. Since
$ \log d_{\mathrm{spec}}(K(L)) = H(p(K(L))),$
we estimate the model by minimizing the entropy-regularized objective
\begin{equation}
\label{eq:map-objective}
    F_\eta(L,\bm{a},\beta)
    =
    \ell(L,\bm{a},\beta)
    - \eta \log d_{\mathrm{spec}}(K(L))
    + R(L,\bm{a},\beta).
\end{equation}
The scalar $\eta$ serves as a Lagrange-style control parameter governing the
tradeoff between predictive fit and spectral spread. Larger positive values of
$\eta$ place greater reward on spectral entropy and favor representations with
larger effective spectral dimension while smaller or negative values favor more
concentrated spectra and lower effective spectral dimension.

For each value of $\eta$, optimization yields a fitted factor
$L^\star(\eta)$ and a corresponding trace-normalized kernel $K^\star(\eta)=K(L^\star(\eta)).$
This induces the realized-dimension profile
\begin{equation}
\label{eq:capacity-profile}
    d_{\mathrm{spec}}(\eta)
    =
    d_{\mathrm{spec}}(K^\star(\eta)).
\end{equation}
We report fitted models in the interpretable coordinate
$d_{\mathrm{spec}}(K^\star(\eta))$, which gives the effective number of spectral
modes realized by the learned representation.

\textbf{Interpretation.}
Two kernels with the same hard rank can have very different
$d_{\mathrm{spec}}$, and conversely. The quantity summarizes the
shape of the occupancy distribution rather than only its support
size. Among spectral summaries, the participation ratio
$\mathrm{PR}(K) = 1/\sum_i p_i(K)^2$ is the exponential of the
second-order R\'enyi entropy and emphasizes high-mass modes through
a second-order moment, while thresholded ranks
$k_\tau(K) = |\{i : \lambda_i(K) \geq \tau \lambda_1(K)\}|$ depend
on a discontinuous cutoff. The Shannon effective rank in
$d_{\mathrm{spec}}$ is the unique choice that is continuous in the
occupancy distribution, mass-proportional, and recovers the hard
rank in the equimass limit, making it a smooth coordinate for
controlling and reporting realized dimension.

\textbf{Identifiability and optimality of spectral summaries.} We record two properties of the spectrum of the trace-normalized kernel. The first
shows that spectral prefixes are optimal low-rank summaries of a learned kernel and follows from Eckart-Young-Mirsky \cite{Eckart1936TheAO}.
The second shows that, under a simple  spectrum, the spectral modes are
identified up to sign and permutation.
We present the proofs in Appendices \ref{ap:thm:optimal-spectral-prefix} and \ref{ap:thm:identifiability-spectral-modes} for completeness.

\begin{theorem}[Optimal spectral prefix]\label{thm:optimal-spectral-prefix}
Let \(K \succeq 0\) be a symmetric matrix with rank $r$ and eigenvalues $\lambda_1\geq \lambda_2\geq \cdots \geq \lambda_r>0$ and eigenvectors $u_1,\ldots, u_r$ such that $K=\sum_{j=1}^r \lambda_j u_j u_j^\top$.
For any positive integer \(k\leq r\), define $K_k:=\sum_{j=1}^k \lambda_j u_j u_j^\top$.
Then \(K_k\) is a best rank-\(k\) positive semidefinite approximation of \(K\) in Frobenius norm, i.e., 
$
    K_k\in
    \arg\min_{\substack{M\succeq 0, \operatorname{rank}(M)\leq k}}
    \|K-M\|_F^2.
$
Moreover, if \(\lambda_k>\lambda_{k+1}\),  the optimizer is unique. 
\end{theorem}



\begin{theorem}[Identifiability of simple spectral modes]
\label{thm:identifiability-spectral-modes}
Let \(K\succeq 0\) be a symmetric matrix of rank $r$ with eigenvalues $\lambda_1\geq \lambda_2\geq \cdots \geq \lambda_r>0$ with $\lambda_i\neq \lambda_j$ for all $i\neq j$ and eigenvectors $u_1,\ldots, u_r$ such that $K=\sum_{j=1}^r \lambda_j u_j u_j^\top$.
Then, if
\[
    K
    =
    U\operatorname{diag}(\lambda_1,\ldots,\lambda_r)U^\top
    =
    U'\operatorname{diag}(\lambda'_1,\ldots,\lambda'_r){U'}^\top,
\]
with \(U^\top U={U'}^\top U'=I_r\), there exist a permutation matrix \(\Pi\) and a diagonal
 matrix \(S\), with diagonal entries in \(\{\pm 1\}\), such that
\(
    U'=U\Pi S,
    \ 
    \lambda'=\Pi^\top \lambda.
\)
\end{theorem}

\textbf{Implications for representation summaries.}
Theorem~\ref{thm:optimal-spectral-prefix} shows that the leading spectral modes provide
canonical low-rank summaries of the fitted kernel: truncating to the first $k$
eigenmodes is the optimal rank-$k$ PSD approximation in Frobenius norm. Thus,
spectral prefixes are not arbitrary coordinate projections of the factor $L$,
but intrinsic summaries of the kernel itself. Theorem~\ref{thm:identifiability-spectral-modes}
further shows that, when the active spectrum is simple, these spectral modes are
stable representation-level objects, identifiable up to sign and permutation.
Together, these results separate the parameterization from the representation:
the rank cap $r$ controls the largest available kernel rank, while
$d_{\mathrm{spec}}(K)$ measures the effective spectral dimension realized by the
learned kernel.

\textbf{Elastic prefix extraction.}
A fitted kernel also induces a nested family of aligned lower-dimensional representations, in the spirit of \emph{Matryoshka Representation Learning} \citep{kusupati2024matryoshkarepresentationlearning}, but obtained here without any multi-scale training objective. If
$K = U \Lambda U^\top$, define
\begin{equation}
    K_k = U_{1:k}\Lambda_{1:k}U_{1:k}^\top,
    \qquad k = 1, \ldots, r .
\end{equation}
The family $\{K_k\}_{k=1}^r$ is strictly nested, since
$\operatorname{range}(K_k) \subset \operatorname{range}(K_{k+1})$
because $K_{k+1}$ extends $K_k$ by one additional eigenpair without
modifying the previous $k$. Independently trained rank-$k$ models
do not satisfy this property: their eigenbases are unrelated across
$k$, blocking direct comparison of low- and high-capacity summaries
of the same data. Each $K_k$ is the optimal rank-$k$ PSD
approximation to $K$ in Frobenius norm by
Theorem~\ref{thm:optimal-spectral-prefix}, in contrast to MRL,
where nested representations are enforced through a joint loss at
training time, our nestedness and optimality at every $k$ are
deterministic consequences of the Eckart--Young--Mirsky theorem
applied to a single-objective fit. A single learned kernel thus
yields an auditable sequence of deployment-time operating points
indexed by $k$ and summarized by $d_{\mathrm{spec}}(K_k)$.

\textbf{Local behavior as $\eta$ varies.}
We next investigate the one-dimensional capacity path induced by varying the entropy weight
$\eta$. Because the objective may have multiple critical points, there is no globally single-valued optimizer map $\eta\mapsto L^\star(\eta)$ in general.
The result below is therefore local and branch-wise: it establishes a path of nondegenerate restricted local minima corresponding to values $\eta$ in the neighborhood of $\eta_0$. Along such
a branch, the realized dimension
\[
    d_{\mathrm{spec}}(\eta)
    =
    d_{\mathrm{spec}}\bigl(K(L^\star(\eta))\bigr)
\]
varies smoothly and is locally nondecreasing, provided no spectral or
second-order degeneracy occurs.

The statement must quotient out the orthogonal gauge of the factorization.
Indeed, $K(LQ)=K(L)$ for every orthogonal $Q\in\mathbb R^{r\times r}$, so the
full parameterization contains flat directions unrelated to the represented
kernel. Fix a rank $s>0$. After a right-orthogonal
reparameterization, work on the slice
\[
    \widetilde{\Theta}
        \coloneqq
        \left\{
            ([L_1\mid 0],a,\beta):
            L_1\in\mathbb R^{N\times s},
            \operatorname{rank}(L_1)=s,\,
            L_1^\top L_1 \text{ is diagonal},\,
            a\in\mathbb R^N,\,
            \beta>0
        \right\}
        \subset \Theta .
\]
This slice is a smooth embedded submanifold of the parameter space, and all
derivatives below are restricted to its tangent spaces.\footnote{Basic differential-geometric definitions are recalled in Appendix~\ref{app:dg-definitions}.\label{fn:definitions}} Its tangent space at $\theta\in\widetilde{\Theta}$ is denoted by
$T_\theta\widetilde{\Theta}$. The restricted objective
$F_\eta|_{\widetilde{\Theta}}$ is $\mathcal C^2$; see
Appendix~\ref{app:lemma}. We denote the restricted gradient by
$
    \nabla_\theta F_\eta(\theta)\big|_{T_\theta\widetilde{\Theta}}
$
and the Hessian of the restricted function $F_\eta|_{\widetilde{\Theta}}$, as
$
    \nabla_\theta^2 F_\eta(\theta)\big|_{T_\theta\widetilde{\Theta}}
$ evaluated on tangent directions.

Fix $\eta_0\in\mathbb R$. We take as reference an arbitrary restricted critical point
$
    \theta^\star=(L^\star,a^\star,\beta^\star)\in\widetilde{\Theta}$
of $F_{\eta_0}|_{\widetilde{\Theta}}$, meaning that
$
    \nabla_\theta F_{\eta_0}(\theta^\star)
    \big|_{T_{\theta^\star}\widetilde{\Theta}}
    =
    0.
$
The following assumptions are local regularity conditions for the chosen
restricted critical point $\theta^\star$. If a rank transition, eigenvalue crossing, or Hessian
degeneracy occurs, this branch-wise smoothness guarantee no longer applies.

\begin{assumption}[Simple active spectrum]
\label{ass:A2}
The positive eigenvalues of \(K(L^{*})\) are pairwise distinct.
\end{assumption}

\begin{assumption}[Nondegenerate second-order condition]
\label{ass:A3}
The restricted Hessian is positive definite:
\[
    v^\top
    \left(
        \nabla_\theta^2F_{\eta_0}(\theta^{*})
        \big|_{T_{\theta^{*}}\widetilde{\Theta}}
    \right)
    v
    >
    0
    \quad
    \text{for every }
    0\neq v\in T_{\theta^{*}}\widetilde{\Theta}.
\]
\end{assumption}


The proof of the following Theorem is in Appendix \ref{app:proof-regularity} and an accompanying illustration in Figure \ref{fig:thm6}. 

\begin{theorem}[Local smoothness and monotonicity of the effective spectral dimension]
\label{thm:frontier-regularity}
Let \(\eta_0\in\mathbb R\), and let $\theta^\star=(L^\star,a^\star,\beta^\star)$ satisfy Assumptions~\ref{ass:A2}--\ref{ass:A3}.
Then, there exists an open interval \(U\ni\eta_0\) and a \(\mathcal{C}^1\) map $\theta^\star\colon \R\to\widetilde{\Theta}$ such that $\theta^\star(\eta_0)=\theta^\star$ and the following three claims hold: 
\textbf{(i)} for every $\eta\in U$, it holds
    \(
        \nabla_\theta F_\eta(\theta^\star(\eta))
        \big|_{T_{\theta^\star(\eta)}\tilde\Theta}
        =
        0
    \)
    and \(\theta^\star(\eta)\) is a strict local minimum of
    \(F_\eta\big|_{\widetilde{\Theta}}\), 
\textbf{(ii)} the capacity profile
    \(
        d_{\mathrm{spec}}(\eta)
        :=
        d_{\mathrm{spec}}\bigl(K(L^\star(\eta))\bigr)
    \)
    is \(\mathcal{C}^1\) on \(U\) and,
\textbf{(iii)} for every \(\eta\in U\),
    \[\frac{d}{d\eta}\log d_{\mathrm{spec}}(\eta)\ge 0\]
\end{theorem}

The theorem is intentionally local. Its hypotheses can fail when - as $\eta$ varies - an eigenvalue
enters or leaves the active support, when positive eigenvalues collide, or when
the restricted Hessian becomes degenerate. These events can produce nonsmooth spectrum
changes (see Appendix~\ref{app:structural-events}, Fig.~\ref{fig:frontier-allfour}). The appropriate interpretation of Theorem~6 is a sensitivity statement for a
chosen nondegenerate local optimum, rather than a global statement about the
optimizer set.  In nonconvex problems the argmin correspondence
$\eta \mapsto \operatorname{argmin} F_\eta$ can be set-valued and need not vary
continuously.  The theorem therefore proves the property that is actually
needed locally: whenever a fitted solution is nondegenerate and its active
spectrum remains simple, the implicit-function theorem selects a unique nearby
branch of local minima, and along this branch increasing the entropy weight
cannot decrease $\log d_{\mathrm{spec}}$.  Thus the result justifies local
continuation and local one-dimensional calibration of realized capacity, but it
does not assert that all local minima, or the global minimizer, lie on a single
smooth monotone path.

\textbf{Targeting a desired effective dimension.} Motivated by the local monotonicity in Theorem~6, we target a desired effective
dimension by treating \(d_{\mathrm{spec}}(\eta)\) as a one-dimensional response
curve on empirically identified monotone intervals. The entropy weight $\eta$ parameterizes a one-dimensional family of
fitted kernels: $\eta>0$ raises $d_{\mathrm{spec}}$, $\eta<0$ lowers
it. Given a target $d_{\mathrm{spec}}^\star$, we solve
$d_{\mathrm{spec}}(\eta) \approx d_{\mathrm{spec}}^\star$ as
one-dimensional root finding for
$g(\eta) := d_{\mathrm{spec}}(\eta) - d_{\mathrm{spec}}^\star$.
Starting from an anchor fit at $\eta=0$, the sign of $g(0)$ selects
the search direction; geometric expansion brackets the target and
bisection refines $\eta$ until $|g(\eta)|/d_{\mathrm{spec}}^\star
\leq \tau$ (we use $\tau=0.02$). A monotone guard during expansion
falls back to bidirectional search if local monotonicity fails, and
probe fits may early-stop once $d_{\mathrm{spec}}$ stabilizes,
since exact convergence is unnecessary for bracketing. The selected
$\eta^\star$ is retrained from a fresh initialization at the full
training budget. By Theorem~\ref{thm:frontier-regularity}, bisection
reaches $\eta$-precision $\delta$ in $O(\log(1/\delta))$ probe
evaluations on intervals where the branch is monotone, so calibration
adds at most a logarithmic multiplicative factor over a single
training run.

\textbf{Complexity and parametrization.} We parametrize the kernel through its SVD,
$L = Q\,\mathrm{diag}(\sigma)$, with $Q$ column-orthonormal and
$\sigma$ positive via softplus. This is equivalent to a free
factor under $K(L)=N\,LL^\top/\mathrm{tr}(LL^\top)$ but reads the
eigendecomposition of $K$ directly off the parameters
($\lambda_j(K) = N\sigma_j^2/\sum_l\sigma_l^2$, eigenvectors $Q$),
making $d_{\mathrm{spec}}$ and prefixes $K_k$ available without an
eigensolve and exposing the signed-permutation gauge of
Theorem~\ref{thm:identifiability-spectral-modes} at the parameter
level. Each training step costs
$\mathcal O(|E_{\mathrm{batch}}|\, r + N r)$ with $\mathcal O(Nr)$
memory; total cost is
$\mathcal{O}(T(|E_{\mathrm{batch}}|\,r + Nr))$, linear in $N$ and
$r$ and tractable for $r \ll N$. Calibration adds a logarithmic
factor.
\section{Experiments \& Results}
\label{sec:experiments}

\begin{table}[t]
\centering
\caption{Dataset statistics for eight undirected graphs used in the experiments.}
\label{tab:dataset-stats}
\resizebox{\textwidth}{!}{
\begin{tabular}{lrrrrrrrr}
\toprule
Statistic 
& ca-GrQc 
& ca-HepTh 
& socfb-American75 
& socfb-Amherst41 
& bio-grid-human 
& bio-grid-worm 
& inf-power 
& inf-openflights \\
\midrule
Domain 
& Collaboration 
& Collaboration 
& Social 
& Social 
& Biological 
& Biological 
& Infrastructure 
& Infrastructure \\
Nodes 
& 4{,}158 & 8{,}638 & 6{,}370 & 2{,}235 & 9{,}186 & 3{,}343 & 4{,}941 & 2{,}905 \\
Edges 
& 13{,}426 & 24{,}818 & 224{,}024 & 93{,}188 & 40{,}224 & 9{,}780 & 11{,}534 & 18{,}550 \\
\bottomrule
\end{tabular}
}
\end{table}

We evaluate \textsc{Spectra} against representative latent
representation methods spanning Euclidean embeddings, matrix
factorization, mixed-membership models, and latent-distance
simplex-based representations. \textsc{Spectra} is optimized with
Adam~\citep{kingma2014adam} at $5{:}1$ negatives-to-positives
sampling per step. For embedding baselines, dyadic edge features
use the binary operator (average, Hadamard, weighted-$L_1$,
weighted-$L_2$) maximizing each baseline's test AUC, giving each
baseline its strongest configuration; likelihood-based latent
graph models, including \textsc{Spectra}, are evaluated directly
from the learned log-odds. \textsc{Spectra} is reported in a
strictly test-blind regime: no hyperparameter, $\eta$ value,
$d_{\mathrm{spec}}$ target, or checkpoint is selected with
reference to any test quantity at any stage.
Hyperparameters, hardware, and implementation pointers are in
Appendix \ref{sec:supp-experimental}.

\begin{figure}[t]
  \centering

  \noindent
  \begin{minipage}[c]{0.03\linewidth}
    \centering
    \rotatebox{90}{\itshape\small Saturated regime}
  \end{minipage}\hfill
  \begin{minipage}[c]{0.96\linewidth}
    \begin{minipage}[c]{0.245\linewidth}\centering
      \textsl{inf-power}~{\footnotesize\textcolor{gray}{$\cdot$ infrastructure}}
    \end{minipage}\hfill
    \begin{minipage}[c]{0.245\linewidth}\centering
      \textsl{inf-openflights}~{\footnotesize\textcolor{gray}{$\cdot$ infrastructure}}
    \end{minipage}\hfill
    \begin{minipage}[c]{0.245\linewidth}\centering
      \textsl{socfb-American75}~{\footnotesize\textcolor{gray}{$\cdot$ social}}
    \end{minipage}\hfill
    \begin{minipage}[c]{0.245\linewidth}\centering
      \textsl{socfb-Amherst41}~{\footnotesize\textcolor{gray}{$\cdot$ social}}
    \end{minipage}

    \vspace{2pt}

    \includegraphics[width=0.245\linewidth]{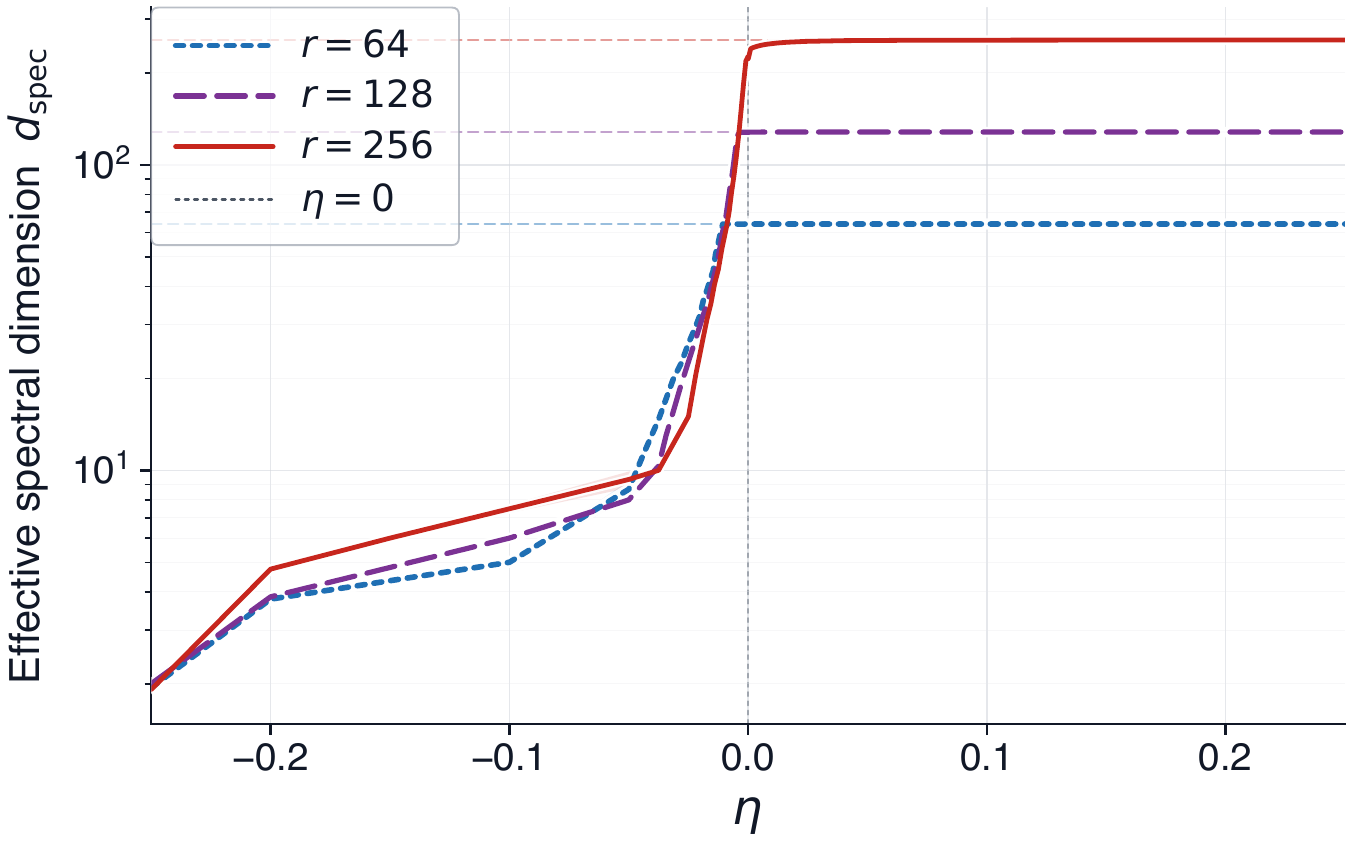}\hfill
    \includegraphics[width=0.245\linewidth]{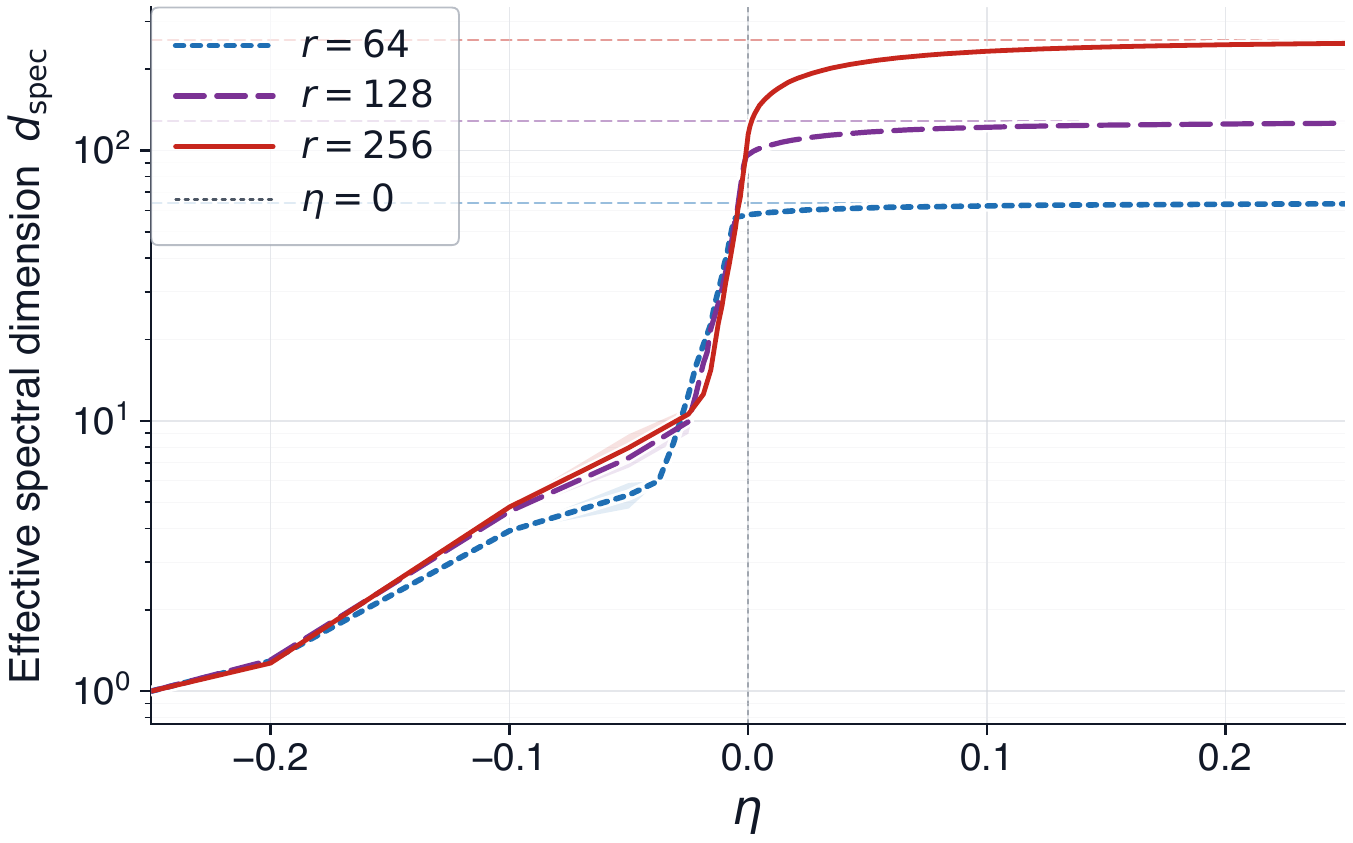}\hfill
    \includegraphics[width=0.245\linewidth]{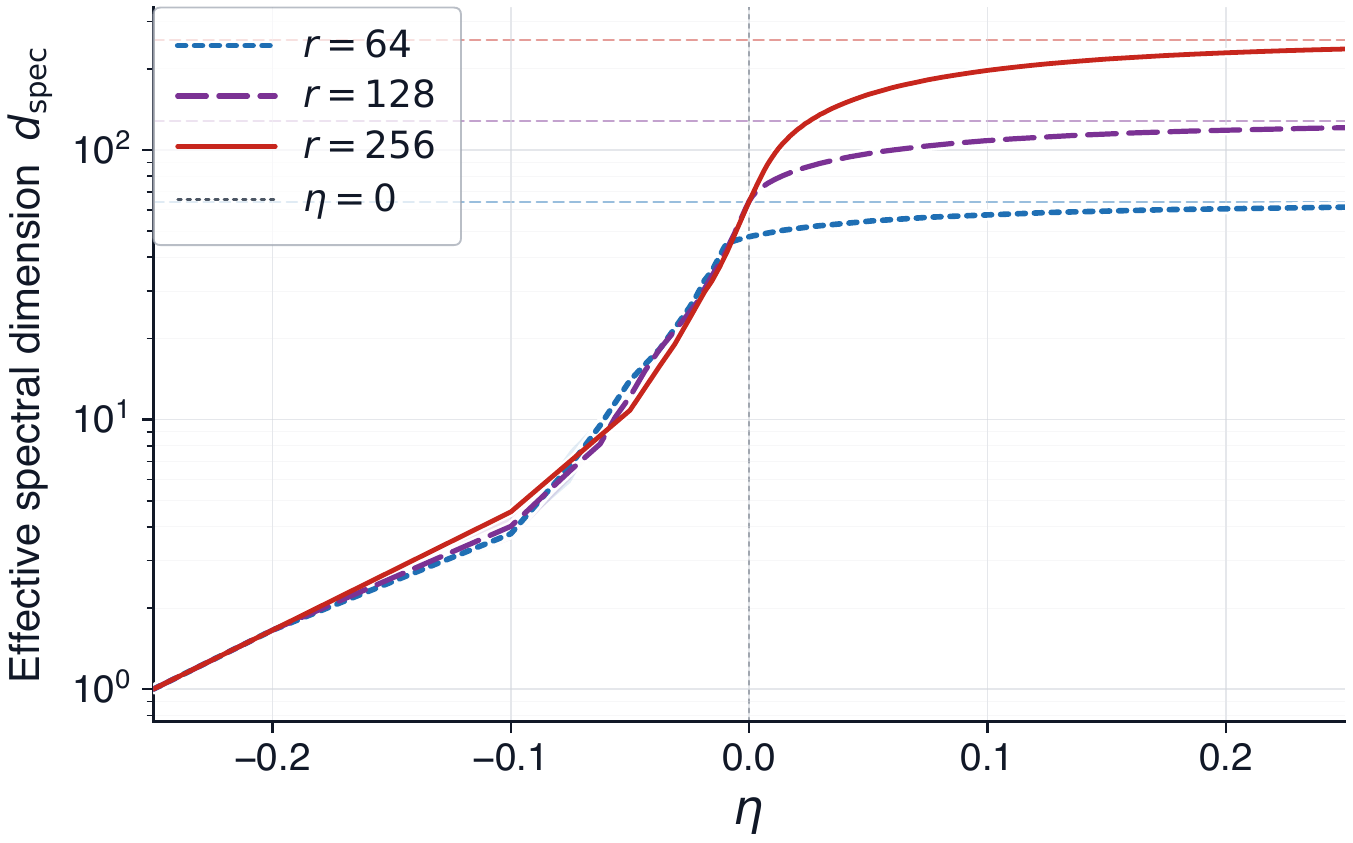}\hfill
    \includegraphics[width=0.245\linewidth]{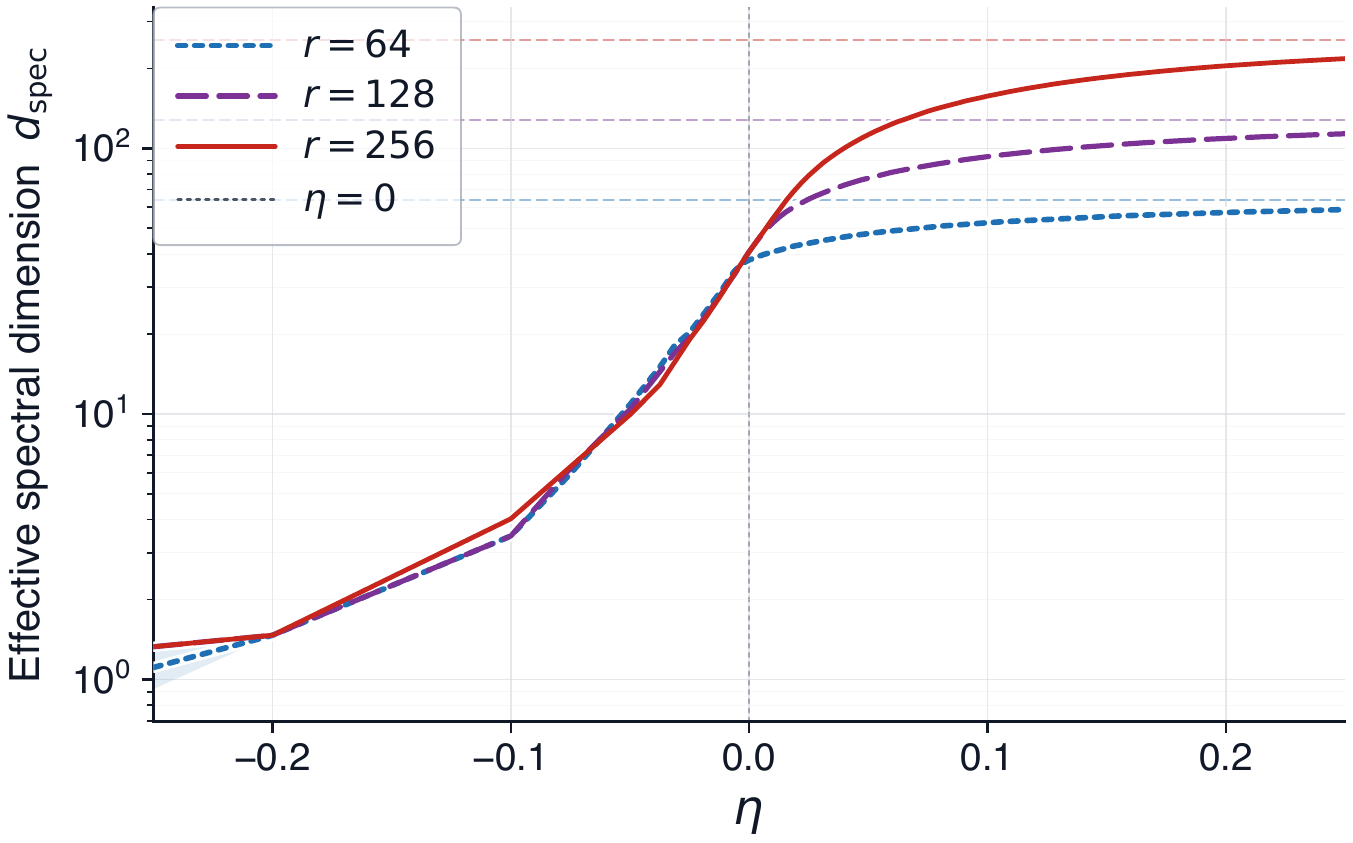}

    \vspace{1pt}

    \includegraphics[width=0.245\linewidth]{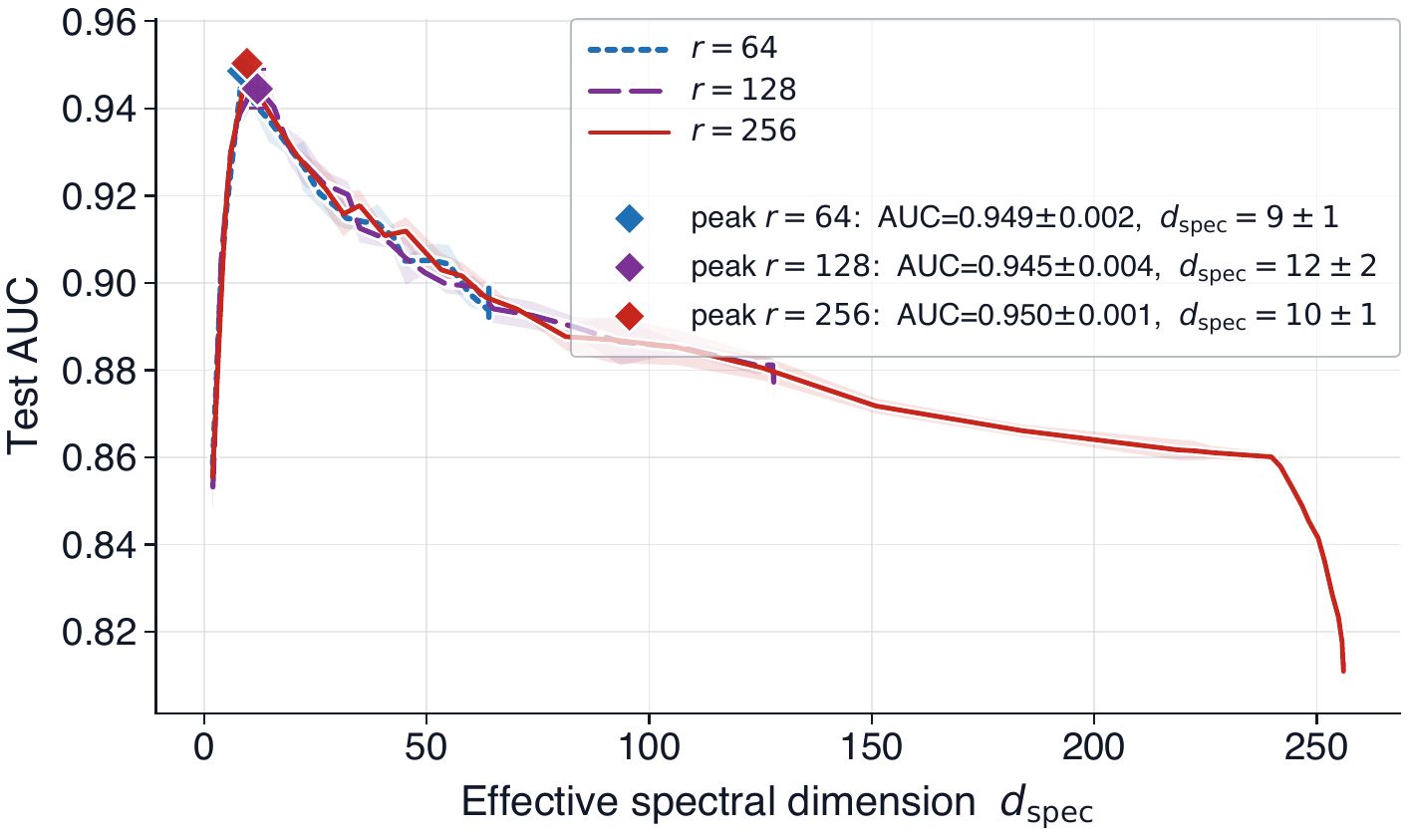}\hfill
    \includegraphics[width=0.245\linewidth]{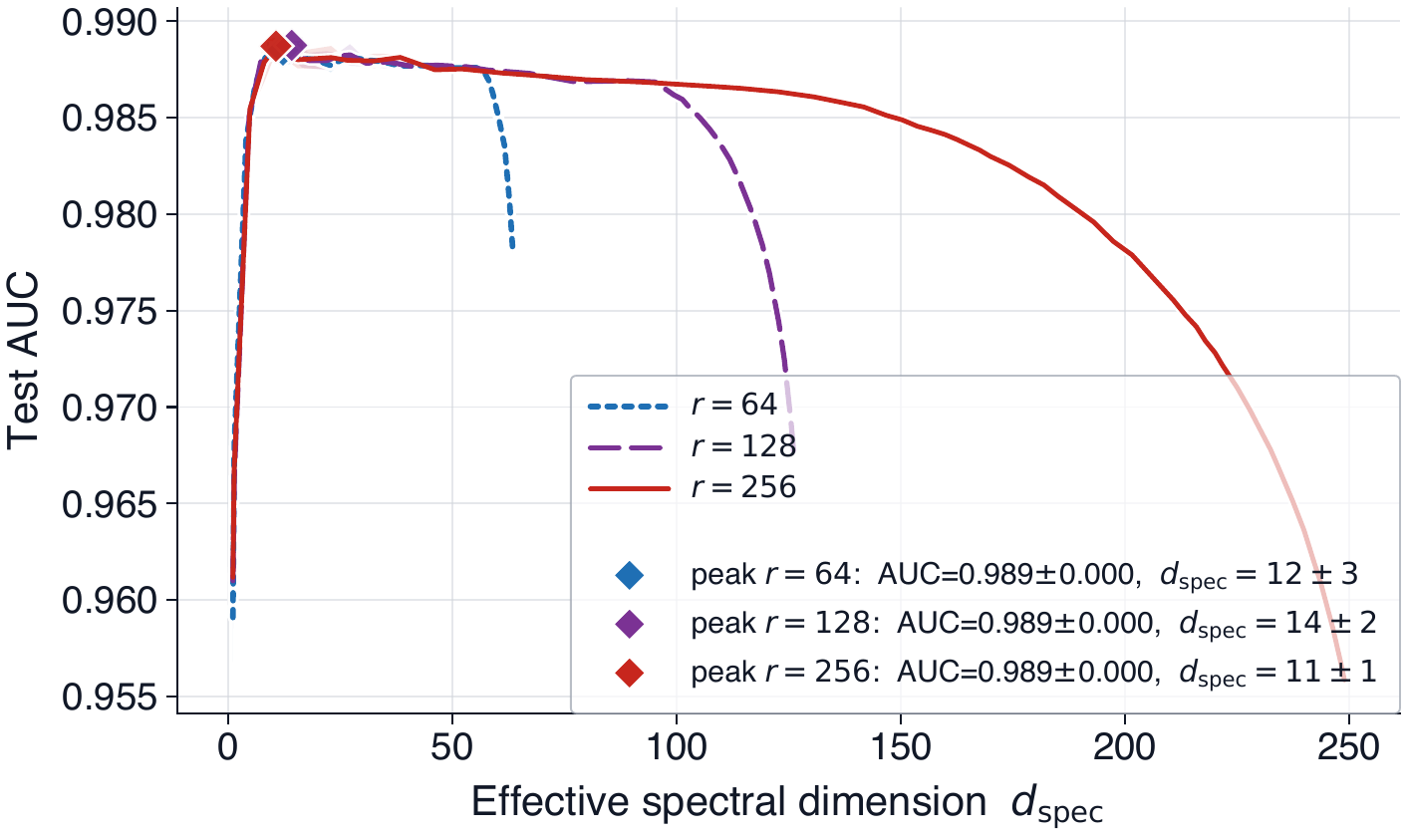}\hfill
    \includegraphics[width=0.245\linewidth]{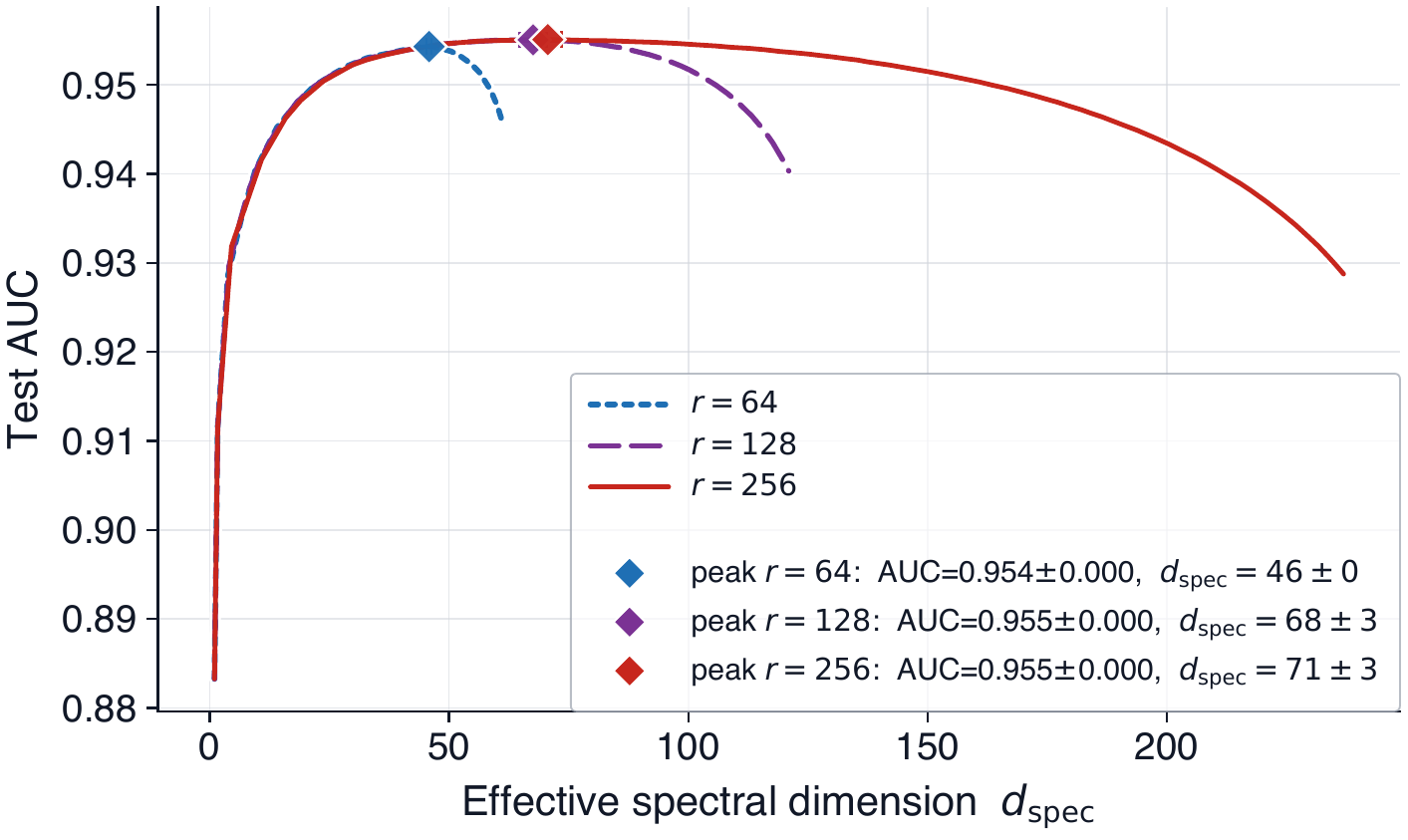}\hfill
    \includegraphics[width=0.245\linewidth]{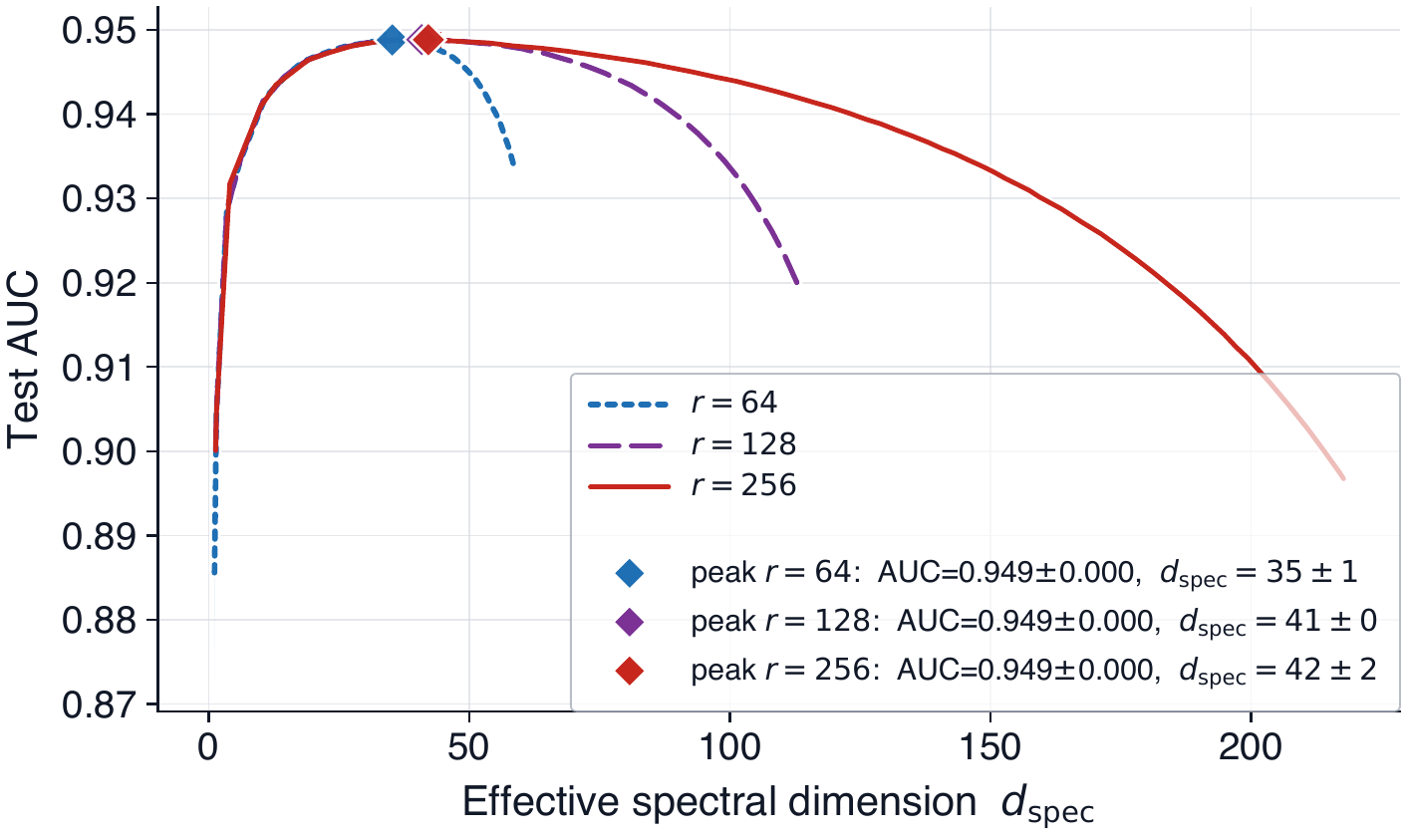}
  \end{minipage}

  \vspace{8pt}

  \noindent
  \begin{minipage}[c]{0.03\linewidth}
    \centering
    \rotatebox{90}{\itshape\small Rank-cap-binding regime}
  \end{minipage}\hfill
  \begin{minipage}[c]{0.96\linewidth}
    \begin{minipage}[c]{0.245\linewidth}\centering
      \textsl{ca-grqc}~{\footnotesize\textcolor{gray}{$\cdot$ citation}}
    \end{minipage}\hfill
    \begin{minipage}[c]{0.245\linewidth}\centering
      \textsl{ca-hepth}~{\footnotesize\textcolor{gray}{$\cdot$ citation}}
    \end{minipage}\hfill
    \begin{minipage}[c]{0.245\linewidth}\centering
      \textsl{bio-grid-human}~{\footnotesize\textcolor{gray}{$\cdot$ biological}}
    \end{minipage}\hfill
    \begin{minipage}[c]{0.245\linewidth}\centering
      \textsl{bio-grid-worm}~{\footnotesize\textcolor{gray}{$\cdot$ biological}}
    \end{minipage}

    \vspace{2pt}

    \includegraphics[width=0.245\linewidth]{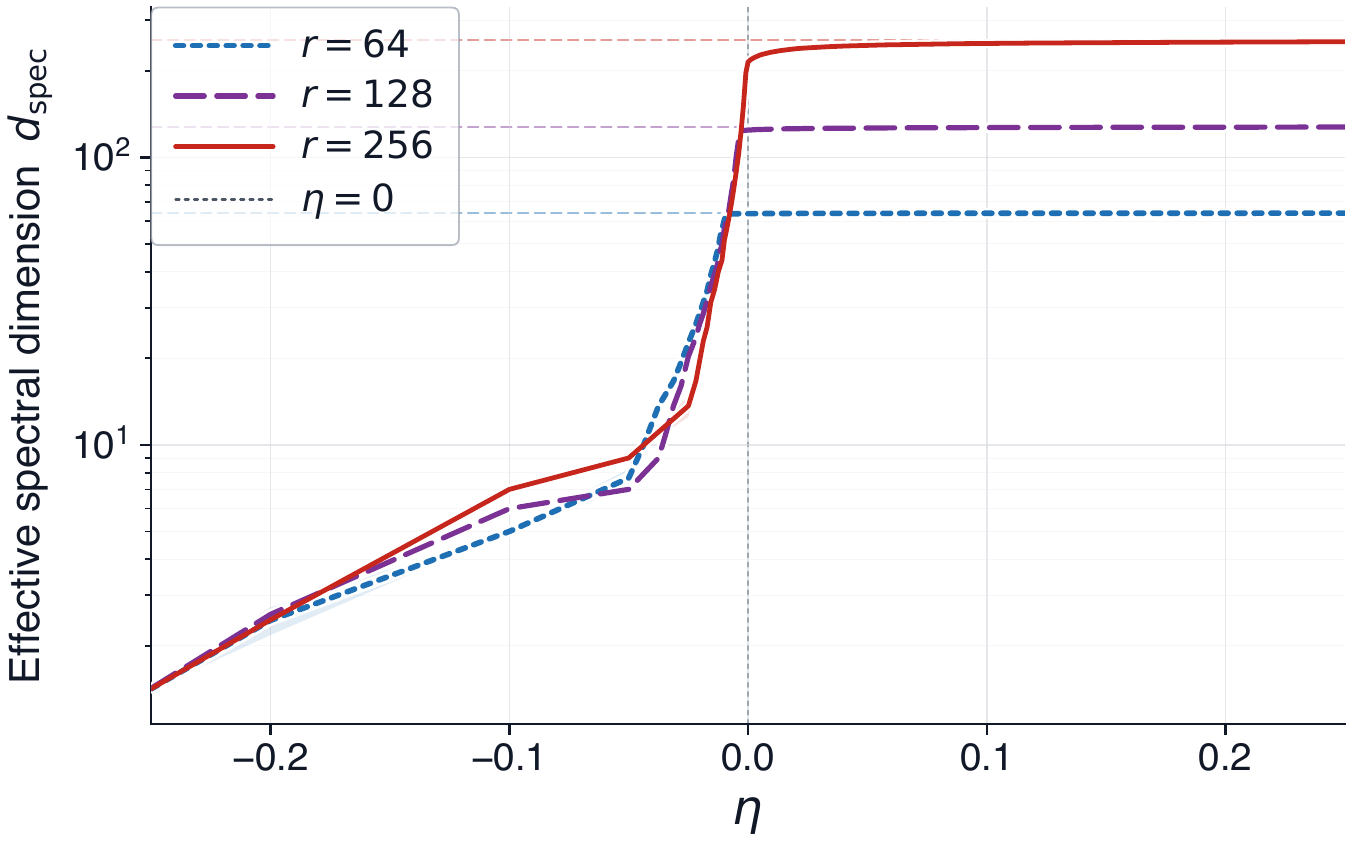}\hfill
    \includegraphics[width=0.245\linewidth]{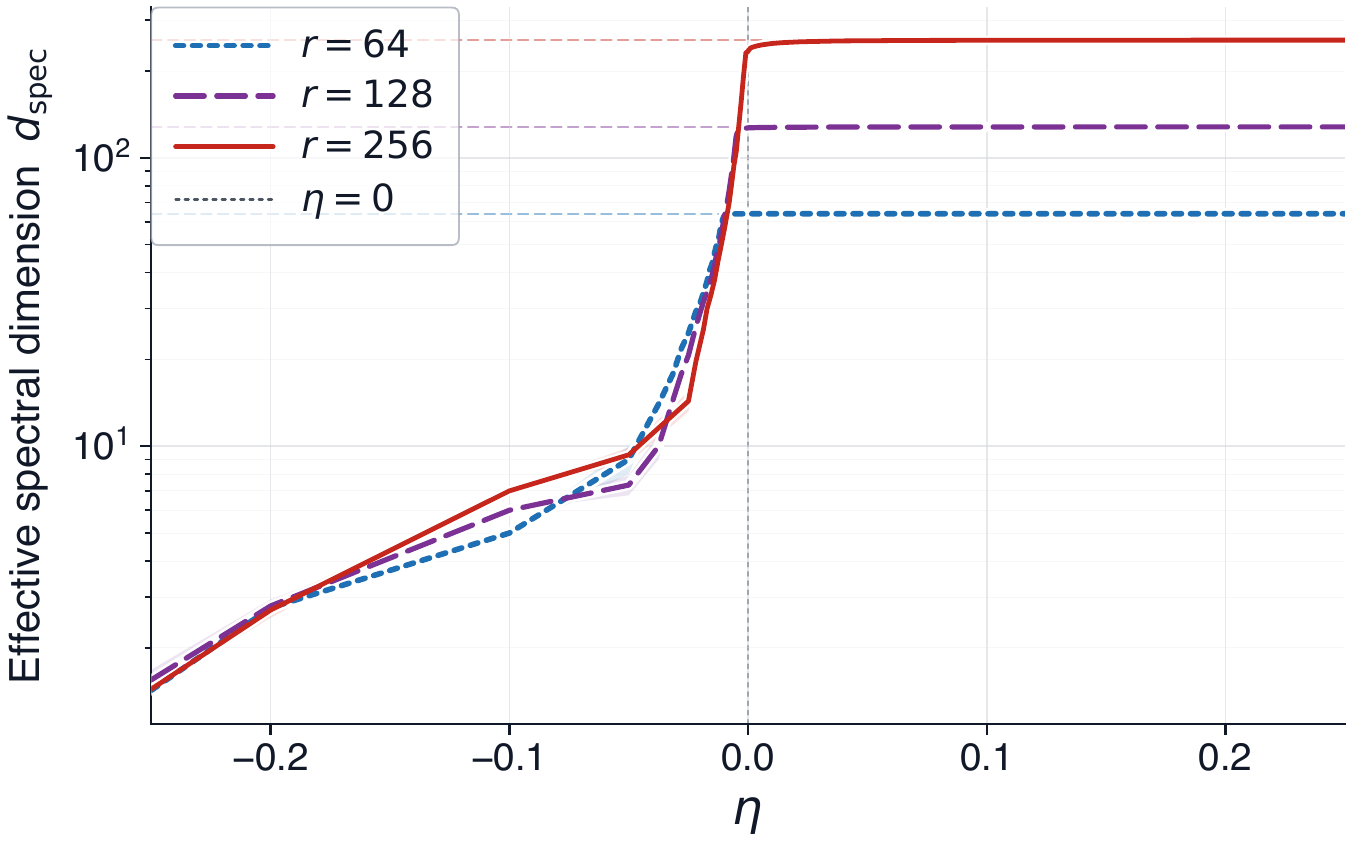}\hfill
    \includegraphics[width=0.245\linewidth]{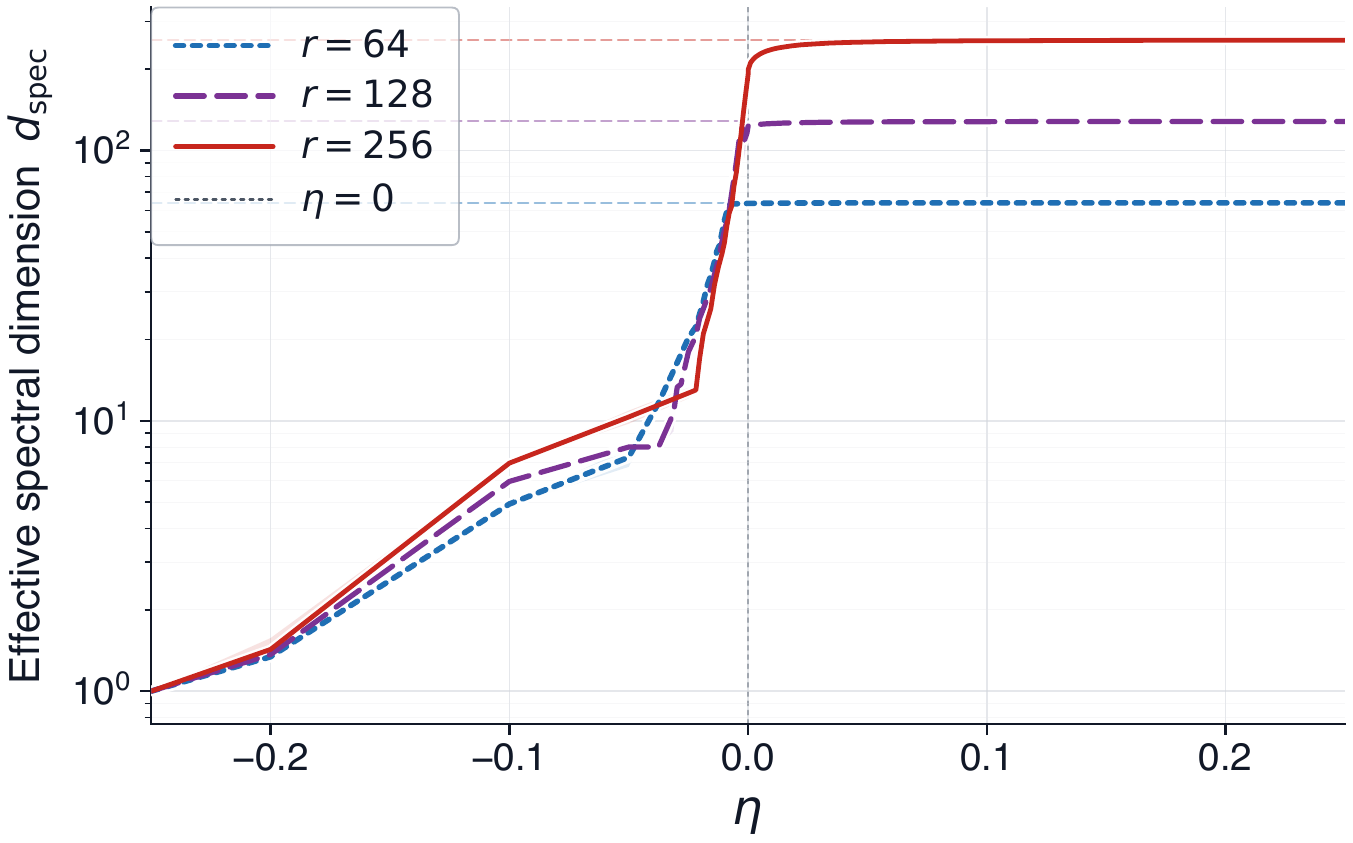}\hfill
    \includegraphics[width=0.245\linewidth]{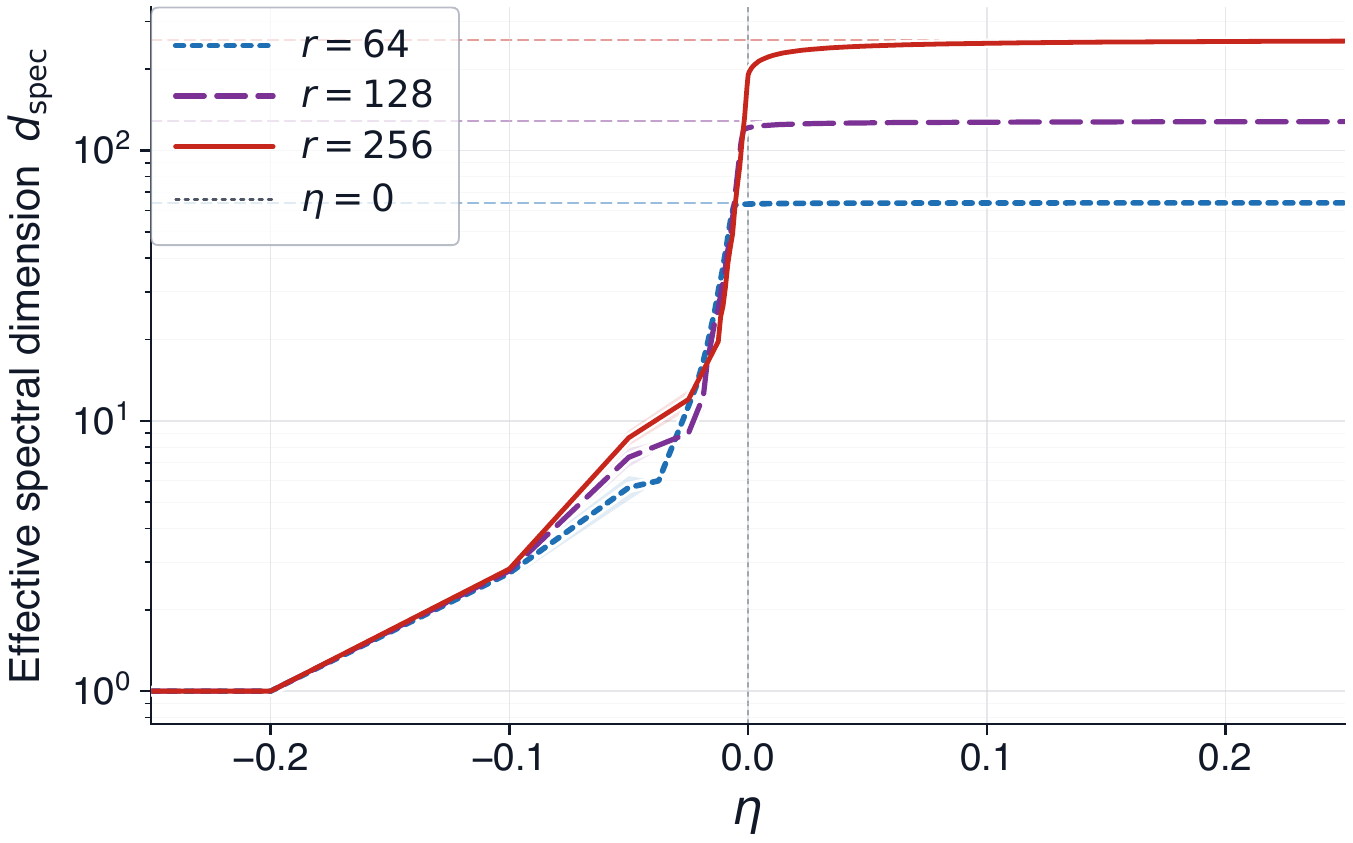}

    \vspace{1pt}

    \includegraphics[width=0.245\linewidth]{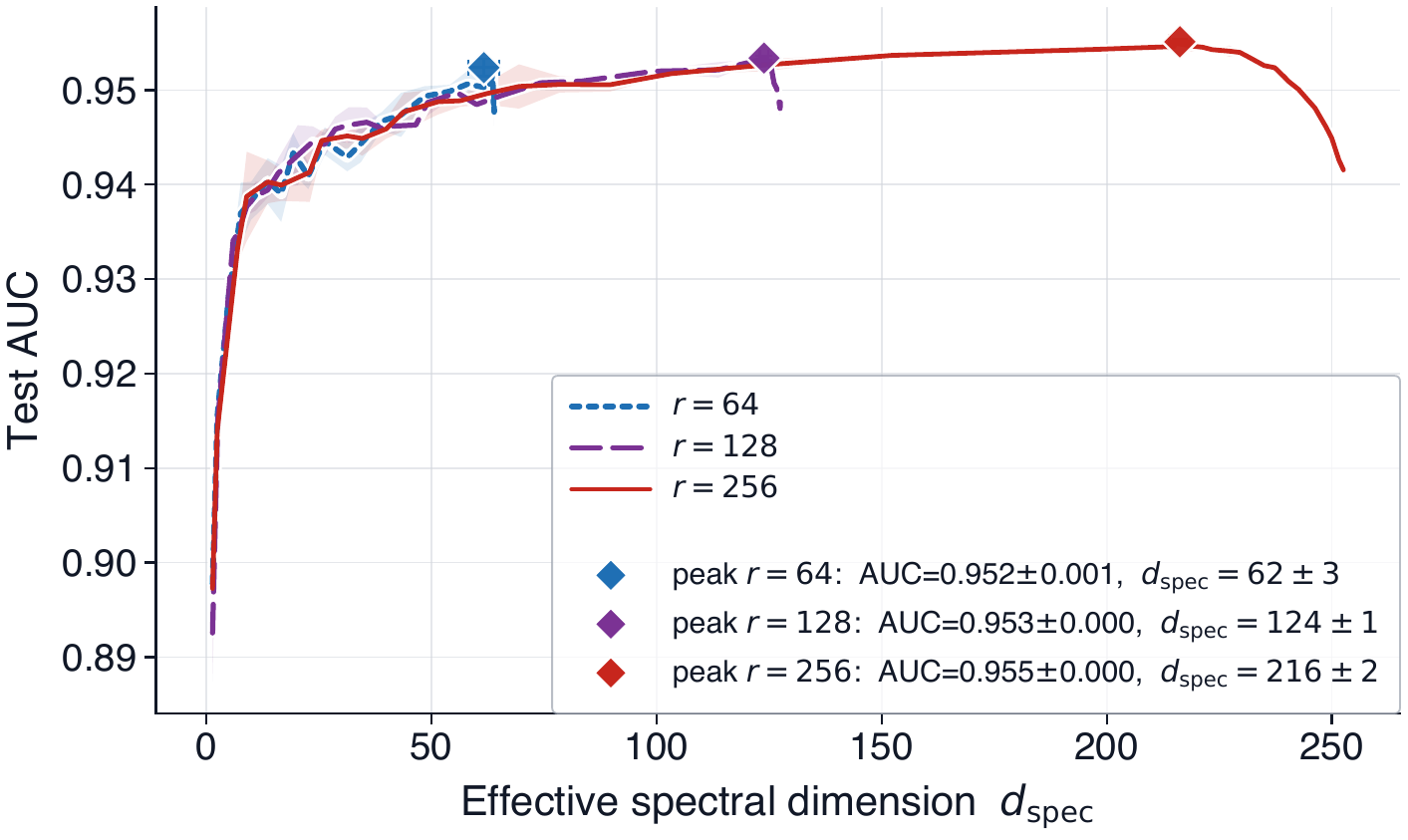}\hfill
    \includegraphics[width=0.245\linewidth]{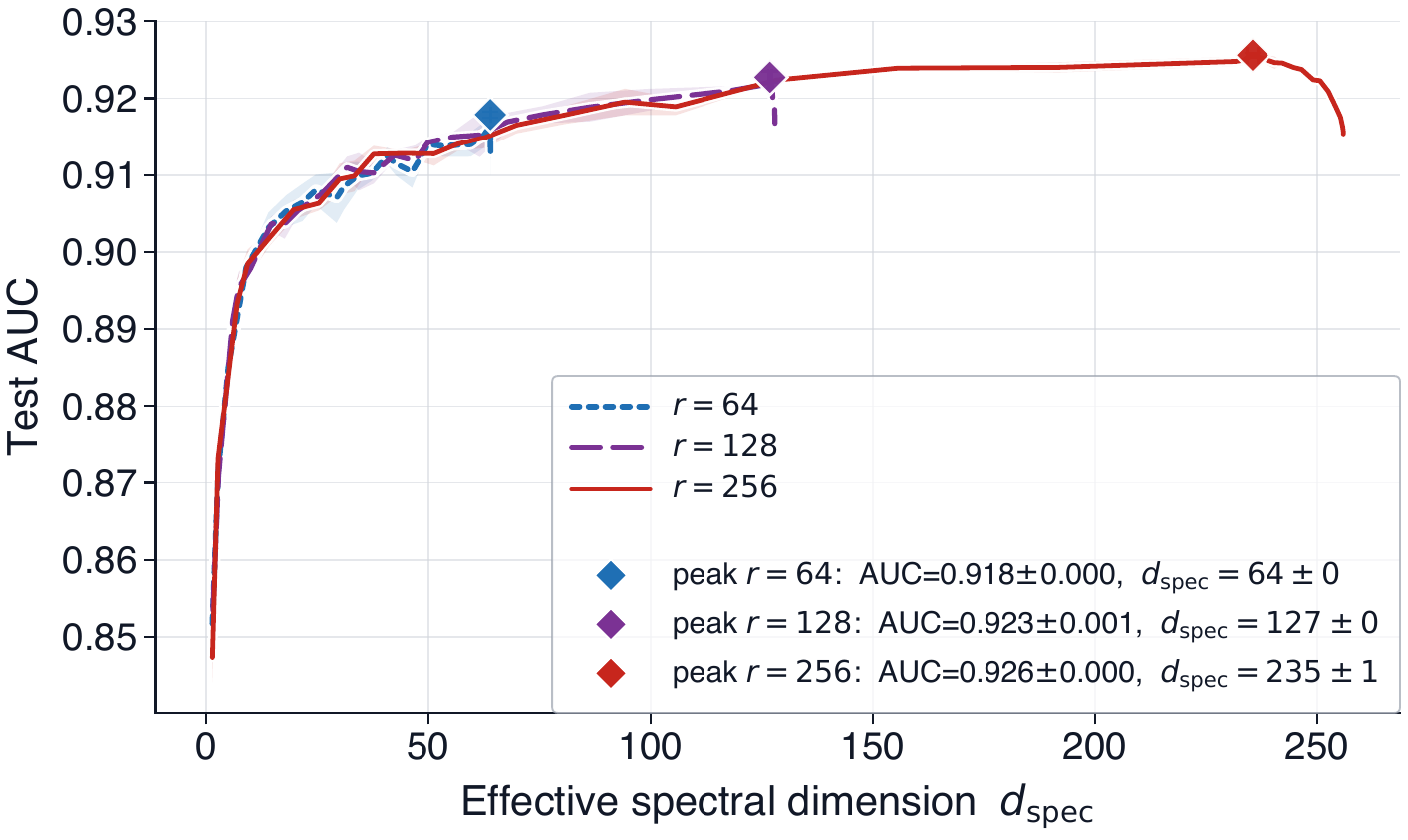}\hfill
    \includegraphics[width=0.245\linewidth]{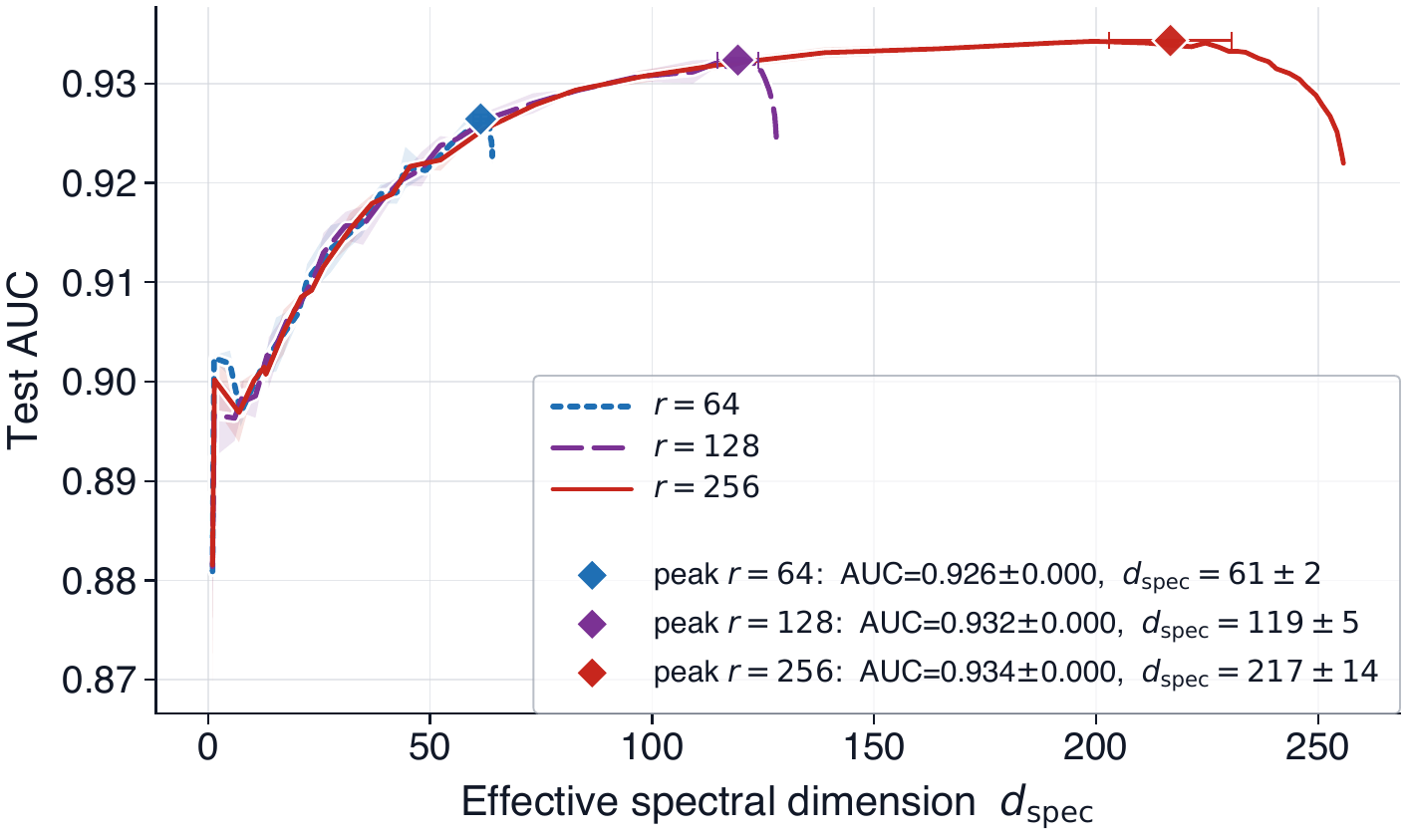}\hfill
    \includegraphics[width=0.245\linewidth]{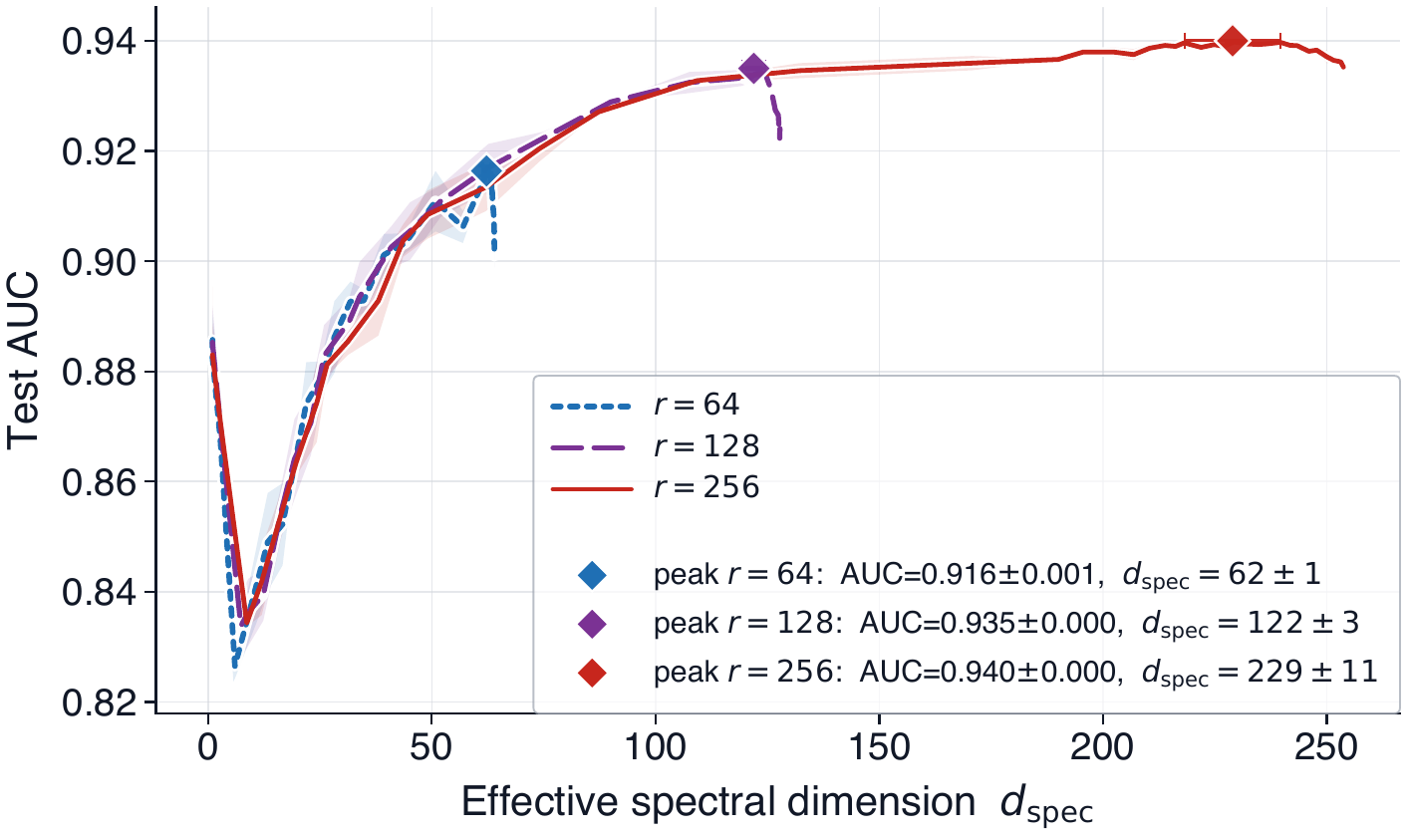}
  \end{minipage}

\caption{\textbf{Capacity frontiers across rank caps.}
Each dataset is swept over $\eta\in[-0.25,0.25]$ at
$r\in\{64,128,256\}$; curves show mean$\pm1\sigma$ across three seeds. Upper
panels show achieved $d_{\mathrm{spec}}$ versus $\eta$, lower panels show test
AUC versus achieved $d_{\mathrm{spec}}$. Top: saturated datasets. Bottom:
rank-cap-binding datasets. Legends report per-rank peak coordinates.}
  \label{fig:frontier-multirank-all}
\end{figure}

\textbf{Datasets and Baselines.}
All experiments use featureless undirected graphs spanning five
structural regimes: collaboration networks
\textsl{ca-GrQc} and \textsl{ca-HepTh}~\citep{leskovec2007graph};
social networks \textsl{socfb-American75} and
\textsl{socfb-Amherst41}~\citep{fb}; biological protein-protein interaction
networks \textsl{bio-grid-human} and \textsl{bio-grid-worm}~\citep{biogrid};
and infrastructure networks \textsl{inf-power}~\citep{power} and
\textsl{inf-openflights}~\citep{airport} (see Table \ref{tab:dataset-stats} for details). Datasets are obtained from
SNAP~\citep{snapnets} and the Network Repository~\citep{nr}. We compare against four families of latent representation methods:
(i)~\emph{Euclidean vector-space embeddings} that learn node representations
in $\mathbb{R}^d$ under inner-product geometry, represented by
\textsc{Node2Vec}~\citep{grover2016node2vec} and
\textsc{Role2Vec}~\citep{role2vec};
(ii)~\emph{matrix-factorization methods}, represented by
\textsc{NetMF}~\citep{netmf} and \textsc{GraRep}~\citep{grarep};
(iii)~\emph{mixed-membership models}, represented by
\textsc{MNMF}~\citep{MNMF};
and (iv)~\emph{latent-distance simplex-based models}, represented by
\textsc{HM-LDM}~\citep{nakis2022hmldmhybridmembershiplatentdistance}. We focus on featureless graphs in an unsupervised link-prediction setting; the comparison class is therefore the latent-representation methods that operate in this regime, rather than message-passing GNNs, which usually underperform under this regime \cite{nikolentzos2024gnnsactuallylearnunderstanding}.

\paragraph{Empirical tracking protocol.}
Theorem \ref{thm:frontier-regularity} provides a local sensitivity result for nondegenerate branches of
local minima.  In the experiments, we complement this branch-wise theory with a
more stringent protocol: each value of $\eta$ is trained from independent random
initializations rather than obtained by warm-start continuation from a single
solution.  Stability of achieved $d_{\mathrm{spec}}$ across seeds therefore
tests whether spectral capacity control is reproducible under the actual
nonconvex training procedure, not only along a locally tracked branch.

\textbf{Capacity Frontiers.}
We test whether predictive performance is better organized by realized spectral
capacity $d_{\mathrm{spec}}(\eta)$ than by the nominal rank cap. For each
dataset, we sweep $\eta$ at rank caps $r\in\{64,128,256\}$ over three seeds,
using an adaptive cold-start grid that refines regions where
$d_{\mathrm{spec}}(\eta)$ changes rapidly (Appendix section \ref{app:frontier-protocol}).
Figure~\ref{fig:frontier-multirank-all} shows the resulting frontiers. Across all eight datasets, the map $\eta\mapsto d_{\mathrm{spec}}$ is
monotone with narrow cross-seed variation, spanning roughly two orders of
magnitude in realized dimension before saturating near the rank cap. When
plotted against achieved $d_{\mathrm{spec}}$, AUC curves from different rank
caps align wherever their realized-capacity ranges overlap, indicating that
performance is organized by realized spectral capacity rather than by $r$
alone. At very large $d_{\mathrm{spec}}$, performance can decline because
positive $\eta$ pushes the spectrum toward equal mass over all available modes,
including modes that carry little signal. The peak locations reveal two regimes. Infrastructure and social networks
(\textsl{inf-power}, \textsl{inf-openflights}, \textsl{socfb-American75},
\textsl{socfb-Amherst41}) are saturated: their best operating points occur at
similar $d_{\mathrm{spec}}$ across rank caps and lie well below the smallest cap
tested. Citation and biological networks (\textsl{ca-grqc}, \textsl{ca-hepth},
\textsl{bio-grid-human}, \textsl{bio-grid-worm}) are rank-cap-binding: their
best achieved $d_{\mathrm{spec}}$ increases with $r$, and AUC continues to
improve as the cap is enlarged. Thus the multi-$r$ sweep distinguishes datasets
whose capacity demand is already saturated from those for which the
parameterization remains constraining through $r=256$.

\begin{table*}[!t]
\centering
\caption{\textbf{AUC-ROC scores} on eight benchmark graphs,
over $S=5$ seeds for various $d$. Each method block reports AUC and
achieved $d_{\mathrm{spec}}$. \textbf{Bold}: best AUC for each
(dataset, $d$); \underline{underline}: second best.}
\label{tab:auc_roc} 
\setlength{\tabcolsep}{2pt}
\renewcommand{\arraystretch}{1.0}
\scriptsize
\definecolor{auccol}{RGB}{215,230,246}   
\definecolor{dspeccol}{RGB}{249,234,212} 
\resizebox{1\textwidth}{!}{%
\begin{tabular}{l c >{\columncolor{auccol}}c >{\columncolor{dspeccol}}c >{\columncolor{auccol}}c >{\columncolor{dspeccol}}c >{\columncolor{auccol}}c >{\columncolor{dspeccol}}c >{\columncolor{auccol}}c >{\columncolor{dspeccol}}c >{\columncolor{auccol}}c >{\columncolor{dspeccol}}c >{\columncolor{auccol}}c >{\columncolor{dspeccol}}c >{\columncolor{auccol}}c >{\columncolor{dspeccol}}c >{\columncolor{auccol}}c >{\columncolor{dspeccol}}c >{\columncolor{auccol}}c >{\columncolor{dspeccol}}c}
\toprule
&   & \multicolumn{2}{c}{\textsc{Node2Vec}} & \multicolumn{2}{c}{\textsc{Role2Vec}} & \multicolumn{2}{c}{\textsc{NetMF}} & \multicolumn{2}{c}{\textsc{GraRep}} & \multicolumn{2}{c}{\textsc{MNMF}} & \multicolumn{2}{c}{\textsc{HM-LDM}} & \multicolumn{2}{c}{\textsc{Spectra} ($\eta\!=\!0$)} & \multicolumn{2}{c}{\textsc{Spectra} (med $d_{\mathrm{spec}}$)} & \multicolumn{2}{c}{\textsc{Spectra} (min $d_{\mathrm{spec}}$)} \\
\cmidrule(lr){3-4}\cmidrule(lr){5-6}\cmidrule(lr){7-8}\cmidrule(lr){9-10}\cmidrule(lr){11-12}\cmidrule(lr){13-14}\cmidrule(lr){15-16}\cmidrule(lr){17-18}\cmidrule(lr){19-20}
\textbf{Dataset} & $d$ & AUC & $d_{\mathrm{spec}}$ & AUC & $d_{\mathrm{spec}}$ & AUC & $d_{\mathrm{spec}}$ & AUC & $d_{\mathrm{spec}}$ & AUC & $d_{\mathrm{spec}}$ & AUC & $d_{\mathrm{spec}}$ & AUC & $d_{\mathrm{spec}}$ & AUC & $d_{\mathrm{spec}}$ & AUC & $d_{\mathrm{spec}}$ \\
\midrule
\multirow{4}{*}{ca-GrQc} & 16 & .931 & 15.91 & .936 & 15.36 & .887 & 15.58 & .906 & 8.45 & .905 & 14.99 & .937 & 14.72 & \textbf{.940} & 15.99 & \underline{.938} & 15.18 & .932 & 8.45 \\
& 32 & .936 & 31.62 & .936 & 30.47 & .899 & 30.71 & .909 & 13.7 & .918 & 30.94 & \underline{.942} & 28.15 & \textbf{.944} & 31.99 & \underline{.942} & 30.59 & .940 & 13.7 \\
& 64 & .937 & 59.64 & .931 & 55.88 & .897 & 60.38 & .897 & 22.45 & .924 & 62.68 & .937 & 48.40 & \underline{.949} & 63.74 & \textbf{.950} & 57.76 & .944 & 22.45 \\
& 128 & .934 & 86.41 & .924 & 125.58 & .888 & 116.33 & .893 & 35.08 & .918 & 125.58 & .949 & 90.85 & \underline{.952} & 124.51 & \textbf{.953} & 103.59 & .946 & 35.08 \\
\midrule
\multirow{4}{*}{ca-HepTh} & 16 & .883 & 15.95 & \textbf{.907} & 15.12 & .831 & 15.75 & .846 & 5.47 & .845 & 14.99 & .876 & 14.70 & .902 & 15.99 & \underline{.905} & 15.06 & .891 & 5.47 \\
& 32 & .888 & 31.80 & .902 & 30.13 & .836 & 31.41 & .842 & 8.24 & .864 & 30.97 & .877 & 29.32 & \underline{.910} & 31.99 & \textbf{.911} & 30.55 & .895 & 8.24 \\
& 64 & .892 & 63.36 & .896 & 58.86 & .827 & 62.37 & .832 & 13.11 & .875 & 62.91 & .874 & 56.30 & \underline{.915} & 63.94 & \textbf{.917} & 60.61 & .905 & 13.11 \\
& 128 & .892 & 109.57 & .882 & 108.52 & .810 & 122.42 & .825 & 21.59 & .876 & 126.70 & .891 & 95.85 & \textbf{.921} & 127.00 & \underline{.920} & 109.04 & .908 & 21.59 \\
\midrule
\multirow{4}{*}{socfb-Am.75} & 16 & .923 & 14.35 & .910 & 14.47 & .879 & 13.02 & .932 & 8.05 & .857 & 14.58 & .938 & 11.03 & \underline{.944} & 14.82 & \textbf{.945} & 13.68 & .938 & 8.05 \\
& 32 & .934 & 28.75 & .914 & 28.81 & .896 & 26.35 & \underline{.938} & 11.47 & .876 & 30.21 & \underline{.938} & 15.66 & \textbf{.951} & 27.61 & \textbf{.951} & 27.55 & .944 & 11.47 \\
& 64 & .930 & 57.64 & .910 & 60.97 & .905 & 52.22 & .942 & 16.95 & .895 & 60.97 & .938 & 17.24 & \textbf{.954} & 47.65 & \underline{.952} & 54.93 & .947 & 16.95 \\
& 128 & .914 & 114.33 & .896 & 120.43 & .906 & 100.69 & .943 & 26.44 & .912 & 120.43 & .939 & 20.99 & \textbf{.955} & 64.58 & \underline{.949} & 107.51 & .949 & 20.99 \\
\midrule
\multirow{4}{*}{socfb-Amh.41} & 16 & .926 & 13.18 & .903 & 13.51 & .892 & 11.11 & .930 & 7.81 & .884 & 14.07 & \underline{.939} & 10.12 & \textbf{.943} & 14.23 & .901 & 12.14 & .938 & 7.81 \\
& 32 & .926 & 26.36 & .899 & 26.87 & .901 & 21.80 & .933 & 11.47 & .899 & 29.19 & \underline{.941} & 12.35 & \textbf{.947} & 25.00 & .915 & 24.08 & .942 & 11.47 \\
& 64 & .909 & 52.58 & .882 & 53.59 & .899 & 42.35 & .934 & 17.95 & .909 & 58.34 & .941 & 13.11 & \underline{.949} & 38.01 & \textbf{92.4} & 47.46 & .944 & 13.11 \\
& 128 & .890 & 102.40 & .862 & 107.40 & .885 & 79.84 & .929 & 29.64 & .913 & 115.79 & \underline{.942} & 13.87 & \textbf{.949} & 40.82 & .932 & 91.12 & .944 & 13.87 \\
\midrule
\multirow{4}{*}{bio-grid-h.} & 16 & .902 & 15.75 & .878 & 15.45 & .867 & 14.62 & .910 & 8.66 & .840 & 14.97 & \underline{.919} & 14.29 & .902 & 15.99 & \textbf{.943} & 14.80 & .899 & 8.66 \\
& 32 & .900 & 31.51 & .876 & 30.69 & .851 & 29.09 & \underline{.917} & 14.34 & .858 & 30.89 & \underline{.917} & 27.61 & .915 & 31.97 & \textbf{.947} & 29.89 & .903 & 14.34 \\
& 64 & .892 & 62.43 & .875 & 60.36 & .847 & 56.72 & .918 & 23.99 & .867 & 62.71 & .919 & 45.53 & \underline{.925} & 63.68 & \textbf{.946} & 58.54 & .910 & 23.99 \\
& 128 & .883 & 119.15 & .877 & 118.62 & .854 & 109.69 & .913 & 38.77 & .872 & 125.91 & .923 & 76.82 & \underline{.932} & 122.87 & \textbf{.938} & 114.16 & .919 & 38.77 \\
\midrule
\multirow{4}{*}{bio-grid-w.} & 16 & .786 & 15.52 & .744 & 15.21 & .769 & 15.57 & \underline{.900} & 11.95 & .769 & 14.79 & \textbf{.913} & 13.32 & .852 & 15.99 & .852 & 15.00 & .835 & 11.95 \\
& 32 & .775 & 30.58 & .748 & 29.00 & .821 & 30.10 & \underline{.911} & 17.67 & .780 & 30.64 & \textbf{.914} & 23.06 & .886 & 31.93 & .882 & 29.55 & .860 & 17.67 \\
& 64 & .754 & 57.84 & .772 & 54.80 & .852 & 56.28 & \underline{.916} & 24.17 & .792 & 62.35 & \textbf{.920} & 32.03 & .915 & 63.33 & .913 & 55.54 & .877 & 24.17 \\
& 128 & .747 & 95.37 & .812 & 103.96 & .878 & 98.86 & .917 & 30.30 & .785 & 124.80 & .920 & 36.19 & \textbf{.934} & 121.33 & \underline{.931} & 97.12 & .889 & 30.30 \\
\midrule
\multirow{4}{*}{inf-power} & 16 & .892 & 15.95 & \underline{.942} & 15.25 & .899 & 15.73 & \textbf{.956} & 13.01 & .816 & 14.99 & .885 & 14.67 & .936 & 15.99 & .939 & 15.12 & .949 & 13.01 \\
& 32 & .919 & 31.9 & \underline{.942} & 30.92 & .920 & 31.44 & \textbf{.956} & 20.64 & .841 & 30.97 & .881 & 28.75 & .912 & 31.99 & .916 & 30.94 & .925 & 20.64 \\
& 64 & .914 & 60.28 & \underline{.940} & 55.54 & .892 & 62.45 & \textbf{.953} & 31.07 & .841 & 62.89 & .847 & 50.53 & .895 & 63.98 & .896 & 57.91 & .916 & 31.07 \\
& 128 & .901 & 80.41 & \underline{.935} & 63.67 & .845 & 122.82 & \textbf{.945} & 45.15 & .813 & 126.61 & .889 & 95.65 & .880 & 127.89 & .884 & 88.03 & .910 & 45.15 \\
\midrule
\multirow{4}{*}{inf-openfl.} & 16 & .974 & 15.41 & .962 & 13.56 & .948 & 15.45 & .984 & 11.59 & .944 & 14.63 & \textbf{.992} & 12.88 & .987 & 15.83 & \underline{.988} & 14.10 & .989 & 11.59 \\
& 32 & .974 & 30.16 & .963 & 25.26 & .952 & 29.95 & .986 & 16.45 & .959 & 19.93 & \textbf{.990} & 21.43 & .987 & 30.87 & \underline{.988} & 23.34 & .988 & 16.45 \\
& 64 & .974 & 55.31 & .966 & 48.12 & .949 & 57.18 & .985 & 21.90 & .968 & 60.95 & \textbf{.991} & 26.67 & \underline{.987} & 57.73 & \underline{.987} & 51.72 & .988 & 21.90 \\
& 128 & .973 & 87.94 & .966 & 94.68 & .935 & 103.84 & .982 & 26.60 & .972 & 122.08 & \textbf{.991} & 38.86 & \underline{.987} & 96.43 & \underline{.987} & 91.31 & .988 & 26.60 \\
\bottomrule
\end{tabular}}
\end{table*}

\textbf{Link prediction.}
We follow the standard unsupervised protocol~\citep{perozzi2014deepwalk,
nakis2023hierarchicalblockdistancemodel,libennowell2007link}: remove
$50\%$ of edges while keeping the training graph connected, and sample an equal
number of non-edges as negatives. Table~\ref{tab:auc_roc} reports mean
AUC-ROC over $S=5$ seeds for $d\in\{16,32,64,128\}$; AUC-PR gives
qualitatively identical conclusions. Baselines are trained at native dimension
$d$, while \textsc{Spectra} uses rank cap $r=d$, so comparisons are matched in
parameter count. We also report each method's achieved
$d_{\mathrm{spec}}=\exp(H(\sigma^2/\sum_i\sigma_i^2))$ from embedding singular
values. \textsc{Spectra} is evaluated at $\eta=0$ and at two bisection-selected
settings targeting the median and minimum baseline effective ranks at the same
$d$; these targets are fixed from baseline embeddings, with no validation or
test AUC used to choose $\eta$. At matched parameter count, \textsc{Spectra} $(\eta=0)$ is top or second on
$7/8$ datasets at every $d$; the exceptions are \textsl{inf-power}, where
\textsc{GraRep} dominates, and \textsl{inf-openflights}, where \textsc{HM-LDM}
leads narrowly. Capacity-targeted \textsc{Spectra} remains competitive at much
smaller realized dimensions: the median-target setting leads on
\textsl{ca-HepTh}, \textsl{socfb-American75}, and \textsl{bio-grid-human} at
multiple $d$, while the minimum-target setting often matches full-rank
baselines at $5{-}10\times$ smaller $d_{\mathrm{spec}}$. Baselines also realize
effective ranks well below nominal $d$ on socfb and bio-grid graphs
(e.g., \textsc{HM-LDM} at $d=128$ has
$d_{\mathrm{spec}}\approx21$ on \textsl{socfb-American75} and
$\approx36$ on \textsl{bio-grid-worm}), reinforcing realized spectral capacity
as the operative comparison coordinate.

\begin{figure}[!t]
  \centering
  \begin{subfigure}[t]{0.19\linewidth}
    \centering
    \includegraphics[width=0.9\linewidth]{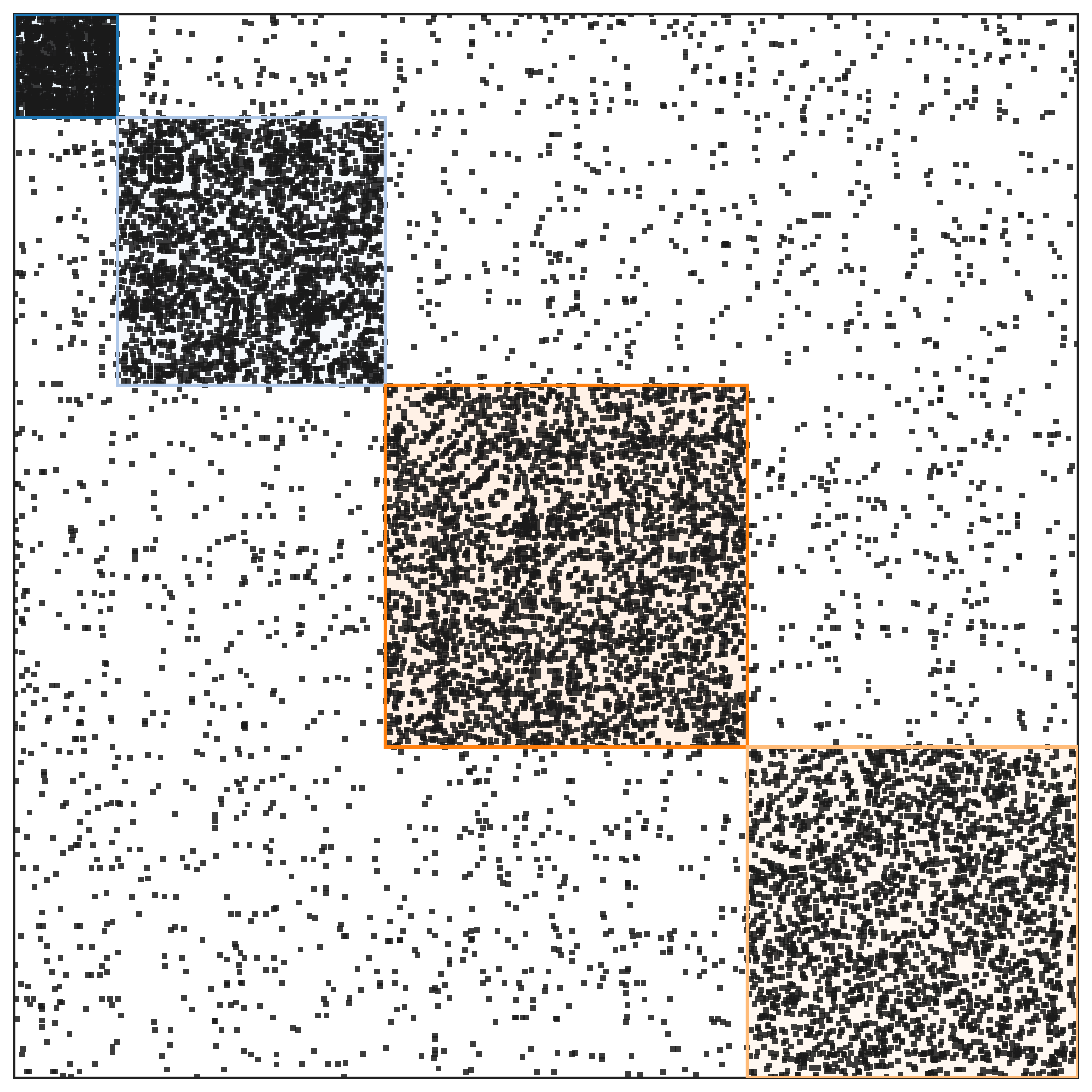}
    \caption{$d_{\mathrm{pre}}{=}4$}
    \label{fig:adj-grqc-k4}
  \end{subfigure}\hfill
  \begin{subfigure}[t]{0.19\linewidth}
    \centering
    \includegraphics[width=0.9\linewidth]{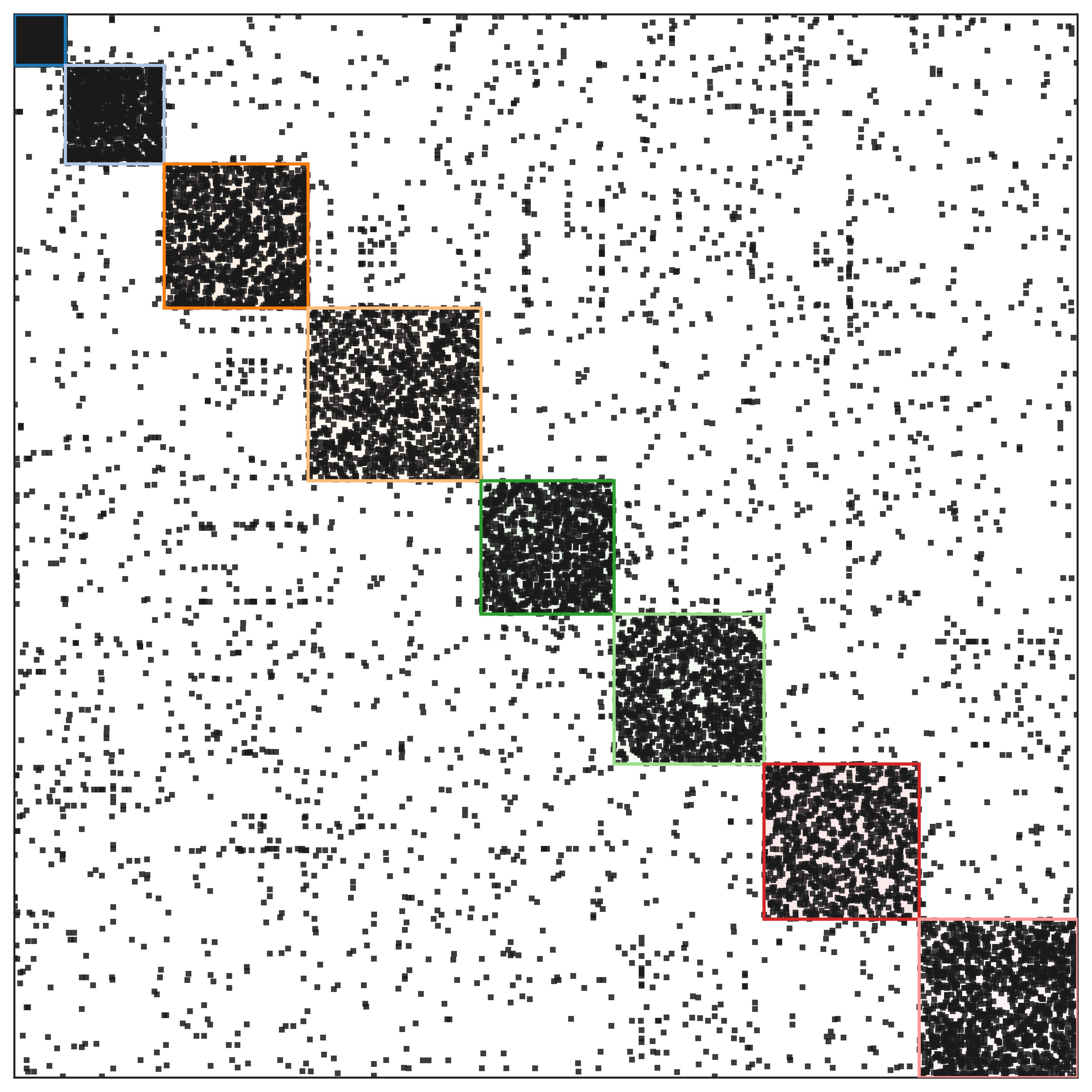}
    \caption{$d_{\mathrm{pre}}{=}8$}
    \label{fig:adj-grqc-k8}
  \end{subfigure}\hfill
  \begin{subfigure}[t]{0.19\linewidth}
    \centering
    \includegraphics[width=0.9\linewidth]{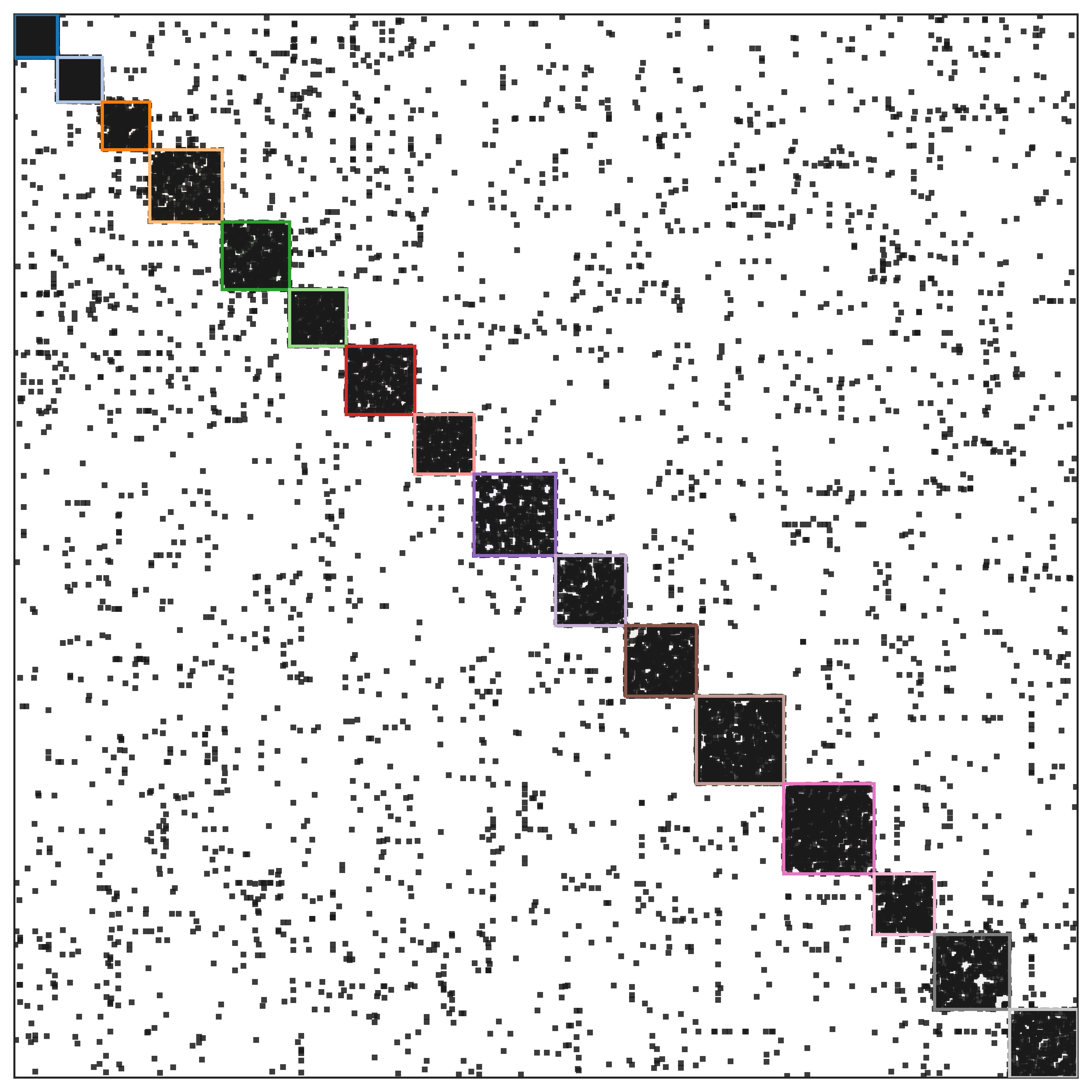}
    \caption{$d_{\mathrm{pre}}{=}16$}
    \label{fig:adj-grqc-k16}
  \end{subfigure}\hfill
  \begin{subfigure}[t]{0.19\linewidth}
    \centering
    \includegraphics[width=0.9\linewidth]{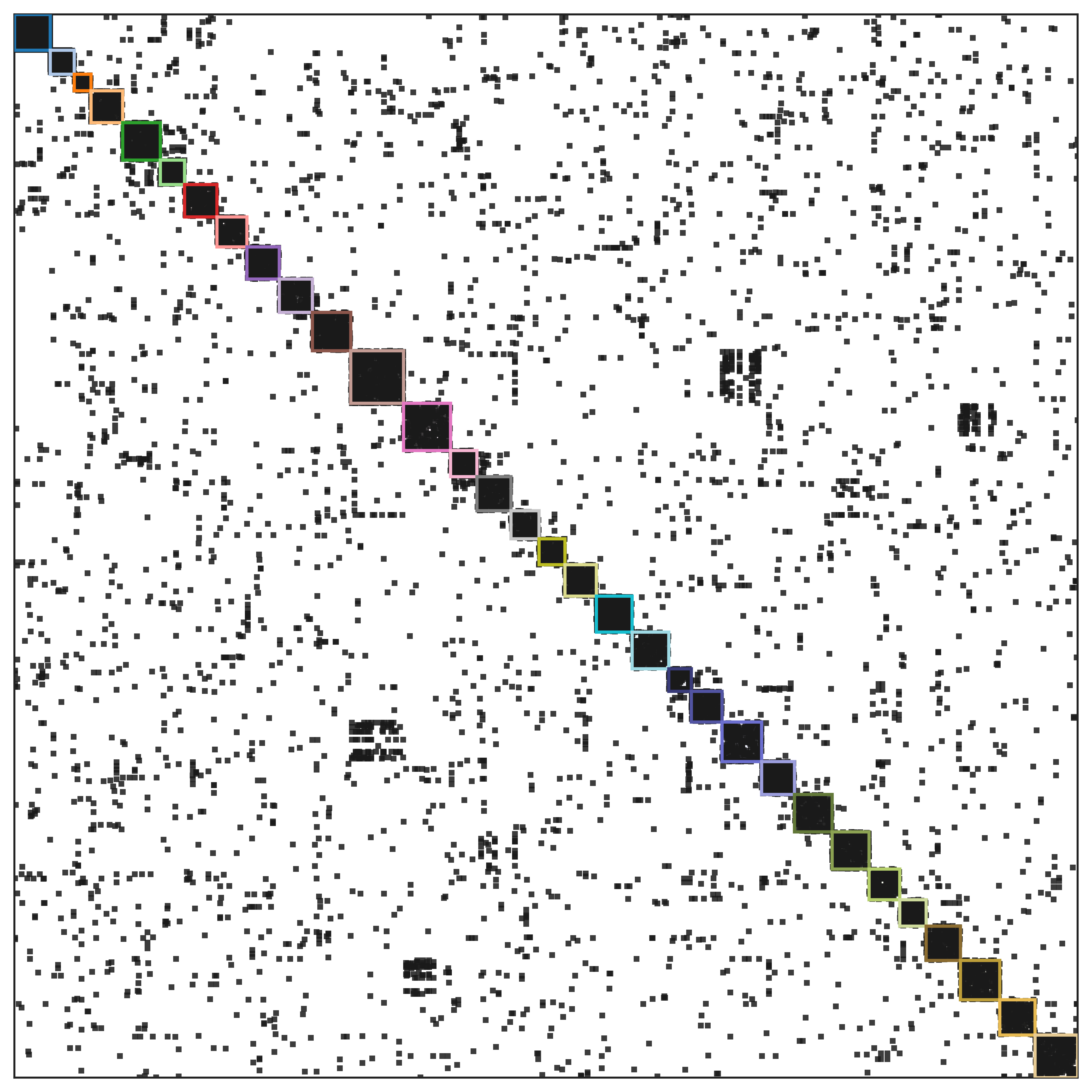}
    \caption{$d_{\mathrm{pre}}{=}32$}
    \label{fig:adj-grqc-k32}
  \end{subfigure}
   \begin{subfigure}[t]{0.19\linewidth}
    \centering
    \includegraphics[width=0.9\linewidth]{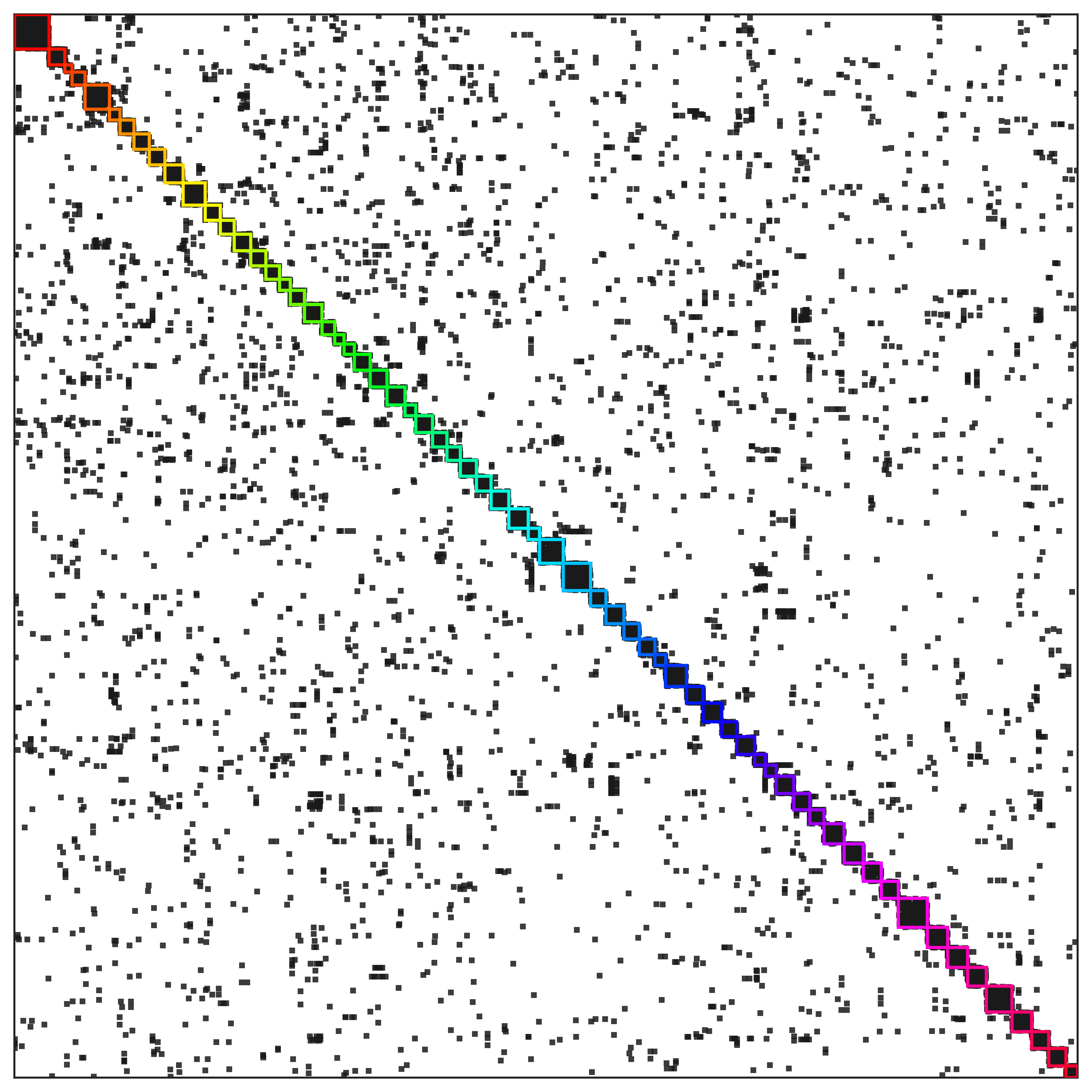}
    \caption{$d_{\mathrm{pre}}{=}64$}
    \label{fig:adj-grqc-k64}
  \end{subfigure}
  \caption{Adjacency matrices reordered by spectral assignment
$c_i(d_{\mathrm{pre}})=\arg\max_{j\le d_{\mathrm{pre}}} U_{ij}^2$ for
$d_{\mathrm{pre}}\in\{4,8,16,32,64\}$. Panels come from a single fitted
\textsc{Spectra} kernel via the rank-$d_{\mathrm{pre}}$ prefix.}
  \label{fig:adj-grqc-prefix-grid}
\end{figure}

\textbf{Visualizing the prefix family.}
Figure~\ref{fig:adj-grqc-prefix-grid} visualizes spectral prefixes
$d_{\mathrm{pre}}\in\{4,8,16,32,64\}$ from a \emph{single}
\textsc{Spectra} fit on \textsl{grqc} with rank cap $r=128$ and target
$d_{\mathrm{spec}}^\star=64$. Each panel uses the rank-$d_{\mathrm{pre}}$
prefix
$
    K_{d_{\mathrm{pre}}}
    =
    U_{1:d_{\mathrm{pre}}}
    \Lambda_{1:d_{\mathrm{pre}}}
    U_{1:d_{\mathrm{pre}}}^{\top},
$
which is the optimal rank-$d_{\mathrm{pre}}$ PSD approximation of the fitted
kernel in Frobenius norm by
Theorem~\ref{thm:optimal-spectral-prefix}. Nodes are reordered by the dominant
retained spectral coordinate,
$
    c_i(d_{\mathrm{pre}})
    =
    \arg\max_{1\le j\le d_{\mathrm{pre}}} U_{ij}^{2},
$
and blocks are shaded by this assignment. This is a standard spectral way to
read block structure from leading eigenspaces
\citep{Ng2001OnSC,vonluxburg2007tutorialspectralclustering}, but here the
eigenspaces come from the learned trace-normalized kernel rather than directly
from the adjacency. As $d_{\mathrm{pre}}$ increases, the reordered adjacency resolves progressively
finer structure: broad macro-groups at $d_{\mathrm{pre}}=4$ split into smaller,
denser blocks by $d_{\mathrm{pre}}=64$, while off-diagonal density decreases.
No panel is a refit. The role of $\eta$ is to set the spectral operating point
\emph{before} prefix extraction: targeting $d_{\mathrm{spec}}^\star$ determines
how spectral mass is distributed across modes, and hence how far the prefix
family adds meaningful structure. This gives a graph analogue of
Matryoshka-style multi-resolution representations
\citep{kusupati2024matryoshkarepresentationlearning}, but with alignment and
optimality following from Eckart--Young--Mirsky for a single fitted kernel
rather than from a multi-scale training loss.

\textbf{Overparameterized target-$d_{\mathrm{spec}}^\star$ acquisition.}
We test whether overparameterization, used through spectral
control, finds a better-calibrated solution than the rank-cap-only
practice. We sweep target effective ranks
$d_{\mathrm{spec}}^\star \in \{16, 32, 64\}$ and rank caps
$r \in \{d_{\mathrm{spec}}^\star, 64, 128, 256\}$ on \textsl{grqc}
and \textsl{bio-grid-human}, with $n=10$ seeds per cell. Each cell
trains two models on identical data: a Phase-0 anchor at $\eta=0$
($d_{\mathrm{spec}}$ saturates near $r$) and a Phase-2 retrain at
$\eta^\star$ obtained by bidirectional bisection of fully-trained
probes ($\tau=0.05$). Metrics target the optimization objective:
test log-likelihood $\mathrm{TLL}$ and the train--test
generalization gap $\mathrm{GenGap} = \mathrm{trNLL} + \mathrm{TLL}$
(lower is better). To assess whether observed
differences between the two conditions reflect a systematic effect, we report paired Wilcoxon
signed-rank tests on the per-seed differences, paired because
the two conditions in each cell share data, negatives, and seed,
and signed-rank rather than $t$-test because paired-difference
distributions across seeds are heavy-tailed and not well-modeled
as Gaussian. \textit{Comparison A} (within-cell, isolating $\eta$ at fixed $r$):
without spectral control, overparameterization is
calibration-disastrous, across 240 paired comparisons, the
retrain improves GenGap ($p < 10^{-30}$) and TLL ($p < 10^{-25}$),
while the anchor fits training strictly harder (lower
$\mathrm{trNLL}$, $p < 10^{-25}$). \textit{Comparison B} (matched
$d_{\mathrm{spec}}^\star$, against the rank-cap-only baseline at
$r = d_{\mathrm{spec}}^\star$): $\eta$-targeted overparameterization
finds a calibration-superior minimum at matched effective capacity, across 180 paired comparisons, both GenGap ($p \approx 10^{-11}$)
and TLL ($p \approx 10^{-12}$) favor the retrain, with larger
effect on \textsl{bio-grid-human} (each $p < 10^{-10}$) than on
the more-saturated \textsl{grqc} (each $p \approx 0.015$). Since
both conditions reach the same realized $d_{\mathrm{spec}}$, the
gain comes not from additional representational capacity but from
the optimizer landscape that the larger parameterization opens up, a setting consistent with the broader literature on
overparameterized models that generalize well when their effective
complexity is controlled~\citep{belkin2019reconciling,bartlett2020benign}.

\section{Conclusion \& Limitations}

We introduced \textsc{Spectra}, which treats capacity in latent graph models as a property of the
fitted representation rather than as a nominal rank chosen before training. Its key
object is the trace-normalized learned PSD kernel: the Shannon effective rank
$d_{\mathrm{spec}}$ gives a smooth realized-dimension coordinate, shaped by an entropy
weight $\eta$ and targeted by bisection. The framework provides identifiability of
active spectral modes, local $C^1$ regularity away from structural events, and
Eckart--Young--Mirsky-optimal spectral prefixes. Empirically, \textsc{Spectra} is
competitive with strong link-prediction baselines at matched parameter count, remains
competitive at substantially smaller realized dimensions, reveals saturated and
rank-cap-binding regimes through capacity frontiers, and yields aligned
Matryoshka-style prefix views without a multi-scale objective. Beyond link prediction, \textsc{Spectra} suggests a broader use for comparing ensembles
of networks through realized representational complexity. For example,
$d_{\mathrm{spec}}$ could be used to ask whether village-level social networks from
field experiments in Honduras \citep{airoldi2024induction}, India
\citep{alexander2022algorithms}, or Malawi \citep{beaman2021network}; school friendship
networks from Add Health \citep{harris2019cohort,jeon2015adolescents} or school
intervention studies \citep{paluck2016changing}; protein interaction networks such as
those studied in lethality-centrality work \citep{jeong2001lethality}; or infrastructure
networks such as city electrical grids \citep{PAGANI20132688} exhibit comparable structural complexity. In
such settings, $d_{\mathrm{spec}}$ could serve as an outcome variable for studying how
network complexity varies with intervention context, school environment, biological
organization, or urban scale.

The main limitation is geometric: we work with Euclidean latent spaces, where capacity is
read from an $O(r)$-invariant Gram spectrum. Simplex-valued and non-Euclidean latent
models require the corresponding invariant kernel, e.g., a log-ratio Gram under
Aitchison geometry for compositional memberships
\citep{nakis2026aitchisonembeddingslearningcompositional} or a geometry-native kernel
for hyperbolic embeddings. Cross-network comparisons would also require common
preprocessing and fitting protocols, with controls for graph size and density. We leave
these extensions to future work.

\section*{Acknowledgments}

The authors gratefully acknowledge support from the NOMIS Foundation, the Stavros Niarchos Foundation, and the National Institutes of Health (R01AG081814).

\bibliographystyle{plainnat}
\bibliography{bibliography}
\clearpage
\appendix

\begin{table}[h]
\centering
\caption{\textbf{Notation used throughout the paper.} Symbols appear in
the order they are introduced. The two effective-dimension symbols
$d_{\mathrm{spec}}$ and $d_{\mathrm{pre}}$ are conceptually distinct:
$d_{\mathrm{spec}}$ is the continuous Shannon effective rank of the
fitted kernel (controlled at training time by $\eta$), while
$d_{\mathrm{pre}}$ is the integer prefix size chosen post-hoc when
extracting a rank-$d_{\mathrm{pre}}$ prefix of the fitted kernel.}
\label{tab:notation}
\renewcommand{\arraystretch}{1.25}
\resizebox{0.65\textwidth}{!}{
\begin{tabular}{l p{0.72\linewidth}}
\toprule
\textbf{Symbol} & \textbf{Meaning} \\
\midrule
\multicolumn{2}{l}{\emph{Graph and observations}} \\
\midrule
$\mathcal{G}=(V,E)$ & Simple undirected graph with node set $V$ and edge set $E$. \\
$N$               & Number of nodes, $N=|V|$. \\
$\bm{Y}\in\{0,1\}^{N\times N}$ & Symmetric adjacency matrix, $\bm{Y}_{ii}=0$. \\
\midrule
\multicolumn{2}{l}{\emph{Latent representation}} \\
\midrule
$L\in\mathbb{R}^{N\times r}$ & Latent factor; rank cap $r$. \\
$r$               & Parameterization ceiling on $\operatorname{rank}(L)$. \\
$K(L)\in\mathbb{R}^{N\times N}$ & Trace-normalized PSD kernel,
$K(L)=N\,LL^\top/\operatorname{tr}(LL^\top)$. \\
$\bm{a}\in\mathbb{R}^N$ & Node-specific log-odds offsets. \\
$\beta>0$         & Global slope on the latent affinity. \\
$\phi(t)$       & Sigmoid, $\phi(t)=(1+e^{-t})^{-1}$. \\
\midrule
\multicolumn{2}{l}{\emph{Spectrum and effective dimension}} \\
\midrule
$\lambda_j(K)$    & The $j$-th eigenvalue of $K$, sorted descending. \\
$u_j$             & Orthonormal eigenvector associated with $\lambda_j(K)$. \\
$U,\Lambda$       & Orthonormal eigenvector matrix and diagonal eigenvalue
matrix in $K=U\Lambda U^\top$. \\
$p_i(K)$          & Spectral occupancy, $p_i(K)=\lambda_i(K)/N$. \\
$p(K)$            & Spectral occupancy distribution
$(p_1(K),\ldots,p_N(K))$ on the $(N{-}1)$-simplex. \\
$H(p)$            & Shannon entropy of $p$,
$H(p)=-\sum_i p_i\log p_i$. \\
$d_{\mathrm{spec}}(K)$ & Effective spectral dimension,
$d_{\mathrm{spec}}(K)=\exp(H(p(K)))$. Continuous; ranges over
$[1,\operatorname{rank}(K)]$. \\
$d_{\mathrm{spec}}(\eta)$ & Realized $d_{\mathrm{spec}}$ along the
$\eta$-controlled solution branch,
$d_{\mathrm{spec}}(\eta)=d_{\mathrm{spec}}(K(L^\star(\eta)))$. \\
$d_{\mathrm{spec}}^\star$ & Target value of $d_{\mathrm{spec}}$ specified
by the user; $\eta$ is calibrated so that
$d_{\mathrm{spec}}(\eta)\approx d_{\mathrm{spec}}^\star$. \\
$\mathrm{PR}(K)$  & Participation ratio, $\mathrm{PR}(K)=1/\sum_i p_i(K)^2$. \\
$k_\tau(K)$       & Thresholded active rank,
$k_\tau(K)=|\{i:\lambda_i(K)\geq \tau\,\lambda_1(K)\}|$. \\
\midrule
\multicolumn{2}{l}{\emph{Spectral prefixes}} \\
\midrule
$d_{\mathrm{pre}}$ & Integer prefix size used for prefix-based
visualization and analysis; $1\leq d_{\mathrm{pre}}\leq r$. \\
$K_{d_{\mathrm{pre}}}$ & Rank-$d_{\mathrm{pre}}$ spectral prefix of $K$,
$K_{d_{\mathrm{pre}}}=U_{1:d_{\mathrm{pre}}}\Lambda_{1:d_{\mathrm{pre}}}
U_{1:d_{\mathrm{pre}}}^\top$. EYM-optimal among rank-$d_{\mathrm{pre}}$
PSD approximations of $K$ (Thm.~\ref{thm:optimal-spectral-prefix}). \\
$\pi_{ij}$ & Soft membership of node $i$ in prefix mode $j$,
$\pi_{ij}\propto q_j[i]^2$, normalized to a row-simplex. \\
\midrule
\multicolumn{2}{l}{\emph{Training objective and capacity control}} \\
\midrule
$\eta\in\mathbb{R}$ & Spectral-entropy weight (Lagrange-style coefficient)
controlling realized $d_{\mathrm{spec}}$. \\
$\ell(L,\bm{a},\beta)$ & Sampled link-prediction negative log-likelihood. \\
$R(L,\bm{a},\beta)$ & Quadratic regularization on parameters. \\
$F_\eta(L,\bm{a},\beta)$ & Entropy-regularized objective,
$F_\eta=\ell-\eta\log d_{\mathrm{spec}}(K(L))+R$. \\
$L^\star(\eta)$ & Local minimizer of $F_\eta$ along a fitted branch. \\
$K^\star(\eta)$ & Trace-normalized kernel at $L^\star(\eta)$. \\
$\widetilde{\Theta}$ & Gauge-fixed slice of the parameter space used to
state second-order conditions. \\
\midrule
\multicolumn{2}{l}{\emph{Calibration of $\eta$}} \\
\midrule
$\tau$ & Relative tolerance of the bisection stopping criterion,
$|d_{\mathrm{spec}}(\eta)-d_{\mathrm{spec}}^\star|/d_{\mathrm{spec}}^\star
\leq\tau$. \\
$\eta^\star$ & Calibrated entropy weight selected by bisection. \\
$[\eta_{\min},\eta_{\max}]$ & Bracket interval used in bisection;
$W=\eta_{\max}-\eta_{\min}$. \\
$\delta$ & Target $\eta$-precision in the complexity statement. \\
\midrule
\multicolumn{2}{l}{\emph{Evaluation}} \\
\midrule
$d$ & Nominal embedding dimension at which baselines are trained
(matched to \textsc{Spectra}'s rank cap $r=d$ in Table~\ref{tab:auc_roc}). \\
TLL & Test log-likelihood. \\
trNLL & Training negative log-likelihood. \\
GenGap & Generalization gap, $\mathrm{GenGap}=\mathrm{trNLL}+\mathrm{TLL}$. \\
$S$ & Number of random seeds per experimental cell. \\
\bottomrule
\end{tabular}}
\end{table}

\section{Extra Preliminaries}
\label{app:dg-definitions}

We give some basic definitions used in the proofs. We only consider smooth
submanifolds of the Euclidean space.

The definitions in this appendix are standard. See, for example,
~\cite{lee2013smooth} and ~\cite{tu2011manifolds}.

\begin{definition}[Smooth embedded submanifold]
\label{def:smoothManifold}
Let \(M\subset \mathbb R^D\). We say that \(M\) is a smooth embedded
submanifold of dimension \(m\) if, around every point \(p\in M\), there exists
an open set \(W\subset \mathbb R^m\) and a smooth parametrization
\[
    \psi:W\to \mathbb R^D
\]
such that \(\psi(W)\) is a neighborhood of \(p\) in \(M\), the map
\(\psi:W\to \psi(W)\) is one-to-one, with smooth inverse, and the Jacobian
matrix
\(
    J_\psi(x)
\)
has rank \(m\) for every \(x\in W\).
\end{definition}

\begin{definition}[Local chart]
Let \(M\subset \mathbb R^D\) be a smooth embedded submanifold. A local chart
around \(p\in M\) is a smooth parametrization
\[
    \psi:W\subset\mathbb R^m\to M
\]
whose image is a neighborhood of \(p\) in \(M\). If \( p=\psi(x), \)
then \(x\in\mathbb R^m\) is the local coordinate representation of \(p\). Thus, a function \(f:M\to\mathbb R\) can locally be written as the ordinary
Euclidean function
\(
    \widehat f(x):=f(\psi(x)).
\)
\end{definition}

\begin{definition}[Tangent vector]
Let \(M\subset\mathbb R^D\) be a smooth embedded submanifold and let \(p\in M\).
A vector \(v\in\mathbb R^D\) is tangent to \(M\) at \(p\) if there exists a
smooth curve
\[
    \gamma:(-\varepsilon,\varepsilon)\to M
\]
such that
\(
    \gamma(0)=p,
    \ 
    \gamma'(0)=v.
\)
\end{definition}

\begin{definition}[Tangent space]
The tangent space of \(M\) at \(p\), denoted \(T_pM\), is the set of all tangent
vectors to \(M\) at \(p\):
\[
    T_pM
    :=
    \left\{
        \gamma'(0)\mid
        \gamma:(-\varepsilon,\varepsilon)\to M
        \text{ is smooth and }\gamma(0)=p
    \right\}.
\]
It is a vector subspace of the ambient space \(\mathbb R^D\).
\end{definition}

\begin{definition}[Tangent space in local coordinates]
Let
\[
    \psi:W\subset\mathbb R^m\to M
\]
be a local chart and let \(p=\psi(x)\). Then
\(
    T_pM=\operatorname{Im} J_\psi(x).
\)
Equivalently, every tangent vector \(v\in T_pM\) can be written as
\(
    v=J_\psi(x)u
\)
for some \(u\in\mathbb R^m\).
\end{definition}

\begin{definition}[Tangent space of a constraint set]
\label{def:tangentSpace:constraint}
Suppose
\[
    M=\{x\in\mathbb R^D:c(x)=0\},
\]
where \(c:\mathbb R^D\to\mathbb R^q\) is smooth. Suppose also that the
Jacobian matrix
\(
    J_c(p)
\)
has full rank for every \(p\in M\). Then \(M\) is a smooth embedded
submanifold, and its tangent space is
\[
    T_pM=\ker J_c(p).
\]
Equivalently,
\(
    v\in T_pM
    \ \Leftrightarrow\ 
    J_c(p)v=0.
\)
\end{definition}

\section{Proof of Theorem \ref{thm:optimal-spectral-prefix}}
\label{ap:thm:optimal-spectral-prefix}
Since \(K\) is symmetric and PSD, the spectral theorem gives an orthonormal
eigenbasis for its image $\mathrm{Im}(K)
    = \{Kx : x \in \mathbb{R}^N\}$. Thus, we may write
\[
    K=\sum_{j=1}^r \lambda_j u_j u_j^\top,
\]
where \(r=\operatorname{rank}(K)\), the vectors \(u_1,\ldots,u_r\) are orthonormal, and the
nonzero eigenvalues satisfy \(\lambda_j>0\). Equivalently,
\[
    K=U\operatorname{diag}(\lambda_1,\ldots,\lambda_r)U^\top,
\]
where \(U=(u_1,\ldots,u_r)\) and \(U^\top U=I_r\). By the Eckart--Young--Mirsky Theorem \cite{Eckart1936TheAO}, the best rank-\(k\) approximation to \(K\) in Frobenius
norm, among all matrices of rank at most \(k\), is obtained by keeping the \(k\) largest singular
values and their corresponding singular vectors. Since \(K\succeq 0\), its singular values are
exactly its eigenvalues. Therefore the unconstrained best rank-\(k\) approximation is
\[
    K_k=\sum_{j=1}^k \lambda_j u_j u_j^\top .
\]
Next, we want to check that the positive semidefinite constraint does not change the optimizer.
The optimization over all rank-\(k\) matrices is a relaxation of the optimization over rank-\(k\)
positive semidefinite matrices:
\[
    \min_{\operatorname{rank}(M)\leq k}\|K-M\|_F^2
    \leq
    \min_{\substack{M\succeq 0\\ \operatorname{rank}(M)\leq k}}\|K-M\|_F^2.
\]
But the unconstrained minimizer \(K_k\) is itself trivially PSD, because it is a nonnegative
linear combination of rank-one positive semidefinite matrices. Therefore,
it is also optimal among positive semidefinite rank-\(k\) matrices.


It remains to prove the uniqueness statement.  Assume that \(\lambda_k>\lambda_{k+1}\), and let \(M\succeq 0\) with
\(\operatorname{rank}(M)\leq k\) be any minimizer. Write the spectral
decomposition of \(M\) as
\[
    M=\sum_{i=1}^k \mu_i v_i v_i^\top,
\]
where \(\mu_1\geq \cdots \geq \mu_k\geq 0\)
and \(v_1,\ldots,v_k\) are orthonormal. Using the Frobenius inner product,
\[
    \|K-M\|_F^2
    =
    \|K\|_F^2+\|M\|_F^2-2\langle K,M\rangle.
\]
Since \(K\succeq 0\) and \(M\succeq 0\), von Neumann's trace inequality gives
\[
    \langle K,M\rangle
    =
    \operatorname{tr}(KM)
    \leq
    \sum_{i=1}^k \lambda_i \mu_i .
\]
Therefore,
\[
\begin{aligned}
    \|K-M\|_F^2
    &\geq
    \sum_{j=1}^r \lambda_j^2
    +
    \sum_{i=1}^k \mu_i^2
    -
    2\sum_{i=1}^k \lambda_i\mu_i  \\
    &=
    \sum_{i=1}^k(\mu_i-\lambda_i)^2
    +
    \sum_{j=k+1}^r \lambda_j^2 .
\end{aligned}
\]
On the other hand, the truncated matrix \(K_k\) satisfies
\(
    \|K-K_k\|_F^2
    =
    \sum_{j=k+1}^r \lambda_j^2.
\)
Since \(M\) is also assumed to be optimal, we must have
\(
    \|K-M\|_F^2
    =
    \sum_{j=k+1}^r \lambda_j^2.
\)
Comparing this with the lower bound above, we get
\[
    \sum_{i=1}^k(\mu_i-\lambda_i)^2=0.
\]
Hence
\(
    \mu_i=\lambda_i,
    \  i=1,\ldots,k.
\)

Moreover, equality must also hold in von Neumann's trace inequality. The
equality implies that the image
of \(M\) must be a top-\(k\) eigenspace of \(K\). Because \(\lambda_k>\lambda_{k+1}\), this top-\(k\) eigenspace is uniquely
determined, namely
\[
    \operatorname{span}\{u_1,\ldots,u_k\}.
\]
Therefore \(M\) has eigenvalues \(\lambda_1,\ldots,\lambda_k\) on this uniquely
determined subspace and is zero on its orthogonal complement. Consequently,
\[
    M
    =
    \sum_{j=1}^k \lambda_j u_j u_j^\top
    =
    K_k.
\]
Thus \(K_k\) is the unique minimizer.

\section{Proof of Theorem \ref{thm:identifiability-spectral-modes}}
\label{ap:thm:identifiability-spectral-modes}
Since \(K\) is symmetric and PSD, the spectral theorem gives an orthonormal
basis of eigenvectors for its image. Thus \(K\) can be written as
\[
    K
    =
    \sum_{j=1}^r \lambda_j u_j u_j^\top,
\]
where \(u_1,\ldots,u_r\) are orthonormal eigenvectors and
\(\lambda_1,\ldots,\lambda_r\) are the positive eigenvalues. 

For any eigenvalue \(\lambda_j\), the associated eigenspace is
\[
    E_{\lambda_j}
    =
    \{v:Kv=\lambda_j v\}.
\]
Since \(E_{\lambda_j}\) is one-dimensional, any unit eigenvector \(u'_j\) associated with
\(\lambda_j\) satisfies
\(
    u'_j = c_j u_j
\)
for some scalar \(c_j\). Because the vectors in the eigendecompositioon are orthonormal it holds that \(\|u'_j\|_2=\|u_j\|_2=1\), we have \(|c_j|=1\), and hence
\(c_j\in\{\pm1\}\). Therefore \begin{equation}\label{eq:pm_u}
u'_j=\pm u_j.
\end{equation}

Now suppose
\[
    K
    =
    U\operatorname{diag}(\lambda_1,\ldots,\lambda_r)U^\top
    =
    U'\operatorname{diag}(\lambda'_1,\ldots,\lambda'_r){U'}^\top
\]
are two spectral decompositions, with \(U^\top U={U'}^\top U'=I_r\).

For each \(\ell\in\{1,\ldots,r\}\), the vector \(u'_\ell\), the \(\ell\)-th column of \(U'\), is a unit
eigenvector of \(K\) with eigenvalue \(\lambda'_\ell\). Since the positive eigenvalues of \(K\) are
exactly \(\lambda_1,\ldots,\lambda_r\), and these eigenvalues are pairwise distinct, there exists a
unique index \(\sigma(\ell)\in\{1,\ldots,r\}\) such that
\[
    \lambda'_\ell=\lambda_{\sigma(\ell)}.
\]
The map \(\sigma:\{1,\ldots,r\}\to\{1,\ldots,r\}\) is a permutation.

Since both \(u'_\ell\) and \(u_{\sigma(\ell)}\) are unit eigenvectors associated with the same
one-dimensional eigenspace \(E_{\lambda_{\sigma(\ell)}}\), using \eqref{eq:pm_u} there exists a sign
\(s_\ell\in\{\pm1\}\) such that
\[
    u'_\ell=s_\ell u_{\sigma(\ell)}.
\]

Let \(\Pi\) be the permutation matrix whose \(\ell\)-th column is \(e_{\sigma(\ell)}\), and let
\[
    S=\operatorname{diag}(s_1,\ldots,s_r).
\]
Then the column-wise identities above imply
\[
    U'=U\Pi S.
\]
Moreover, the identities \(\lambda'_\ell=\lambda_{\sigma(\ell)}\) imply
\(
    \lambda'=\Pi^\top \lambda
\) and the proof is complete.

\section{Smoothness of the restricted objective, $F_\eta|_{\widetilde{\Theta}}$.}
\label{app:lemma}

\begin{lemma}
\label{lem:objective:smoothness}
For any $\eta$, the restricted objective
$F_\eta|_{\widetilde{\Theta}}$ is $\mathcal{C}^2$.
\end{lemma}

\begin{proof}
Recall that
\[
    F_\eta(\theta)
    =
    \Phi(\theta)-\eta h(\theta),
    \qquad
    \Phi(\theta):=\ell(\theta)+R(\theta),
    \qquad
    h(\theta):=\log d_{\mathrm{spec}}(K(L)).
\]

The logistic negative log-likelihood $\ell$ is $\mathcal{C}^\infty$ in $\theta$ on
$\{L \ne 0\}$ as it is a finite sum of terms $\log \phi(\pm(a_i + a_j + \beta K_{ij}(L)))$,
where $\phi$ and $\log\phi$ are $\mathcal{C}^\infty$ on $\R$, and
$K(L) = nLL^\top / \|L\|_F^2$ is rational in $L$ with denominator nonzero on
$\{L \ne 0\}$. The regularizer $R$ is a quadratic polynomial and
therefore $\mathcal{C}^\infty$. It remains to show that
$h|_{\widetilde{\Theta}}$ is smooth.

Let
\[
    \theta=([L_1\mid 0],a,\beta)\in\widetilde{\Theta},
\]
and write the columns of $L_1$ as
\(
    L_1=(\ell_1,\ldots,\ell_s).
\)
Since $L_1^\top L_1$ is diagonal and $L_1$ has full column rank, define
\[
    d_i:=\|\ell_i\|_2^2>0,
    \qquad
    Q:=\sum_{j=1}^s d_j=\|L_1\|_F^2.
\]
Then
\(
    L_1^\top L_1=\operatorname{diag}(d_1,\ldots,d_s).
\)
The nonzero eigenvalues of $L_1L_1^\top$ are the same as the eigenvalues of
$L_1^\top L_1$. Therefore the nonzero eigenvalues of
\[
    K(L)
    =
    N\frac{L_1L_1^\top}{\|L_1\|_F^2}
\]
are
\[
    \lambda_i(K(L))
    =
    N\frac{d_i}{Q},
    \qquad i=1,\ldots,s.
\]
Hence the active spectral occupancies are
\[
    p_i(L)
    =
    \frac{\lambda_i(K(L))}{N}
    =
    \frac{d_i}{Q},
    \qquad i=1,\ldots,s.
\]
Thus, on $\widetilde{\Theta}$,
\[
    h(\theta)
    =
    \log d_{\mathrm{spec}}(K(L))
    =
    -\sum_{i=1}^s
    \frac{d_i}{Q}
    \log\left(\frac{d_i}{Q}\right).
\]
Each $d_i$ is a polynomial function of the entries of $L_1$, and on
$\widetilde{\Theta}$ we have $d_i>0$ and $Q>0$. Therefore the expression above
is smooth in $L_1$. It does not depend on $a$ or $\beta$. Hence
$h|_{\widetilde{\Theta}}$ is smooth.

Consequently,
\[
    F_\eta|_{\widetilde{\Theta}}
    =
    \Phi|_{\widetilde{\Theta}}
    -
    \eta h|_{\widetilde{\Theta}}
\]
is $\mathcal{C}^2$ for every fixed $\eta$. This proves the lemma.
\end{proof}

\section{Proof of Theorem \ref{thm:frontier-regularity}}
\label{app:proof-regularity}

\begin{figure}[t]
  \centering
  \begin{subfigure}[b]{0.48\textwidth}
    \centering
    \includegraphics[width=\linewidth]{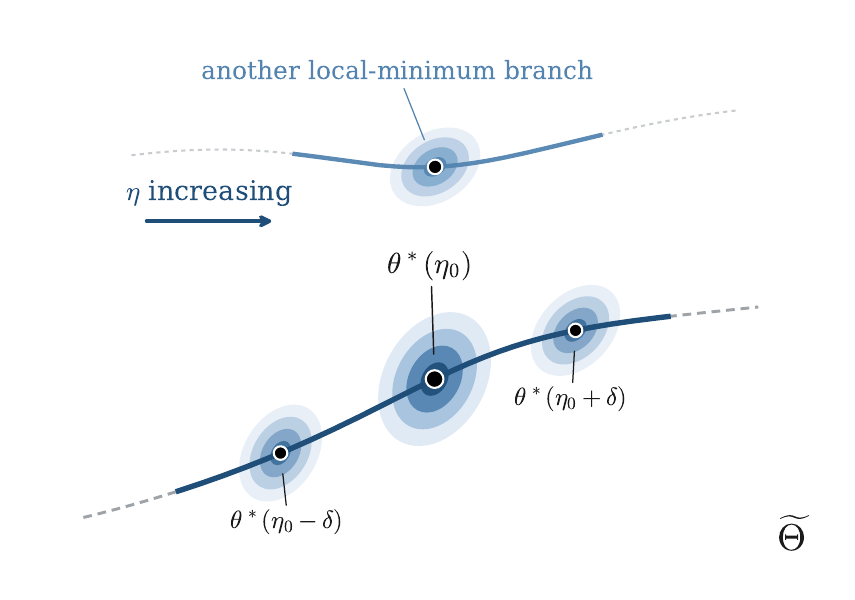}
    \caption{}
    \label{fig:thm6-a}
  \end{subfigure}\hfill
  \begin{subfigure}[b]{0.48\textwidth}
    \centering
    \includegraphics[width=\linewidth]{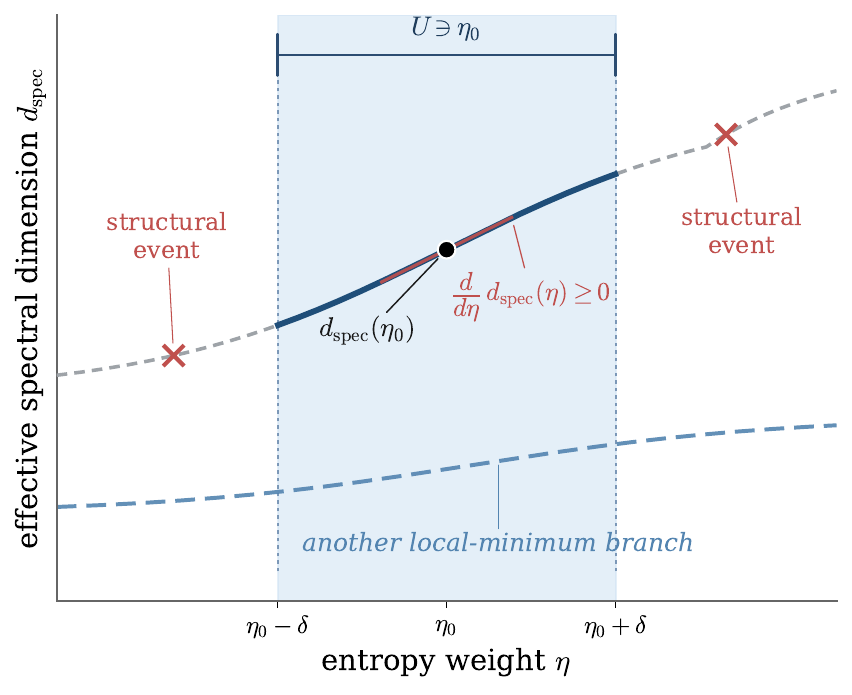}
    \caption{}
    \label{fig:thm6-b}
  \end{subfigure}
  \caption{
    \small Illustration of Theorem~\ref{thm:frontier-regularity}. \textbf{(a)} Solution branches in $\widetilde{\Theta}$.\ Near a reference
      point $\theta^*(\eta_0)$ satisfying the Assumptions~\ref{ass:A2} and \ref{ass:A3}, the optimum
      traces a smooth curve as $\eta$ varies, with each $\theta^*(\eta)$
      a strict local minimum of $F_\eta$ on its slice. $F_\eta$ may have additional local
      minima, each generating its own local branch under the same
      conditions. \textbf{(b)} Capacity profiles $d_{\mathrm{spec}}(\eta)$ along the
      branches. The tracked branch is smooth and nondecreasing
      on $U$. Beyond $U$, structural events on the tracked branch can interrupt the local guarantee.}
  \label{fig:thm6}
\end{figure}

First, we show the set $\widetilde{\Theta}$ is a smooth submanifold and define the tangent space $T_\theta \widetilde{\Theta}$.
Then we show the three parts of Theorem \ref{thm:frontier-regularity}.

\paragraph{Smooth Submanifold and Tangent Space.}
Recall the definition of the set $\widetilde{\Theta}$, i.e.,
\[
    \widetilde{\Theta}
    :=
    \left\{
        ([L_1\mid 0],a,\beta):
        L_1\in\mathbb R^{N\times s},
        \ \operatorname{rank}(L_1)=s,
        \ L_1^\top L_1 \text{ is diagonal},
        \ a\in\mathbb R^N,
        \ \beta>0
    \right\}.
\]
The set \(\widetilde{\Theta}\) is a smooth submanifold as per Definition \ref{def:smoothManifold}: indeed, the
conditions \(\operatorname{rank}(L_1)=s\) and \(\beta>0\) are open conditions,
while the constraints \(L_2=0\) and
\(\operatorname{offdiag}(L_1^\top L_1)=0\) form a smooth submersion on the
full-column-rank set; hence the claim follows from the regular level-set
theorem (see Chapter 8, \cite{lee2013smooth}).

Fix a point $\theta=([L_1\mid 0], \alpha, \beta)\in\widetilde{\Theta}$.
Next we show the tangent space of $\widetilde{\Theta}$ is
\[
    T_\theta \widetilde{\Theta}
        \coloneqq \left\{
            ([\dot L_1\mid 0],\dot a,\dot\beta):
            \operatorname{offdiag}
            \left(
                L_1^\top \dot L_1+\dot L_1^\top L_1
            \right)=0
        \right\}
    \,,
\]
using Definition \ref{def:tangentSpace:constraint}.
We write a general factor matrix \(L\in\mathbb R^{N\times r}\) as $L=[L_1\mid L_2]$ where $L_1\in\mathbb R^{N\times s}$, $L_2\in\mathbb R^{N\times(r-s)}$.
Consider the open set
$
    \mathcal U
    :=
    \left\{
        (L_1,L_2,a,\beta):
        \operatorname{rank}(L_1)=s,\ \beta>0
    \right\}
$.
Then $\widetilde\Theta \subseteq \mathcal U$ and is defined by the
smooth equality constraints
$L_2=0$ and $\operatorname{offdiag}(L_1^\top L_1)=0$.
Equivalently, define 
$c:\mathcal U \to \mathbb R^{N\times(r-s)} \times \mathbb S_0^s$ where 
$\mathbb S_0^s := \{ B\in\mathbb R^{s\times s} : B^\top=B,\ B_{ii}=0\text{ for all }i\}$ by 
\[
    c(L_1,L_2,a,\beta)
    :=
    \left(
        L_2,\,
        \operatorname{offdiag}(L_1^\top L_1)
    \right)
    \,.
\]
Then $\widetilde\Theta=c^{-1}(0)$.
After identifying matrix spaces with Euclidean spaces by vectorization, this is exactly of the form \(M=\{x:c(x)=0\}\) in Definition
\(\ref{def:tangentSpace:constraint}\).
Thus, we know $T_\theta \widetilde{\Theta} = \ker J_c(\theta)$ which we calculate next.

Let $\dot\theta=([\dot L_1\mid \dot L_2],\dot a,\dot\beta)$ be an arbitrary direction in the ambient space. The first component of \(c\) is the map $(L_1,L_2,a,\beta)\mapsto L_2$, which is linear in \(L_2\) and independent of \(L_1,a,\beta\). Therefore the Jacobian of this component maps $([\dot L_1\mid \dot L_2],\dot a,\dot\beta)\mapsto\dot L_2$.

It remains to compute the Jacobian of the second component, $L_1\mapsto \operatorname{offdiag}(L_1^\top L_1)$.
Write $L_1=(\ell_1,\ldots,\ell_s)$ and $\dot L_1=(\dot\ell_1,\ldots,\dot\ell_s)$, where \(\ell_i,\dot\ell_i\in\mathbb R^n\). For \(i\neq j\), define the scalar constraint
\[
    c_{ij}(L_1)
    :=
    (L_1^\top L_1)_{ij}
    =
    \ell_i^\top \ell_j
    =
    \sum_{k=1}^n (L_1)_{ki}(L_1)_{kj}
    \,.
\]
We identify \(L_1\) with the vector \(\operatorname{vec}(L_1)\). Hence
\(J_{c_{ij}}(L_1)\) denotes the ordinary Jacobian row of the scalar function
\(c_{ij}\) with respect to the entries of \(L_1\). For \(a\in[N]\) and
\(b\in[s]\), we have
\[
    \frac{\partial c_{ij}}{\partial (L_1)_{ab}}(L_1)
    =
    \mathbf 1_{\{b=i\}}(L_1)_{aj}
    +
    \mathbf 1_{\{b=j\}}(L_1)_{ai}
    \,.
\]
Therefore, applying this Jacobian row to the direction
\(\operatorname{vec}(\dot L_1)\), we obtain
\begin{align*}
    J_{c_{ij}}(L_1)\operatorname{vec}(\dot L_1)
    &=
    \sum_{a=1}^n\sum_{b=1}^s
    \frac{\partial c_{ij}}{\partial (L_1)_{ab}}(L_1)
    (\dot L_1)_{ab}                                      \\
    &=
    \sum_{a=1}^n
    (L_1)_{aj}(\dot L_1)_{ai}
    +
    \sum_{a=1}^n
    (L_1)_{ai}(\dot L_1)_{aj}                            \\
    &=
    \ell_j^\top \dot\ell_i+\ell_i^\top\dot\ell_j         \\
    &=
    \bigl(
        L_1^\top \dot L_1+\dot L_1^\top L_1
    \bigr)_{ij}
    \,.
\end{align*}
Collecting these identities over all off-diagonal pairs \(i\neq j\), we get
$J_{\operatorname{offdiag}(L_1^\top L_1)}(L_1) \operatorname{vec}(\dot L_1) = \operatorname{offdiag}
    \bigl(
        L_1^\top \dot L_1+\dot L_1^\top L_1
    \bigr)
$.

Combining the two components of \(c\), the full Jacobian satisfies
\[
    J_c(L_1,L_2,a,\beta)
    \begin{pmatrix}
        \dot L_1\\
        \dot L_2\\
        \dot a\\
        \dot\beta
    \end{pmatrix}
    =
    \left(
        \dot L_2,\,
        \operatorname{offdiag}
        \bigl(
            L_1^\top \dot L_1+\dot L_1^\top L_1
        \bigr)
    \right).
\]


By Definition \ref{def:tangentSpace:constraint}, the tangent space is the kernel of the Jacobian, i.e, all $\dot\theta$ such that $J_c(\theta)\, \dot\theta = 0$.
Therefore, using the formula above, we get the result
\[
    T_\theta\widetilde{\Theta}
    =
    \left\{
        ([\dot L_1\mid 0],\dot a,\dot\beta):
        \operatorname{offdiag}
        \bigl(
            L_1^\top \dot L_1+\dot L_1^\top L_1
        \bigr)=0
    \right\}
    \,.
\]

\paragraph{Claim (i)}
Recall that $\theta^\star\in\widetilde{\Theta}$ satisfies Assumptions \ref{ass:A2} and \ref{ass:A3} and
\[
    F_\eta(\theta)
    =
    \Phi(\theta)-\eta h(\theta),
    \qquad
    \Phi(\theta):=\ell(\theta)+R(\theta),
    \qquad
    h(\theta):=\log d_{\mathrm{spec}}(K(L)).
\]
Since $\widetilde \Theta$ is a smooth manifold as per Definition \ref{def:smoothManifold} there exists a smooth local chart
\(
    \psi:W\subset\mathbb R^m\to\tilde\Theta
\) with
\(
    \psi(0)=\theta^\star.
\)
Define
\(
    \widehat F(x,\eta)
    :=
    F_\eta(\psi(x)),
    \)
where
\(   \widehat h(x)
    :=
    h(\psi(x))
\)
and
\(
    \widehat \Phi(x):=\Phi(\psi(x)).
\)
Then,
\[
    \widehat F(x,\eta)
    =
    \widehat\Phi(x)-\eta\widehat h(x).
\]
By Lemma~\ref{lem:objective:smoothness} in Appendix~\ref{app:lemma}, \(\widehat F\) is \(C^2\) in \((x,\eta)\). Now define $G\colon\R^m\times\R\to\R^m$ as
\[
    G(x,\eta)
    :=
    \nabla_x\widehat F(x,\eta).
\]
By Assumption~\ref{ass:A3}, since $\psi(0)=\theta^\star$, it holds $G(0,\eta_0)=\nabla_x\widehat F(0,\eta_0)=0$.
Moreover, let
\(
    \nabla_xG(0,\eta_0)
    =
    \nabla_x^2\widehat F(0,\eta_0).
\)
By Assumption~\ref{ass:A3} matrix $\nabla_xG(0,\eta_0)$ is positive definite.
In particular, it is invertible. Hence, by the implicit function theorem
\citep[Ch.~5]{lee2013smooth}, there exist an open interval \(U\ni\eta_0\) and a
\(\mathcal{C}^1\) map
\(
    x^\star:U\to W
\)
such that
\(
    x^\star(\eta_0)=0
\)
and
\(
    G(x^\star(\eta),\eta)=0
    \ 
    \text{for every }\eta\in U.
\)
Define
\[
    \theta^\star(\eta):=\psi(x^\star(\eta)).
\]
Then \(\theta^\star:U\to\tilde\Theta\) is \(\mathcal{C}^1\),
\(\theta^\star(\eta_0)=\theta^\star\), and
\begin{align}    
    \nabla_\theta F_\eta(\theta^\star(\eta))
    \big|_{T_{\theta^\star(\eta)}\tilde\Theta}
    =
    0
    \qquad
    \text{for every }\eta\in U.
    \label{eq:restrictedCriticalPointCondition}
\end{align}
This proves the restricted criticality part of claim~(i). It remains to show that these restricted critical points, i.e., the ones satisfying Equation \ref{eq:restrictedCriticalPointCondition}, are strict local
minima. Define 
\[
    H(\eta)
    :=
    \nabla_x^2\widehat F(x^\star(\eta),\eta).
\]
The map \(\eta\mapsto H(\eta)\) is continuous, because \(\widehat F\) is \(\mathcal{C}^2\)
and \(x^\star\) is \(\mathcal{C}^1\). Notice that, for \(\eta=\eta_0\), we have
\(
    H(\eta_0)
    =
    \nabla_x^2\widehat F(0,\eta_0)
    \succ 0
\)
by Assumption~\ref{ass:A3}. Positive definiteness is an open condition.
Therefore, after possibly shrinking \(U\), we have
\(
    H(\eta)\succ 0
    \qquad
    \text{for every }\eta\in U.
\)

Thus for any \(\eta\in U\) the function \(x\mapsto\widehat F(x,\eta)\) satisfies
\[
    \nabla_x\widehat F(x^\star(\eta),\eta)=0
    \qquad
    \nabla_x^2\widehat F(x^\star(\eta),\eta)=H(\eta)\succ 0
    \,.
\]
By the standard second-order sufficient condition for an unconstrained local
minimum \citep[Ch.~2]{nocedal2006numerical}, \(x^\star(\eta)\) is a strict
local minimum of \(x\mapsto \widehat F(x,\eta)\). Since \(\psi\) is a local
coordinate chart for \(\widetilde\Theta\), this is exactly the statement that
\(\theta^\star(\eta)\) is a strict local minimum of \(F_\eta\) restricted to
\(\widetilde\Theta\). Claim~(i) follows.

\paragraph{Claim (ii)}
By claim~(i), the map
\(
    \eta\mapsto\theta^\star(\eta)
\)
is \(\mathcal{C}^1\) and takes values in \(\tilde\Theta\). Therefore we may write
\(
    L^\star(\eta)=[L_1^\star(\eta)\mid 0],
\)
where
\(
    L_1^\star(\eta)
    =
    (\ell_1(\eta),\ldots,\ell_s(\eta))
\)
depends \(\mathcal{C}^1\) on \(\eta\).
Define
\[
    d_i(\eta):=\|\ell_i(\eta)\|_2^2,
    \qquad
    Q(\eta):=\sum_{j=1}^s d_j(\eta).
\]
Since
\(
    d_i(\eta_0)=d_i^\star>0,
\)
continuity implies that, after shrinking \(U\) if necessary,
\(
    d_i(\eta)>0
    \qquad
    \text{for every }i=1,\ldots,s
    \text{ and every }\eta\in U.
\)
Therefore,
\[
    p_i(\eta)
    :=
    \frac{d_i(\eta)}{Q(\eta)}
\]
is \(\mathcal{C}^1\) on \(U\). 
Notice that $p_i(\eta) = \lambda_i(K(L^\star(\eta)))/N$ since $(L^\star)^T L$ is diagonal and ${K(L^\star) = N\cdot L^\star(L^\star)^T / \tr(L^\star(L^\star)^T)}$.\footnote{We repeat this argument in more detail in the proof of Lemma \ref{lem:objective:smoothness}.}
Hence
\[
    \log d_{\mathrm{spec}}(\eta)
    =
    \log d_{\mathrm{spec}}(K(L^\star(\eta)))
    =
    -\sum_{i=1}^s p_i(\eta)\log p_i(\eta).
\]
The function \(x\mapsto x\log x\) is \(\mathcal{C}^\infty\) on \((0,\infty)\), and each
\(p_i(\eta)\) is positive and \(\mathcal{C}^1\) (since it is a polynomial in the entries of $L^\star$ and $L^\star$ is $\mathcal{C}^1$ on $\eta$). Therefore \(\log d_{\mathrm{spec}}(\eta)\) is \(\mathcal{C}^1\).
Since
\(
    d_{\mathrm{spec}}(\eta)=\exp(\log d_{\mathrm{spec}}(\eta)),
\)
the capacity profile \(d_{\mathrm{spec}}\) is also \(\mathcal{C}^1\). This proves Claim~(ii).

\paragraph{Claim (iii)}
We work in the same local coordinates as in the proof of claim~(i). Recall that
\[
    \widehat F(x,\eta)
    =
    \widehat\Phi(x)-\eta\widehat h(x).
\]
The branch satisfies
\[
    \nabla_x\widehat F(x^\star(\eta),\eta)=0
    \qquad
    \text{for every }\eta\in U.
\]
Differentiate this identity with respect to \(\eta\). Here the first term comes
from differentiating the map
\[
    x\mapsto \nabla_x\widehat F(x,\eta)
\]
along the curve \(x^\star(\eta)\), while the second term comes from the
explicit dependence of \(\widehat F\) on \(\eta\). Since
\[
    \partial_\eta\nabla_x\widehat F(x,\eta)
    =
    -\nabla_x\widehat h(x),
\]
we get
\[
    0
    =
    \frac{d}{d\eta}
    \nabla_x\widehat F(x^\star(\eta),\eta)
    =
    H(\eta)\frac{dx^\star}{d\eta}(\eta)
    -
    \nabla_x\widehat h(x^\star(\eta)).
\]
Therefore,
\[
    H(\eta)\frac{dx^\star}{d\eta}(\eta)
    =
    \nabla_x\widehat h(x^\star(\eta)).
\]
By claim~(i), \(H(\eta)\succ 0\) for every \(\eta\in U\). Thus \(H(\eta)\) is
invertible, and
\[
    \frac{dx^\star}{d\eta}(\eta)
    =
    H(\eta)^{-1}
    \nabla_x\widehat h(x^\star(\eta)).
\]
Now
\[
    \log d_{\mathrm{spec}}(\eta)
    =
    h(\theta^\star(\eta))
    =
    \widehat h(x^\star(\eta)).
\]
Hence, by the chain rule,
\[
    \frac{d}{d\eta}\log d_{\mathrm{spec}}(\eta)
    =
    \nabla_x\widehat h(x^\star(\eta))^\top
    \frac{dx^\star}{d\eta}(\eta).
\]
Substituting the expression for \(dx^\star/d\eta\), we obtain
\[
    \frac{d}{d\eta}\log d_{\mathrm{spec}}(\eta)
    =
    \nabla_x\widehat h(x^\star(\eta))^\top
    H(\eta)^{-1}
    \nabla_x\widehat h(x^\star(\eta)).
\]
 Since \(H(\eta)\succ 0\), also
\(H(\eta)^{-1}\succ 0\). Therefore
\[
    \frac{d}{d\eta}\log d_{\mathrm{spec}}(\eta)\ge 0.
\]

\subsection{Structural events along the $\eta$ path}
\label{app:structural-events}

Theorem~\ref{thm:frontier-regularity} guarantees $\mathcal{C}^1$
regularity and local monotonicity of the realized effective
dimension $d_{\mathrm{spec}}(\eta)$ on a neighborhood of any
reference $\eta_0$ where Assumptions~\ref{ass:A2}--\ref{ass:A3}
hold. Three failure modes can interrupt this regularity along an
extended $\eta$ path: an eigenvalue enters or leaves the active
support; two adjacent positive eigenvalues approach equality
(violating Assumption~\ref{ass:A2}); or the restricted Hessian
becomes degenerate (violating Assumption~\ref{ass:A3}).

We monitor the first two directly during the $\eta$ sweep. Support
changes are flagged when $k_\tau(K) =
|\{i : \lambda_i(K) \geq \tau \lambda_1(K)\}|$ at relative-mass
threshold $\tau$ changes between consecutive probes. Near-degeneracy
events are flagged when the minimum gap between adjacent positive
eigenvalues drops below a relative tolerance of the largest
eigenvalue. Hessian degeneracy is not directly observable from the
spectrum and is treated as correlated with the first two.

Figure~\ref{fig:frontier-allfour} shows the $\eta$-path of the
realized effective dimension on four representative datasets, one
from each structural family, at rank caps $r{=}128$ and $r{=}256$.
Markers indicate detected structural events, without distinguishing
type. Three observations are consistent across datasets:
\textbf{(i)} the $\eta\!\mapsto\!d_{\mathrm{spec}}$ map is monotone
in $\eta$ across the entire sweep, including across detected events;
\textbf{(ii)} events concentrate near the extremes of the
$\eta$-range --- at large positive $\eta$ the spectrum is driven
toward uniformity and adjacent eigenvalues collide, while at large
negative $\eta$ the spectrum compresses and modes drop below the
support threshold; \textbf{(iii)} between events, the path is
visibly smooth, consistent with the local $\mathcal{C}^1$ guarantee
of Theorem~\ref{thm:frontier-regularity}.

Empirically, structural events do not invalidate the capacity
frontier: $d_{\mathrm{spec}}(\eta)$ remains monotone and continuous
in $\eta$ on every dataset and rank cap we examined. The events
appear as visible inflections rather than as discontinuities,
suggesting that on these graphs the local $\mathcal{C}^1$ branches
connect into a globally continuous (if not globally $\mathcal{C}^1$)
profile. We report $d_{\mathrm{spec}}(K^\star)$ as the operating
coordinate while tracking events as a diagnostic.

Comparing the two rank caps shows that the event landscape depends
on whether $r$ binds the representation. On rank-cap-binding graphs
(\textsl{ca-grqc}, \textsl{bio-grid-worm}), increasing $r$ from
$128$ to $256$ extends the $d_{\mathrm{spec}}$ range and shifts
events upward; on saturated-regime graphs (\textsl{inf-power},
\textsl{socfb-American75}), the path stays well below the cap at
both $r$ and the event landscape is sparser. This is consistent
with the regime distinction in
the main paper.

 \begin{figure}[t]
  \centering
  \begin{subfigure}[t]{0.245\linewidth}
    \includegraphics[width=\linewidth]{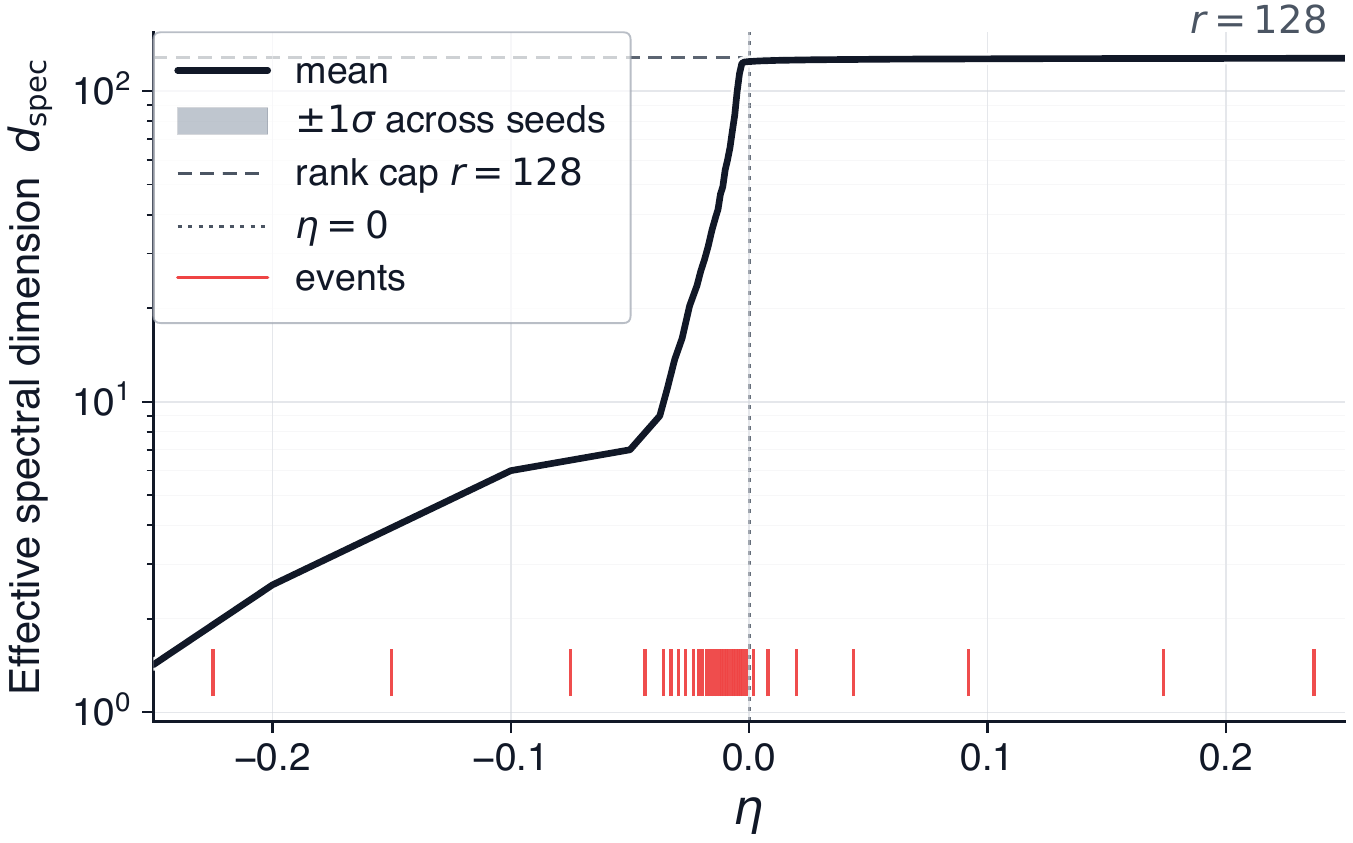}
    \caption*{\textsl{ca-grqc}}
  \end{subfigure}\hfill
  \begin{subfigure}[t]{0.245\linewidth}
    \includegraphics[width=\linewidth]{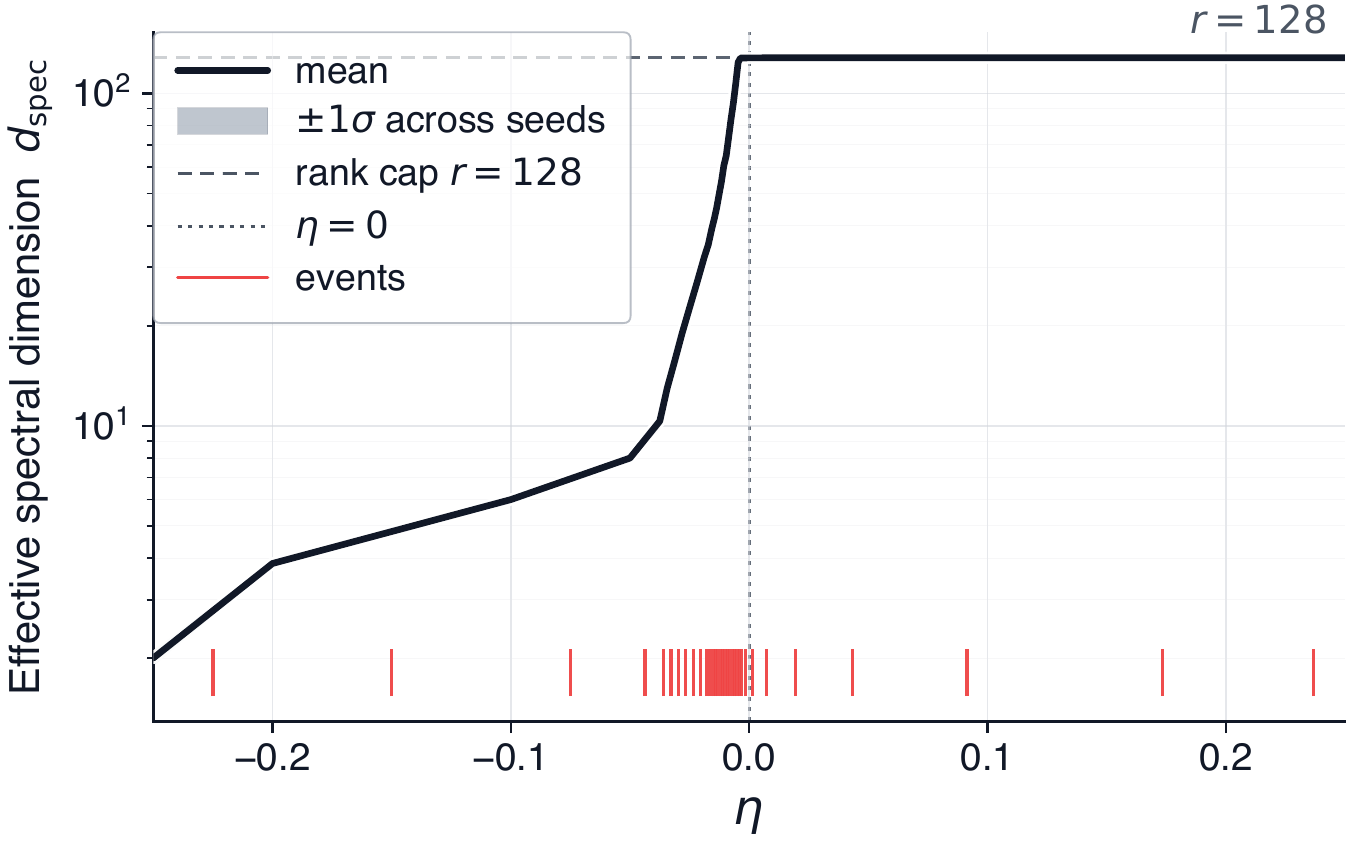}
    \caption*{\textsl{inf-power}}
  \end{subfigure}\hfill
  \begin{subfigure}[t]{0.245\linewidth}
    \includegraphics[width=\linewidth]{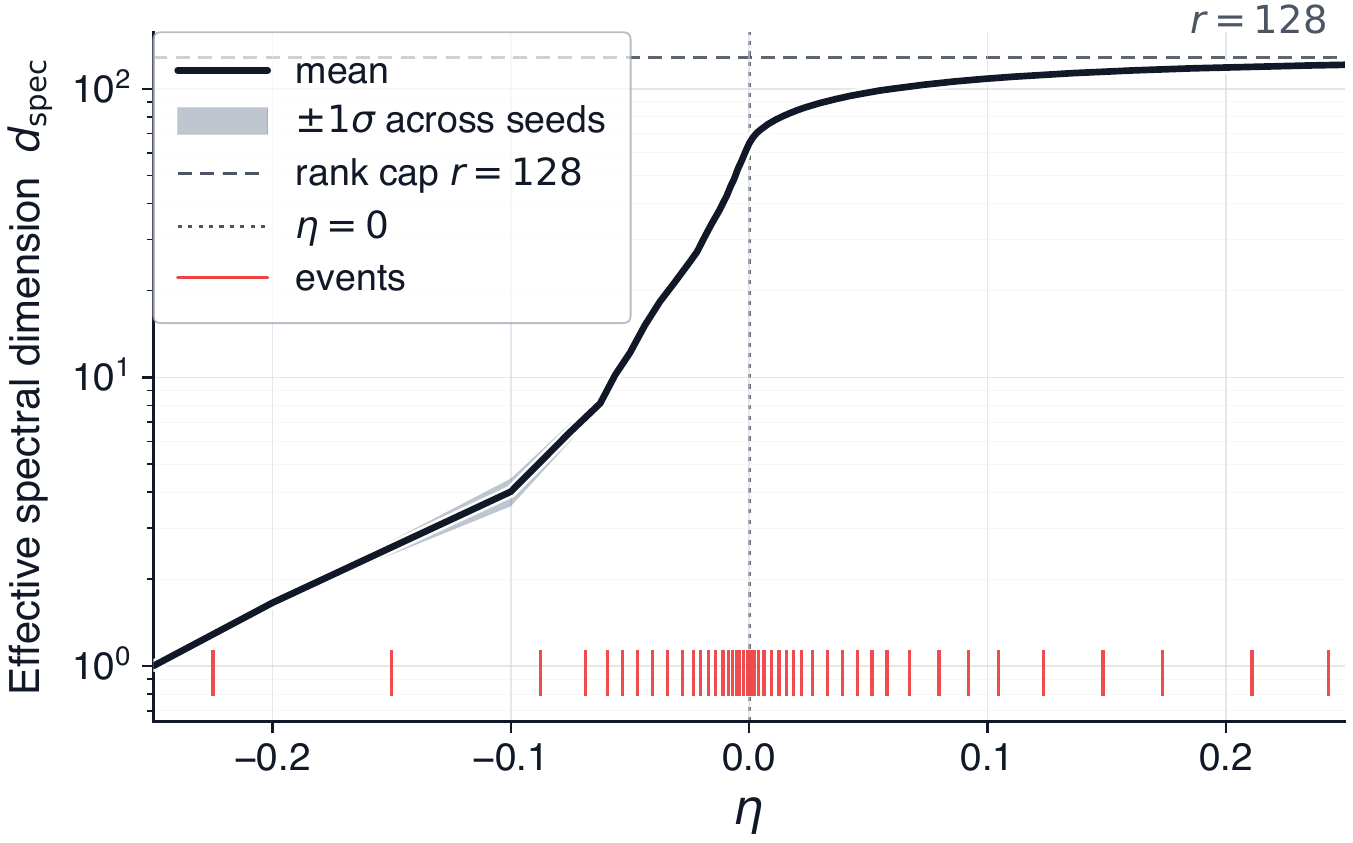}
    \caption*{\textsl{socfb-American75}}
  \end{subfigure}\hfill
  \begin{subfigure}[t]{0.245\linewidth}
    \includegraphics[width=\linewidth]{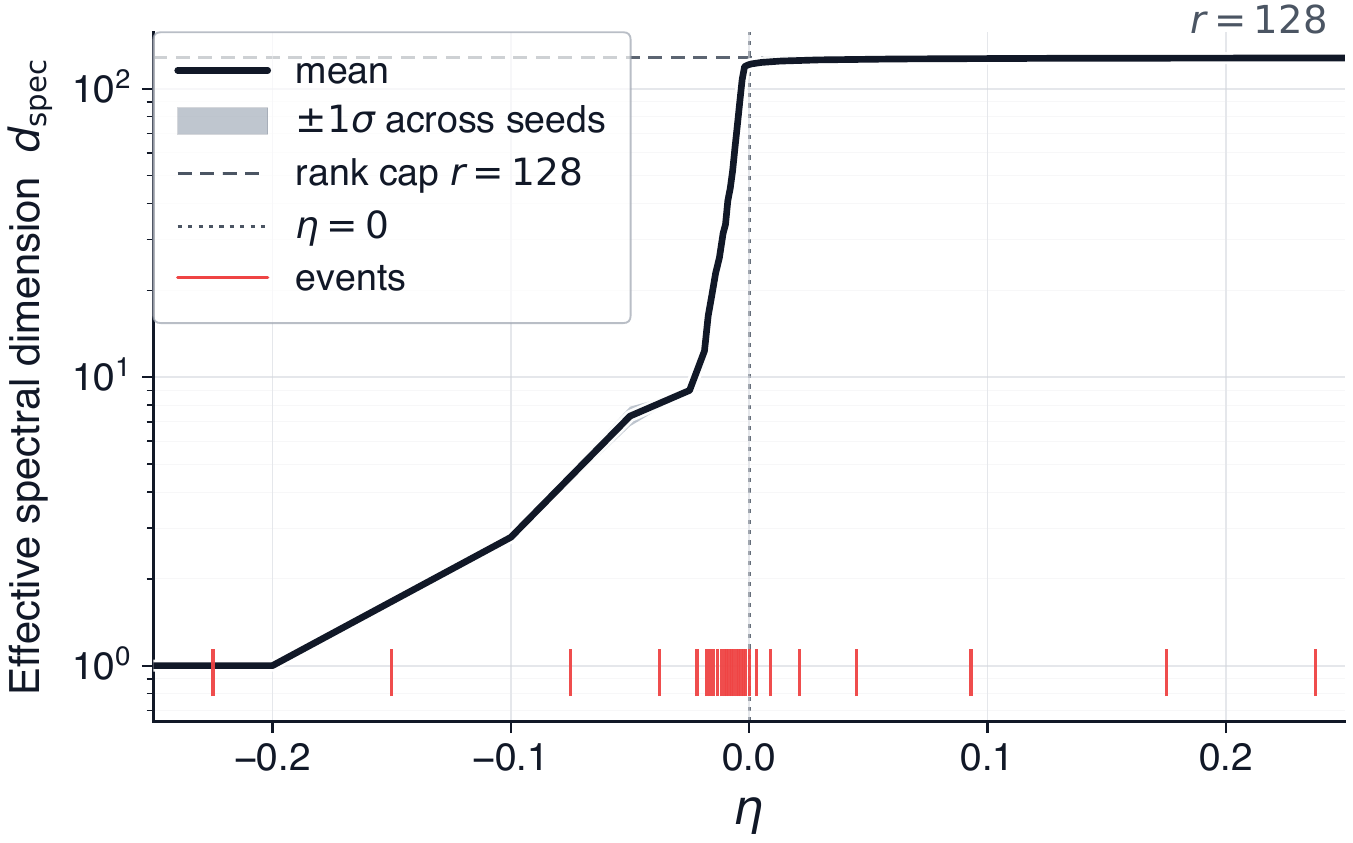}
    \caption*{\textsl{bio-grid-worm}}
  \end{subfigure}
  \vspace{2pt}
  \begin{subfigure}[t]{0.245\linewidth}
    \includegraphics[width=\linewidth]{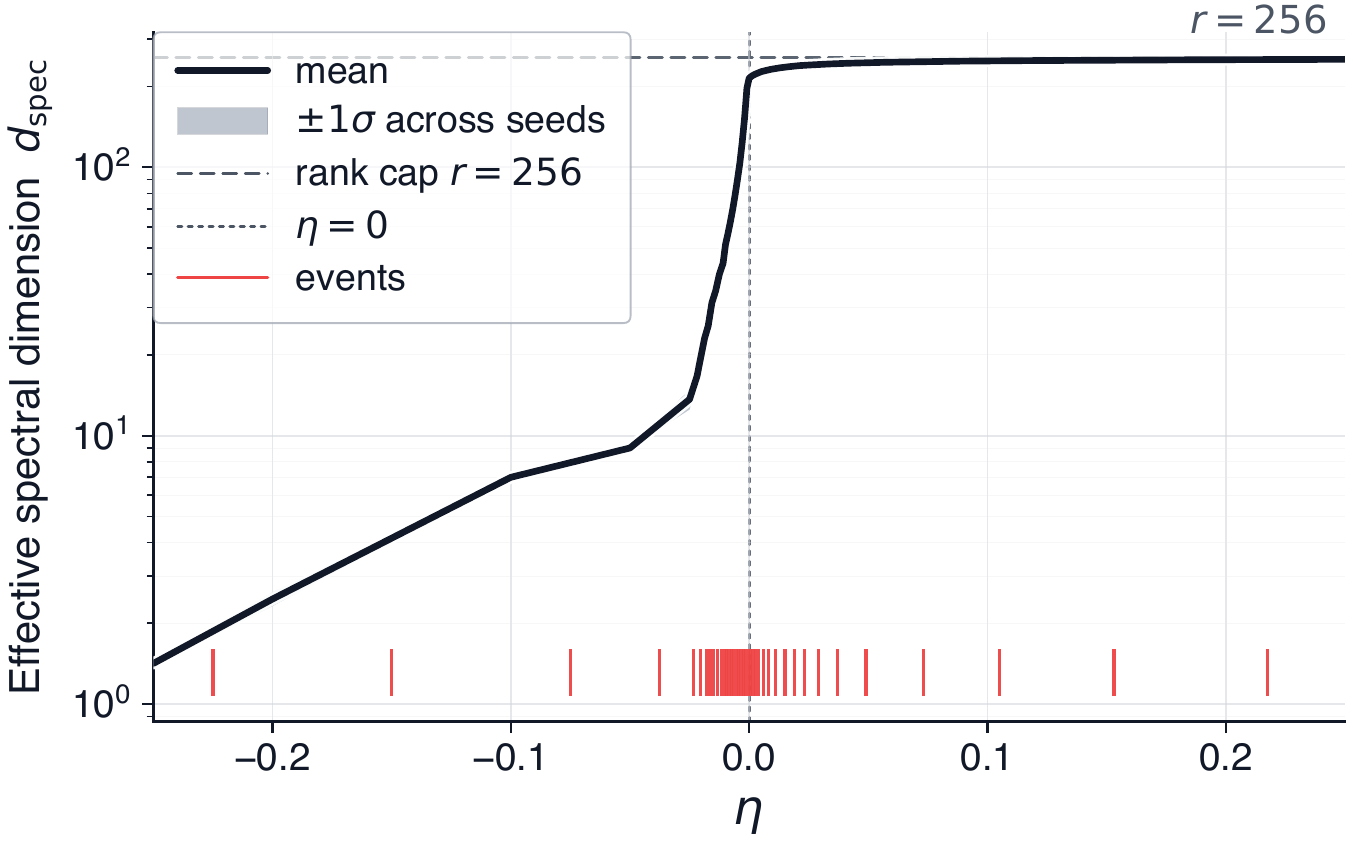}
  \end{subfigure}\hfill
  \begin{subfigure}[t]{0.245\linewidth}
    \includegraphics[width=\linewidth]{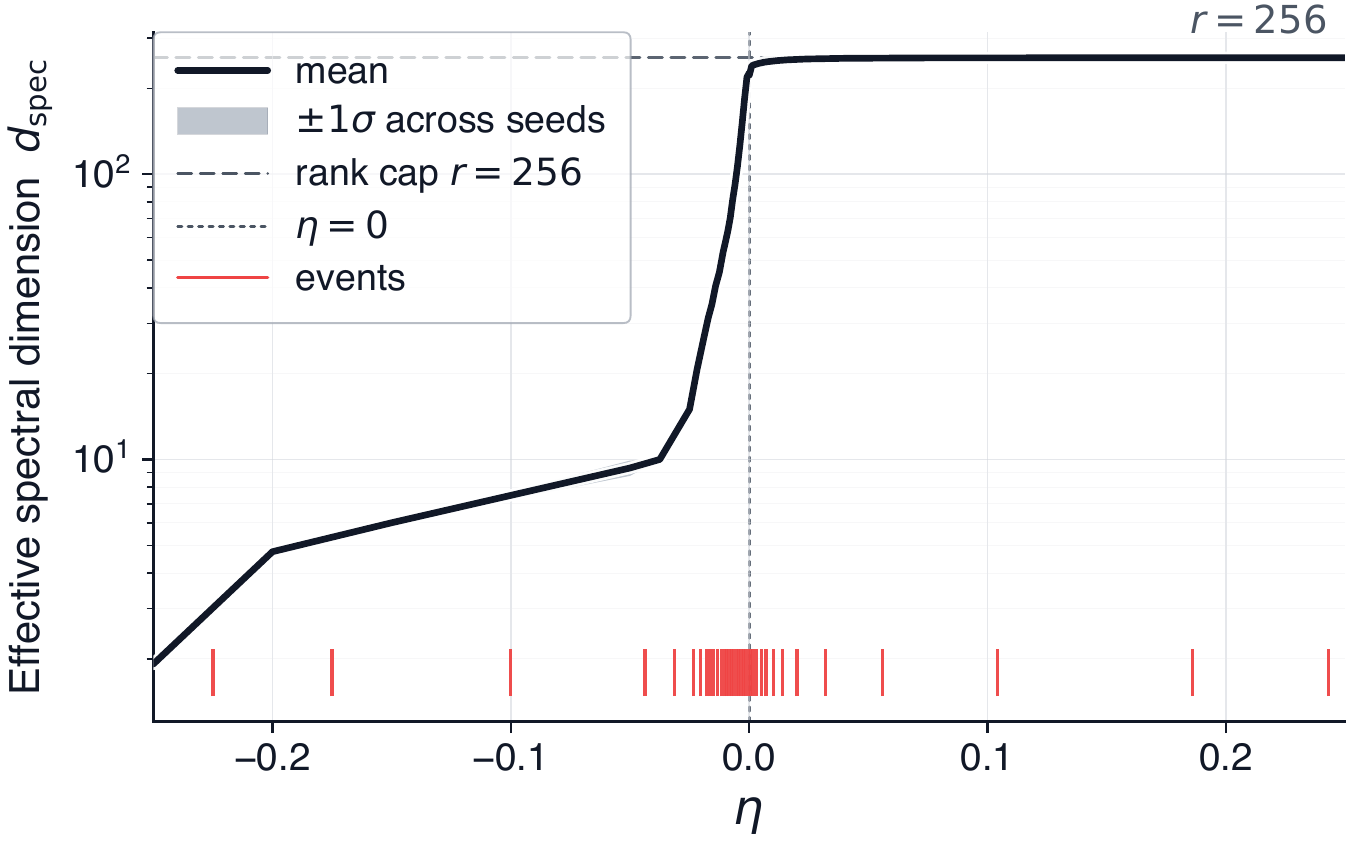}
  \end{subfigure}\hfill
  \begin{subfigure}[t]{0.245\linewidth}
    \includegraphics[width=\linewidth]{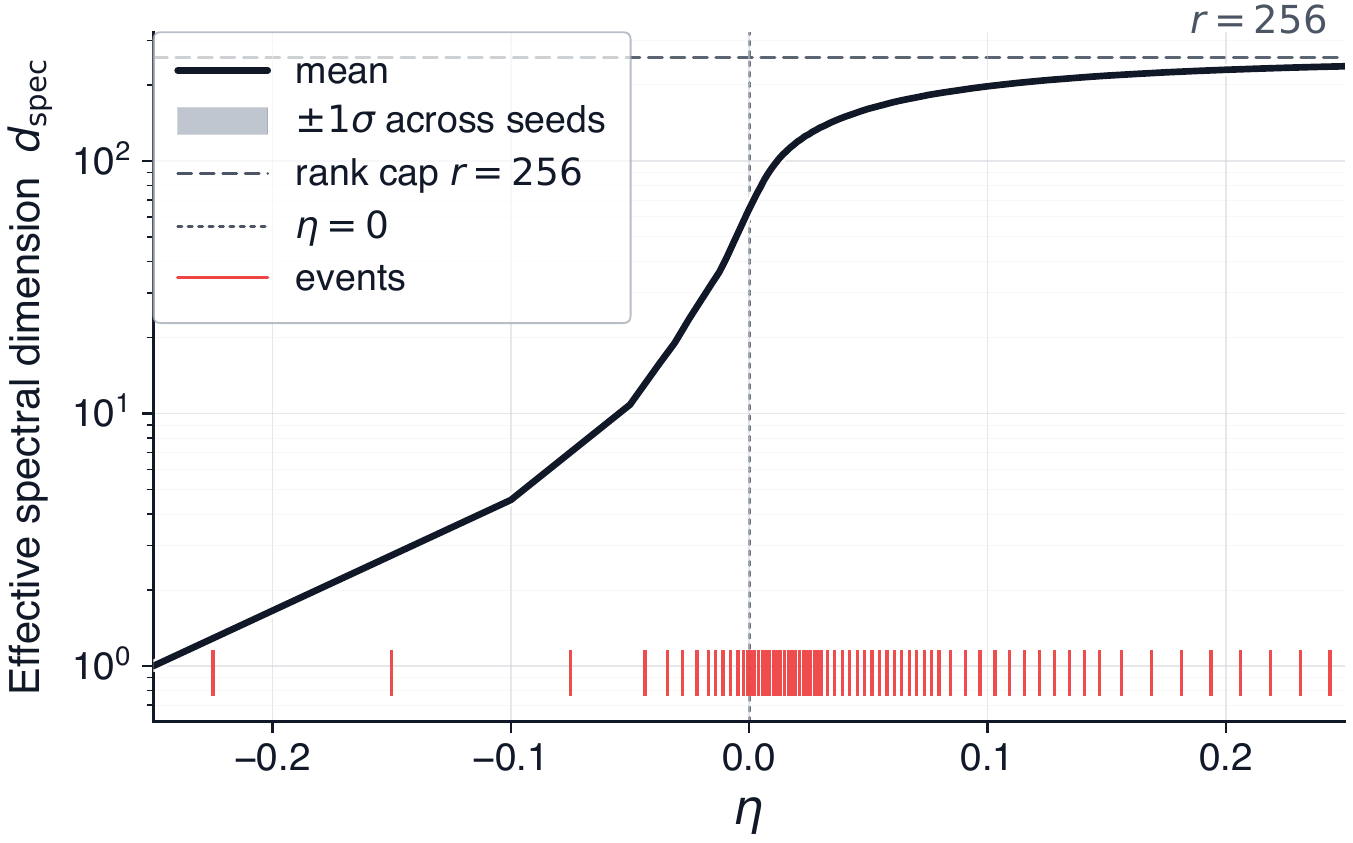}
  \end{subfigure}\hfill
  \begin{subfigure}[t]{0.245\linewidth}
    \includegraphics[width=\linewidth]{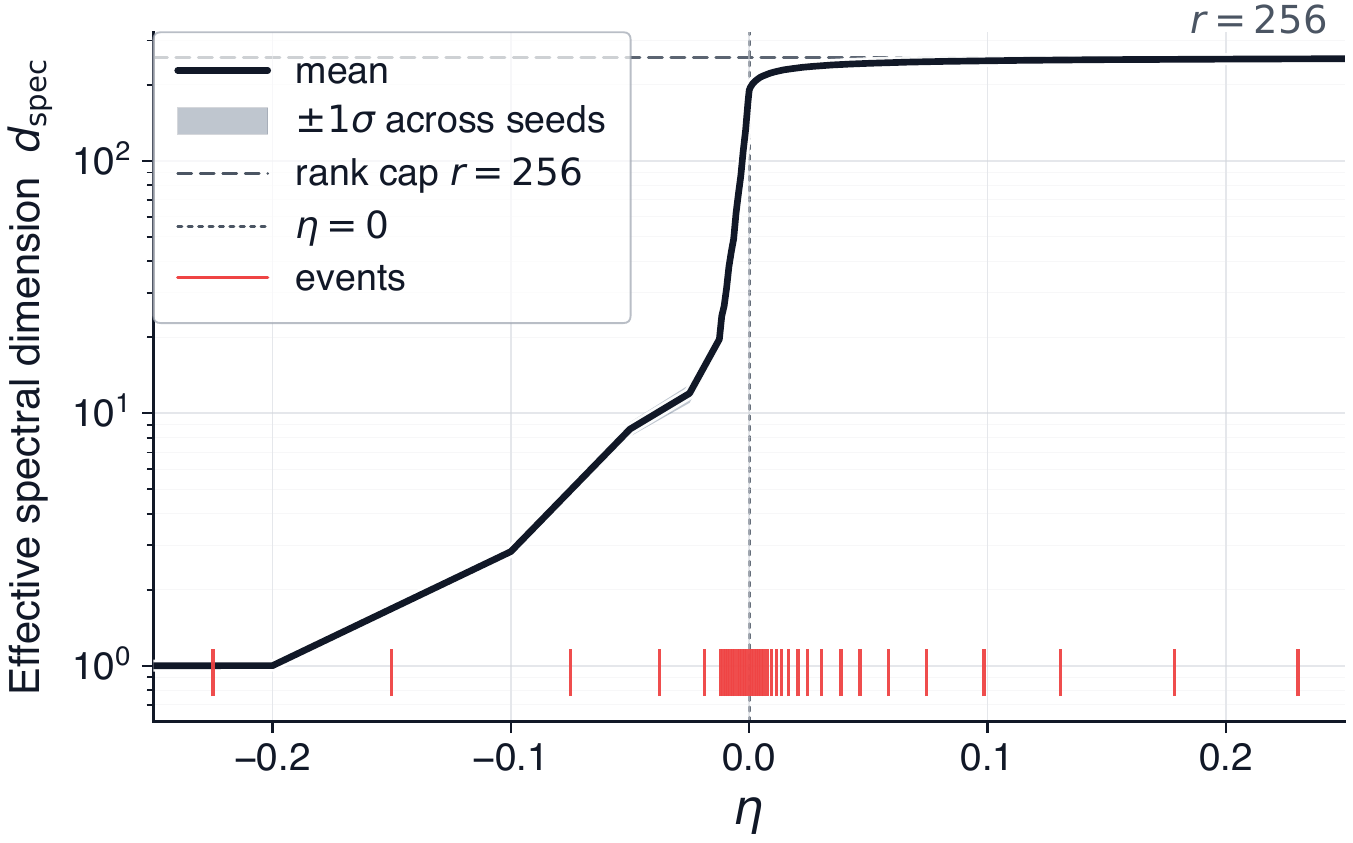}
  \end{subfigure}
\caption{\textbf{Structural events along the $\eta$ path at two
rank caps}, on four representative datasets spanning the four
structural families. \emph{Top row:} $r{=}128$. \emph{Bottom row:}
$r{=}256$. Columns: citation (\textsl{ca-grqc}), infrastructure
(\textsl{inf-power}), social (\textsl{socfb-American75}),
biological (\textsl{bio-grid-worm}). Each panel shows
$d_{\mathrm{spec}}(\eta)$ vs $\eta$ on log scale, mean $\pm 1\sigma$
across three seeds. Vertical line: $\eta{=}0$. Event markers
indicate probes where a structural event was detected along the
fitted branch (eigenvalue entering or leaving the active support,
or adjacent positive eigenvalues approaching equality; see
text). Most $\eta$-paths are smooth across $\eta\in[-0.25, 0.25]$
with events concentrating near the extremes; the realized-dimension
profile remains continuous and monotone in $\eta$ on every dataset
even where local $\mathcal{C}^1$ regularity is interrupted by an
event.}
  \label{fig:frontier-allfour}
\end{figure}

\section{Experimental setup}
\label{sec:supp-experimental}

We extensively evaluate \textsc{Spectra} against prominent baseline graph
representation learning methods, including unconstrained Euclidean
embeddings, matrix factorization, mixed-membership models, and
latent-distance simplex-based representations, across networks of
varying sizes and structures. We ran all methods using publicly
available implementations. When GPU
support was available, experiments were executed on an NVIDIA A100 GPU;
otherwise, we used an Apple M2 machine with 8\,GB RAM.

\paragraph{Optimization.}
For \textsc{Spectra} we optimize the $\eta$-augmented objective $F_\eta$
(Eq.~\eqref{eq:map-objective}) with Adam~\citep{kingma2014adam} at
learning rate $10^{-2}$ for $6{,}000$ iterations unless stated
otherwise. At each step we sample five non-edges per observed edge
uniformly from the complement of the training graph, and combine the
positive and negative log-likelihood terms with equal weight.

\paragraph{Regularization.}
The spectral regularizer $R$ in $F_\eta$ is weighted by a fixed
coefficient $\lambda = 10^{-4}$ in all experiments reported in the main
text and supplement. We found results to be robust to this choice:
varying $\lambda$ over several orders of magnitude around $10^{-4}$
produced negligible changes in test AUC and in the achieved
$d_{\mathrm{spec}}$, so no per-dataset tuning was performed.

\paragraph{Baselines.}
For embedding baselines, dyadic edge features are constructed using the
standard binary operators (average, Hadamard, weighted-$L_1$,
weighted-$L_2$), followed by an $L_2$-regularized logistic regression
classifier; in each cell we report the operator that maximizes the
baseline's test AUC, giving each baseline its strongest possible
configuration. For likelihood-based latent graph models, including
ours, we evaluate links directly from the learned log-odds, with no
auxiliary classifier and no operator selection.

\paragraph{Test-blind protocol.}
\textsc{Spectra} is reported in a strictly test-blind regime: no
hyperparameter, $\eta$ value, $d_{\mathrm{spec}}$ target, or model
checkpoint is selected with reference to any test quantity, on any
dataset, at any stage.

\section{Frontier protocol}
\label{app:frontier-protocol}

We test whether predictive performance is organized by realized
spectral capacity $d_{\mathrm{spec}}(\eta)$ rather than by the nominal
rank cap $r$. If this is the case, curves obtained at different $r$
should collapse onto a common frontier when plotted against
$d_{\mathrm{spec}}$, and any residual dependence on $r$ should appear
only once the rank cap binds the achievable capacity.

For each dataset and three seeds, we sweep the spectral-prior weight
$\eta$ at three rank caps $r\in\{64,128,256\}$. Each probe is trained
from scratch (cold start, no warm-starting across $\eta$ values) so
that the realized capacity at a given $\eta$ reflects optimization
under that prior alone, rather than path-dependence from a neighboring
configuration.

\paragraph{Adaptive grid.}
The $\eta$ values are chosen \emph{adaptively}: starting from the most
negative $\eta$, the next probe is selected by halving or doubling the
current $\eta$-step according to whether $d_{\mathrm{spec}}$ moved by
more than an upper threshold or less than a lower threshold between
consecutive probes. This refines the grid where $d_{\mathrm{spec}}(\eta)$
changes rapidly while spending few probes on slowly-varying segments,
so structural transitions are resolved without paying for a uniformly
fine sweep.

\clearpage

\end{document}